\documentclass[twoside,11pt,abbrvbib]{article}

\usepackage[preprint]{jmlr2e}

\usepackage{amsmath}
\usepackage{amsfonts}
\usepackage{mathtools}
\usepackage{bbm}
\usepackage{subcaption}
\usepackage{enumitem}
\usepackage{mathrsfs}
\numberwithin{equation}{section} \usepackage{colortbl} \usepackage[font=footnotesize]{caption}

\usepackage[capitalise, nameinlink]{cleveref}

\usepackage{booktabs}
\usepackage{arydshln}
\usepackage{multirow}
\usepackage{stmaryrd}
\usepackage{thmtools}

\usepackage{regexpatch}
\makeatletter
\xpatchcmd\thmt@restatable{\csname #2\@xa\endcsname\ifx\@nx#1\@nx\else[{#1}]\fi
}{\ifthmt@thisistheone
\csname #2\@xa\endcsname\ifx\@nx#1\@nx\else[{#1}]\fi
\else
\csname #2\@xa\endcsname[{Restated}]
\fi}{}{}
\makeatother

\usepackage[margin=1in]{geometry}

\usepackage{graphicx}    
\usepackage[dvipsnames]{xcolor}
\usepackage{ninecolors}

\usepackage{tikz}
    \usetikzlibrary{calc}
    \usetikzlibrary{positioning}
    \usetikzlibrary{intersections}
    \usetikzlibrary{arrows.meta}
    \usetikzlibrary{tikzmark}
    \usetikzlibrary{patterns}
\usepackage{pgfplots}
    \pgfplotsset{compat=1.18}
    \usepgfplotslibrary{groupplots}
    \usepackage{pgfplotstable}
    \usepgfplotslibrary{fillbetween}
     \usepackage{pgfplotstable}
 	\usepackage{etoolbox}
    
\pgfdeclarelayer{background layer}
\pgfdeclarelayer{foreground layer}
\pgfdeclarelayer{frontmost layer}
\pgfsetlayers{background layer, main, foreground layer, frontmost layer}

\usepackage{algorithm} \usepackage{algpseudocodex}

\newcounter{mylineno}
\renewcommand{\themylineno}{(\roman{mylineno})} 
\newcommand{\myline}[2]{\refstepcounter{mylineno}{\footnotesize#1\themylineno}:\hspace{0.5em}#2}

\usepackage{cancel}

\usepackage[most]{tcolorbox}

\usepackage{xspace} 

\usepackage{wrapfig}

\def\mydefb#1{\expandafter\def\csname bf#1\endcsname{\mathbf{#1}}}
\def\mydefallb#1{\ifx#1\mydefallb\else\mydefb#1\expandafter\mydefallb\fi}
\mydefallb aAbBcCdDeEfFgGhHiIjJkKlLmMnNoOpPqQrRsStTuUvVwWxXyYzZ\mydefallb

\def\mydefb#1{\expandafter\def\csname #1bb\endcsname{\mathbb{#1}}}
\def\mydefallb#1{\ifx#1\mydefallb\else\mydefb#1\expandafter\mydefallb\fi}
\mydefallb aAbBcCdDeEfFgGhHiIjJkKlLmMnNoOpPqQrRsStTuUvVwWxXyYzZ\mydefallb

\def\mydefb#1{\expandafter\def\csname #1cal\endcsname{\mathcal{#1}}}
\def\mydefallb#1{\ifx#1\mydefallb\else\mydefb#1\expandafter\mydefallb\fi}
\mydefallb aAbBcCdDeEfFgGhHiIjJkKlLmMnNoOpPqQrRsStTuUvVwWxXyYzZ\mydefallb

\def\mydefgreek#1{\expandafter\def\csname bf#1\endcsname{\text{$\boldsymbol{\csname #1\endcsname}$}}}
\def\mydefallgreek#1{\ifx\mydefallgreek#1\else\mydefgreek{#1}\lowercase{\mydefgreek{#1}}\expandafter\mydefallgreek\fi}
\mydefallgreek {alpha}{Alpha}{beta}{Beta}{gamma}{Gamma}{delta}{Delta}{epsilon}{Epsilon}{zeta}{Zeta}{eta}{Eta}{theta}{Theta}{iota}{Iota}{kappa}{Kappa}{lambda}{Lambda}{mu}{Mu}{nu}{Nu}{omicron}{Omicron}{pi}{Pi}{rho}{Rho}{sigma}{Sigma}{tau}{Tau}{upsilon}{Upsilon}{phi}{Phi}{xi}{Xi}{chi}{Chi}{psi}{Psi}{omega}{Omega}\mydefallgreek

\def\mydefb#1{\expandafter\def\csname T#1\endcsname{\boldsymbol{\mathcal{\MakeUppercase{#1}}}}}
\def\mydefallb#1{\ifx#1\mydefallb\else\mydefb#1\expandafter\mydefallb\fi}
\mydefallb aAbBcCdDeEfFgGhHiIjJkKlLmMnNoOpPqQrRsStTuUvVwWxXyYzZ\mydefallb

\DeclareMathOperator*{\argmin}{arg\, min}

\DeclareMathOperator{\myvec}{vec}

\DeclareMathOperator{\subjectto}{s.t.}

\definecolor{EmoryBlue}{RGB}{1, 33, 105} 
\definecolor{EmoryDarkBlue}{RGB}{12, 35, 64} 
\definecolor{EmoryMediumBlue}{RGB}{0, 51, 160} 
\definecolor{EmoryLightBlue}{RGB}{0, 125, 186} 
\definecolor{EmoryYellow}{RGB}{242, 169, 0} 
\definecolor{EmoryGold}{RGB}{181, 133, 0} 
\definecolor{EmoryMetallicGold}{RGB}{132, 117, 78} 

\definecolor{TuftsBlue}{RGB}{49,114,174}
\definecolor{TuftsBrown}{RGB}{94,75,60}

\definecolor{matplotlib0}{HTML}{1f77b4}
\definecolor{matplotlib1}{HTML}{ff7f0e}
\definecolor{matplotlib2}{HTML}{2ca02c}
\definecolor{matplotlib3}{HTML}{d62728}
\definecolor{matplotlib4}{HTML}{9467bd}
\definecolor{matplotlib5}{HTML}{8c564b}
\definecolor{matplotlib6}{HTML}{e377c2}
\definecolor{matplotlib7}{HTML}{7f7f7f}
\definecolor{matplotlib8}{HTML}{bcbd22}
\definecolor{matplotlib9}{HTML}{17becf}

\definecolor{matlab1}{rgb}{0.0000, 0.4470, 0.7410}
\definecolor{matlab2}{rgb}{0.8500, 0.3250, 0.0980}
\definecolor{matlab3}{rgb}{0.9290, 0.6940, 0.1250}
\definecolor{matlab4}{rgb}{0.4940, 0.1840, 0.5560}
\definecolor{matlab5}{rgb}{0.4660, 0.6740, 0.1880}
\definecolor{matlab6}{rgb}{0.3010, 0.7450, 0.9330}
\definecolor{matlab7}{rgb}{0.6350, 0.0780, 0.1840}

\newcommand{\DataSpace}{\Dcal}

\newcommand{\NSamp}{N}
\newcommand{\InputSpace}{\Xcal}
\newcommand{\NInput}{n_{\rm in}}
\newcommand{\TargetSpace}{\Ycal}
\newcommand{\NTarget}{n_{\rm target}}

\newcommand{\NFeat}{n_{\rm feat}}
\newcommand{\ThetaSpace}{\Theta}
\newcommand{\NTheta}{n_{\theta}}
\newcommand{\WSpace}{\Wcal}
\newcommand{\NW}{n_{w}}
\newcommand{\FctnClass}{\Fcal}

\newcommand{\Reduced}[1]{{#1^\downarrow}}
\newcommand{\SAA}[1]{#1_{\NSamp}}
\newcommand{\Taylor}[1]{#1^{(m)}}
\newcommand{\Opt}[1]{#1_{\star}}

\newcommand{\ObjFunctional}{\Jcal}
\newcommand{\ObjFctn}{J}

\newcommand{\RegFunctional}{\Rcal}
\newcommand{\RegFctn}{R}
\newcommand{\WOpt}{\bfw_{{\color{black}\star}}}
\newcommand{\WOptTaylor}{\Taylor{\Opt{\bfw}}}

\newcommand{\LossFunctional}{\Lcal}
\newcommand{\LossFctn}{L}

\newcommand{\lossfctn}{\ell}

\newcommand{\Featurizer}{z}
\newcommand{\FeaturizerMat}{A}

\newcommand{\caRho}{\rho^{(m)}}
\newcommand{\caLambda}{\lambda_w^{(m)}}

\newcommand{\caDeltaLam}{\Delta(\caLambda)}

\newcommand{\caL}{\SAA\LossFunctional[f^{(m)}]}
\newcommand{\caDL}{\nabla\caL}
\newcommand{\caDDL}{\nabla^2\caL}

\newcommand{\caAmat}{\FeaturizerMat_{\bftheta}}

\newcommand{\caAmatm}{\FeaturizerMat_{\bftheta}^{(m)}}

\newcommand{\caA}{\caAmat(\cdot)}
\newcommand{\caAm}{\caAmatm(\cdot)}

\newcommand{\caW}{\WOptTaylor(\bftheta)}
\newcommand{\caWLam}{\WOptTaylor(\bftheta; \caLambda)}

\newcommand{\caG}{{\bfg}_{\bftheta}^{(m)}}
\newcommand{\caH}{{\bfH}_{\bftheta}^{(m)}}
\newcommand{\caHLam}{\caH +\caLambda \bfI_{\NW}}

\newcommand{\cah}{h_\star^{(m)}} 

\newcommand{\caQ}{\SAA\Qcal^{(m)}}
\newcommand{\caQmat}{\SAA Q^{(m)}}

\newcommand{\cahN}{h_{\diamond}^{(m)}}

\newcommand{\XGBoost}{\texttt{XGBoost}\xspace}
\newcommand{\VPBoost}{\texttt{VPBoost}\xspace}
\newcommand{\AdaBoost}{\texttt{AdaBoost}\xspace}

\newcommand{\GDBoost}{\texttt{GDBoost}\xspace}

\crefname{lemma}{Lemma}{Lemmas}
\crefname{assumption}{Assumption}{Assumptions}
\crefname{proposition}{Proposition}{Propositions}
\crefname{remark}{Remark}{Remarks}
\crefname{corollary}{Corollary}{Corollaries}
\crefname{definition}{Definition}{Definitions}
\crefname{conjecture}{Conjecture}{Conjectures}
\crefname{axiom}{Axiom}{Axioms}

\Crefname{lemma}{Lemma}{Lemmas}
\Crefname{assumption}{Assumption}{Assumptions}
\Crefname{proposition}{Proposition}{Propositions}
\Crefname{remark}{Remark}{Remarks}
\Crefname{corollary}{Corollary}{Corollaries}
\Crefname{definition}{Definition}{Definitions}
\Crefname{conjecture}{Conjecture}{Conjectures}
\Crefname{axiom}{Axiom}{Axioms}

\newcommand{\R}{\mathbb{R}}

\usepackage[textsize=footnotesize]{todonotes}
\usepackage{tabularx}
\usepackage{array}
\usepackage{makecell}

\newcolumntype{C}[1]{>{\centering\arraybackslash}p{#1}}

\definecolor{matplotlib0}{HTML}{1f77b4}
\definecolor{matplotlib1}{HTML}{ff7f0e}
\definecolor{matplotlib2}{HTML}{2ca02c}
\definecolor{matplotlib3}{HTML}{d62728}
\definecolor{matplotlib4}{HTML}{9467bd}
\definecolor{matplotlib5}{HTML}{8c564b}
\definecolor{matplotlib6}{HTML}{e377c2}
\definecolor{matplotlib7}{HTML}{7f7f7f}
\definecolor{matplotlib8}{HTML}{bcbd22}
\definecolor{matplotlib9}{HTML}{17becf}

\colorlet{colorGB}{matplotlib0}
\colorlet{colorVP}{matplotlib3}
\colorlet{colorVPStart}{matplotlib2}
\colorlet{colorVPEnd}{matplotlib4}
\colorlet{colorVPStartEnd}{matplotlib8}
\colorlet{colorFullGB}{matplotlib5}
\colorlet{colorFullVP}{matplotlib6}
\colorlet{colorXGBoost}{black}

\pgfplotsset{
	styleShared/.style={thick, line width=1.5pt},
	styleGBBoost/.style={colorGB, mark=o},
	styleVPBoost/.style={colorVP, mark=square},
	styleVPStart/.style={colorVPStart, mark=triangle},
	styleVPEnd/.style={colorVPEnd, mark=pentagon*},
	styleVPStartEnd/.style={colorVPStartEnd, mark=diamond*},
styleXGBoostBase/.style={colorXGBoost, dashed, line width=1.5pt} }

\pgfplotsset{
	styleShared/.style={thick, line width=1.5pt},
	style_gb/.style={styleShared, colorGB, mark=o},
	style_vp/.style={styleShared, colorVP, mark=square},
	style_vp_at_start/.style={styleShared, colorVPStart, mark=triangle},
	style_vp_at_end/.style={styleShared, colorVPEnd, mark=pentagon*},
	style_vp_at_start_and_end/.style={styleShared, colorVPStartEnd, mark=diamond*},
	style_full_gb/.style={styleShared, colorGB, mark=*},
	style_full_vp/.style={styleShared, colorVP, mark=square*},
style_xgboost/.style={styleShared, colorXGBoost, dashed, line width=1.5pt} }

\pgfplotstableset{
        create on use/methodname/.style={
        create col/set list={\GDBoost, VP@Start, VP@End, VP@Start+End, \VPBoost, \XGBoost}},
        columns/methodname/.style={column type=l, string type, column name={}}
}

\pgfplotstableset{
        create on use/methodnamewfull/.style={
        create col/set list={\GDBoost, VP@Start, VP@End, VP@Start+End, \VPBoost, \XGBoost, Full NN, Full NN+VP}},
        columns/methodnamewfull/.style={column type=l, string type, column name={}}
}

\pgfplotstableset{
        create on use/method/.style={
        create col/set list={gb, vp_at_start, vp_at_end, vp_at_start_and_end, vp, xgboost}},
        columns/method/.style={column type=l, string type, column name={}}
}

\pgfplotstablenew[columns={method, methodname}]{6}\boostingmethodtable

\newcommand{\targetplotB}[2]{
    \addplot [
		#2, forget plot,
	] table[x=learner, y=target_metric_mean, col sep=comma] {#1}; 
\addplot [name path=upper, draw=none, no marks, forget plot] 
		table[x=learner, y expr=\thisrow{target_metric_mean}+\thisrow{target_metric_std}, col sep=comma] {#1};
	\addplot [name path=lower, draw=none, no marks, forget plot] 
		table[x=learner, y expr=\thisrow{target_metric_mean}-\thisrow{target_metric_std}, col sep=comma] {#1};
	\addplot [#2, opacity=0.25, forget plot] fill between[of=upper and lower];
}

\newcommand{\createcolmeanstd}[3]{
    \pgfplotstablecreatecol[
       create col/assign/.code={\getthisrow{#2_mean}\vala
         \getthisrow{#2_std}\valb
         \edef\newentry{$\pgfmathprintnumber{\vala} \pm \pgfmathprintnumber{\valb}$}\pgfkeyslet{/pgfplots/table/create col/next content}\newentry
       }]{#3}{#1}
   }
   
 \newcommand{\createcolmean}[3]{
    \pgfplotstablecreatecol[
       create col/assign/.code={\getthisrow{#2_mean}\vala
         \edef\newentry{$\pgfmathprintnumber{\vala}$}\pgfkeyslet{/pgfplots/table/create col/next content}\newentry
       }]{#3}{#1}
   }

\pgfplotstableset{
test_metrics_table_style/.style={
every head row/.style={before row=\toprule, after row=\midrule},
         every last row/.style={after row=\bottomrule},
every row 0 column 0/.style={
            postproc cell content/.append style={
                /pgfplots/table/@cell content/.add={\cellcolor{colorGB!25}}{},
            }
        },
        every row 1 column 0/.style={
            postproc cell content/.append style={
                /pgfplots/table/@cell content/.add={\cellcolor{colorVPStart!25}}{},
            }
        },
        every row 2 column 0/.style={
            postproc cell content/.append style={
                /pgfplots/table/@cell content/.add={\cellcolor{colorVPEnd!25}}{},
            }
        },
        every row 3 column 0/.style={
            postproc cell content/.append style={
                /pgfplots/table/@cell content/.add={\cellcolor{colorVPStartEnd!25}}{},
            }
        },
    },
    every row 4 column 0/.style={
            postproc cell content/.append style={
                /pgfplots/table/@cell content/.add={\cellcolor{colorVP!25}}{},
            }
        },
}

\pgfplotstableset{
    my number style/.style={
        every column no/.style={
            /pgf/number format/.cd,
            precision=4,
            fixed,
            fixed zerofill,
            1000 sep={.},
        }, 
        columns/0/.style={string type},
    },
}

\pgfplotstableset{
    columns/0/.style={string type},
    every col no/.style={
        /pgf/number format/.cd,
        precision=4,
        fixed,
        fixed zerofill,
        1000 sep={.},
    },
}

 \usepackage{lineno}

\usepackage{lastpage}
\jmlrheading{26}{2026}{1-\pageref{LastPage}}{3/20; Revised MM/DD}{MM/DD}{21-0000}{Abhijit Chowdhary, Elizabeth Newman, and Deepanshu Verma}

\ShortHeadings{VPBoost}{Chowdhary, Newman, and Verma}
\firstpageno{1}

\begin{document}

\title{Boost Like a (Var)Pro: Trust-Region Gradient Boosting via Variable Projection}

\author{\name Abhijit Chowdhary \email abhijit.chowdhary@tufts.edu\\
       \addr Department of Mathematics\\
       Tufts University\\
       Medford, MA 02155, USA
       \AND
       \name Elizabeth Newman \email e.newman@tufts.edu \\
       \addr Department of Mathematics\\
       Tufts University\\
       Medford, MA 02155, USA
       \AND
       \name Deepanshu Verma \email dverma@clemson.edu \\
       \addr School of Mathematical and Statistical Sciences\\
	Clemson University\\
	Clemson, SC 29634
       }

\editor{My editor}

\maketitle

\begin{abstract}

Gradient boosting, a method of building additive ensembles from weak learners, has established itself as a practical and theoretically-motivated approach to approximate functions, especially  using decision tree weak learners. 
Comparable methods for smooth parametric learners, such as neural
networks, remain less developed in both training methodology and theory. 
To this end,  we introduce \texttt{VPBoost} ({\bf V}ariable {\bf P}rojection {\bf Boost}ing), a gradient
boosting algorithm for separable smooth approximators, i.e., models with a smooth
nonlinear featurizer followed by a final linear mapping.  
\texttt{VPBoost} fuses variable projection, a training paradigm for separable models that enforces optimality of the linear weights, with a second-order weak learning strategy. 
The combination of second-order boosting, separable models, and variable projection give rise to a closed-form solution for the optimal linear weights and a natural interpretation of \VPBoost as a functional trust-region method. 
We thereby leverage trust-region theory to prove \VPBoost converges to a stationary point under mild geometric conditions and, under stronger assumptions, achieves a superlinear convergence rate.   
Comprehensive numerical experiments on synthetic data, image recognition, and scientific machine learning benchmarks demonstrate that \VPBoost learns an ensemble with improved evaluation metrics in comparison to gradient-descent-based boosting and attains competitive performance relative to an industry-standard decision tree boosting algorithm.

\end{abstract}


\section{Introduction} 
\label{sec:introduction}

Global weather forecasting, personalized large-language models, ground-breaking
medical discovery, and beyond: machine learning is tackling the most ambitious
and exciting scientific challenges of today. 
The complexity of these challenges has propelled the deployment of massive
data-driven models, which require equally massive computational resources to
train. 
The behind-the-scenes design and training of machine learning
architectures is itself a computationally demanding task; one must solve a
high-dimensional, non-linear, non-convex optimization problem and calibrate over multiple trials.  
Even upon surmounting the training task, there is computational, fiscal, and environmental pressure to accelerate the evaluation of such models.  
This work takes on these challenges by pairing practical model design with fast, reliable, structure-exploiting training schemes.

Through the lens of gradient boosting~\citep{friedman_greedy_2001}, we construct
an additive ensemble by sequentially training \emph{weak learners}, lightweight
models with limited predictive power.
By design, the smaller scale of weak learners can reduce the cost of training and the additive structure can potentially be exploited for more efficient ensemble evaluation. 
Modern boosting implementations dominate tabular data benchmarks \citep{chen_xgboost_2016, ke_lightgbm_2017, prokhorenkova_catboost_2018} and are supported by well-established convergence theory~\citep{friedman_greedy_2001, friedman_stochastic_2002}. 

This marked success has propelled boosting beyond decision trees to other structured weak learner families, including modern smooth, nonlinear, parameterized weak learners (e.g., neural networks). 
However, even for small weak learners, parameterized models are notoriously difficult to train.   
Standard derivative-based methods must navigate a complex optimization landscape in parameter space and the nonlinearity makes rigorous convergence analysis challenging.   
To simplify the optimization landscape, we impose a practical \emph{separable structure} on each weak learner. 
Separable models encompass a broad class of smooth approximators consisting of a nonlinear featurizer followed by a final linear mapping. 
One can exploit the linearity through partial optimization and thereby reduce the optimization problem to one over the nonlinear parameters only. 
The process of eliminating the linear weights is known as \emph{variable projection (VarPro)}~\citep{golub_differentiation_1973}. 
Notably, training with VarPro improves the conditioning of the optimization problem~\citep{sjoberg_separable_1997, newman_train_2021}.

The additive ensemble structure pairs beautifully with VarPro-trained separable weak learners when using second-order gradient boosting. 
The second-order formulation leads to an optimization problem in parameter space that is quadratic in the linear weights. 
This is a windfall for VarPro, which admits a closed-form expression for the optimal linear weights when applied to a quadratic function.  
Second-order boosting, separable weak learner ensembles, and variable projection come together seamlessly
and give rise to our new algorithm \VPBoost (\textbf{V}ariable \textbf{P}rojection \textbf{Boost}ing).
Our main contributions are outlined below. 
	\begin{itemize}
		\item {\bf Unified Framework:} Gradient boosting operates over a function space and VarPro operates in parameter space. 
		\VPBoost serves as the bridge between the two spaces, enabling optimization in parameter space to align with updates in function space.  
		To the best of our knowledge, this is the first time such a connection has been made explicitly and realized algorithmically. 
		
		\item {\bf Robust Boosting (Section~\ref{sec:vpboost}):}  
		Fundamentally, a second-order boosting strategy determines a candidate weak learner by minimizing a local quadratic model. 
		The optimal linear parameters of separable VarPro weak learners effectively control the scale and directionality of the candidate to be added to the ensemble. 
		Thus, \VPBoost operates as a trust-region method in function space and inherits the adaptivity characteristic of such methods, such as the mechanisms to adjust the weak learner scale autonomously.  
		In the context of machine learning, this adaptivity also reduces the number of hyperparameters one has to tune, further simplifying the underlying optimization problem.

		\item {\bf Convergence Guarantees (Section~\ref{sec:convergence_analysis})}: 
      Using the functional trust-region framework, 
      we prove \VPBoost converges to a stationary point under so-called \emph{subspace
      regularity} conditions. 
      These mild conditions ensure that \VPBoost bridges the parameter-function space gap. 
      Unlike typical ensemble-level weak learning conditions in recent literature, subspace regularity is imposed upon each weak learner during training and enforces weak learner design on the fly. 
      To the best of our knowledge, the subspace regularity conditions and resulting convergence analysis are new and distinctive to \VPBoost.

		\item {\bf Empirical Success (Section~\ref{sec:numerical_experiments}):} Across small-scale synthetic tasks,  image recognition, 
large-scale tabular data classification, and a high-dimensional scientific
machine learning benchmark, \VPBoost consistently achieves the top performance, even compared to  industry-standard decision-tree algorithm. 
Results hold across multiple featurizer architectures, 
demonstrating the flexibility of the separable smooth approximator framework. 
		
\end{itemize}

The remainder of this paper is organized as follows.
Section~\ref{sec:literature_review} provides an extensive literature review of related work on gradient boosting and
neural network boosting and positions \VPBoost relative
to existing methods.
Section~\ref{sec:background} establishes the mathematical framework, including
the function-space formulation of gradient boosting, the separable model class,
and the variable projection theory that underpins our approach.
Section~\ref{sec:vpboost} introduces the \VPBoost algorithm and its functional
trust-region interpretation at the ensemble level.
Section~\ref{sec:convergence_analysis} establishes convergence guarantees,
including the subspace regularity condition and the results on stationarity.
Section~\ref{sec:numerical_experiments} presents numerical experiments on
regression and classification benchmark problems
comparing \VPBoost against gradient descent-based boosting of neural networks and an industry-standard decision-tree algorithm.
Section~\ref{sec:conclusion} concludes with directions for future work.
     \subsection{Literature Review}
\label{sec:literature_review}
At a high level, boosting refers to algorithms that construct a strong learner (i.e., an accurate model) by sequentially building an additive ensemble of weak learners (i.e., models that are slightly more predictive than random guessing) \citep{schapire_strength_1990}. 
It is well-documented that, with a sufficiently expressive class of weak learners \citep{vapnik_uniform_1971}, boosted models can reduce bias/overfitting, increase predictive power/generalizability compared to a single strong learner, and improve computational efficiency of model evaluation \citep{zhou_ensemble_2012, opitz_popular_1999}.

Gradient boosting (GB) is a subclass of weak learning algorithms in which one builds an ensemble greedily by optimizing subsequent weak learners to reduce the current model error  \citep{friedman_greedy_2001, smola_functional_2000, friedman_stochastic_2002, he_gradient_2019, daouia_optimization_2021}. 
Conveniently, many popular boosting algorithms, such as the distribution-aware \AdaBoost (Adaptive Boosting) \citep{freund_experiments_1996}, can be re-framed as GB for a particular objective functional \citep{mason_boosting_1999}.  
\citet{chen_xgboost_2016} recognized the opportunity to accelerate GB using second-order information and developed  \XGBoost (eXtreme Gradient Boosting), a highly-optimized, open-access library for boosting decision trees (DTs). 
It remains one of the most successful and widely-used DT boosting algorithms to date.\footnote{\url{https://github.com/dmlc/xgboost/tree/master/demo}}
While, historically, GB grew popular for nonparametric, piecewise DTs, the techniques are applicable to ensembles comprised of parametrized, smooth learners, including radial basis functions \citep{powell_radial_1987}, splines \citep{schmid_boosting_2008}, wavelets \citep{daubechies_nonlinear_1993}, and neural networks (NNs) \citep{rumelhart_learning_1986}.

The modern computational challenges of training large-scale machine learning models have renewed interest in GB, particularly for NNs \citep{rambhatla_empirical_2022}.   
It has been shown that boosted NNs can outperform boosted decision trees\footnote{Currently, tree-based learning remains state-of-the-art for medium-sized tabular data tasks \citep{grinsztajn_why_2022}.} even, in some cases, on tabular data with sufficient regularity \citep{opitz_popular_1999, popov_neural_2019, mcelfresh_when_2024}.
Theoretical advancements have connected GB to residual NNs \citep{he_deep_2016} through layer-wise boosting accompanied by convergence guarantees and generalization bounds \citep{veit_residual_2016, huang_learning_2018, nitanda_functional_2018}. 
Algorithmically, early works directly applied \AdaBoost to fully-connected NNs \citep{schwenk_adaboosting_1997, schwenk_boosting_2000} for classification tasks. 
More recently, state-of-the-art GB algorithms have been adapted to particular NN architectures, including for convolutional NNs \citep{frazao_weighted_2014, moghimi_boosted_2016, taherkhani_adaboost-cnn_2020, rahul_boosted_2023}, graph NNs \citep{ivanov_boost_2020, oono_optimization_2021}, and physics-informed NNs \citep{raissi_physics-informed_2019, fang_ensemble_2024}.  
Our new boosting algorithm, \VPBoost, is tailored to separable smooth approximators (i.e., nonlinear featurizer + linear mapping) that arise in NNs and beyond.   

\VPBoost accelerates boosting of separable models by exploiting two core principles: (i) incorporate second-order information during weak learning training and (ii) exploit the convexity of the training objective with respect to the final linear mapping. 
To our knowledge, \VPBoost is the first boosting algorithm to integrate these principles accompanied by convergence guarantees. 
However, the opportunity to benefit from each principle individually has been recognized in the literature.

\paragraph{Second-Order Boosting}   
For DTs, Newton-based boosting has been well-explored \citep{friedman_greedy_2001, sigrist_gradient_2021, zheng_functional_2012}, with \XGBoost being one of the most portable and scalable variants~\cite{chen_xgboost_2016}. 
Recently, \citet{luo_trboost_2023} proposed a generic trust-region-style boosting algorithm for DTs, \texttt{TRBoost}, to take advantage of the ability of trust-region methods to optimize non-convex objective functions. 
Beyond DTs, \citet{badirli_gradient_2020} developed \texttt{GrowNet} for second-order boosting of NNs. 
\texttt{GrowNet} utilized additional performance-enhancing features, including passing previous learner features as inputs into subsequent learners and re-training the ensemble globally after the addition of a new learner. 

 \VPBoost is a second-order NN boosting algorithm that is naturally represented as a functional trust-region algorithm for smooth learners. 
 Unlike \texttt{GrowNet}, \VPBoost trains each weak learner once and does not explicitly pass prior feature information as inputs to new weak learners. 
 As a result, the \VPBoost ensembles are truly additive, leading to efficient model evaluation, and the training cost remains low and dependent only on the size of a weak learner.

\paragraph{Boosting with Optimal Linear Weights} 
The benefits of optimizing the final linear weights have been explored in the context of architecture-adaptive NN boosting.  
\citet{bengio_convex_2005} proposed an incremental convex NN algorithm that sequentially built a weighted ensemble of linear classifiers where, for each new classifier, the weights were re-optimized through convex minimization. 
\cite{cortes_adanet_2017} re-framed boosting as a tool to adapt NN structure in \texttt{AdaNet}. 
The central two-step iteration augmented a base NN architecture with a subnetwork of greater depth and then optimized a final sparse linear layer based on a Rademacher-like complexity bound. 
Other works have explored boosting as a tool to adapt underlying NN architectures
through functional gradient descent in an infinite-dimensional space.
\citet{atsushi_nitanda_functional_2020} proposed \texttt{ResFGB}, which treated each
residual block as a functional gradient and stacked them sequentially, growing the
network deeper with each boosting iteration. 
In their framework, the linear weights were trained once at initialization and then held fixed throughout all boosting iterations, reducing the problem to optimizing the featurizer alone in function space.

Instead of pre- or post-optimizing the linear layer, \VPBoost embeds the optimal linear weights, determined analytically, directly into each weak learning training iteration. 
This gives rise to provably better optimization properties along a lower-dimensional manifold (\Cref{sec:varpro}). 
As presented, \VPBoost is not a structurally-adaptive algorithm and instead favors extremely small weak learners with few dense linear weights.

\paragraph{Theoretical Guarantees for Boosting} 

Convergence guarantees and associated rates have been well-documented in GB literature. 
In many cases, GB theoretical analysis begins with a weak learning condition to describe the expressabilty of an ensemble. 
For example, \citet{lu_randomized_2020} introduced the Minimal Cosine Angle (MCA), which measures the density of weak learners in function space, and subsequently proved convergence with a linear rate of their randomized GB algorithm, \texttt{RGBM}. 
For NNs, 
\citet{atsushi_nitanda_functional_2020} subsequently improved upon \texttt{ResFGB} with
\texttt{ResFGB-FW} with the Frank-Wolfe method and new theoretical analysis based on a margin condition, an assumption on how well functions from a learnable hypothesis class can separate data with different labels.  
This milder condition led to a sublinear convergence rate,  improved statistical guarantees, and was extended
to graph NNs in \citet{oono_optimization_2021}.

In comparison, \VPBoost requires no ensemble-level weak learning condition and no assumption on the data distribution; descent is guaranteed automatically at every iteration, and convergence follows from mild geometric conditions on the featurizer that we term \emph{subspace regularity}.  
The analysis is applicable to both regression-based and classification-based tasks and extends beyond NNs to any smooth separable approximator.

Table~\ref{tab:related_work} summarizes the algorithmic design choices and theoretical guarantees of the methods discussed above.

\begin{table}[t]
\centering
\small
\setlength{\tabcolsep}{5pt}
\begin{tabular}{@{}l c cc c @{\hspace{2pt}} ccc@{}}
\toprule
& \multicolumn{3}{c}{\textit{Algorithm}}
& 
& \multicolumn{3}{c}{\textit{Theory}} \\
\cmidrule(lr){2-4}\cmidrule(l){6-8}
Method & Linear Weights & Ensemble & $2^\circ$ && Convergence & Rate & Condition \\
\midrule
\multicolumn{8}{@{}l}{\textit{Classical gradient boosting}} \\[2pt]
\quad GBM~[\citenum{friedman_greedy_2001}]
  & ---  & Additive & $\times$ && --- & --- & --- \\
\quad RGBM~[\citenum{lu_randomized_2020}]
  & ---  & Additive & $\times$ && Global min. & Linear & MCA \\
\quad AGBM~[\citenum{lu_accelerating_2020}]
  & ---  & Additive & $\times$ && Global min. & Accel. & MCA \\
\quad TRBoost~[\citenum{luo_trboost_2023}]
  & ---  & Additive & \checkmark && Global min. & Linear & MCA \\
\quad \XGBoost~[\citenum{chen_xgboost_2016}]
  & --- & Additive & \checkmark && $^{\dagger}$ & --- & --- \\[4pt]
\multicolumn{8}{@{}l}{\textit{Neural network boosting}} \\[2pt]
\quad Bengio et al.~[\citenum{bengio_greedy_2006}]
  & Post-hoc  & Compositional & $\times$ && --- & --- & --- \\
\quad AdaNet$^{\star}$~[\citenum{cortes_adanet_2017}]
  & Post-hoc & Architectural & $\times$ && Gen. bound & --- & --- \\
\quad GrowNet$^{\star}$~[\citenum{badirli_gradient_2020}]
  & Post-hoc & Additive & \checkmark && --- & --- & --- \\
\quad ResFGB~[\citenum{nitanda_functional_2018}]
  & Decoupled & Compositional & $\times$ && Gen. bound & --- & --- \\
\quad ResFGB-FW~[\citenum{atsushi_nitanda_functional_2020}]
  & Decoupled & Compositional & $\times$ && Global min. & Sublinear & Margin cond. \\
\quad \VPBoost~(ours)
  & Integrated & Additive & \checkmark && Stationary pt. & --- & Subspace reg. \\
\midrule 
\midrule
\multicolumn{8}{@{}p{0.95\textwidth}}{\footnotesize
$^{\dagger}$~[\citenum{ki_what_2026}] establish a near-minimax rate for the
 least-squares estimator over the \XGBoost function class; whether the \XGBoost
 algorithm itself converges to this optimum is an open problem. 
$^{\star}$~Weak learner inputs include outputs of previous weak learners.
}\\
\bottomrule
\end{tabular}
{ \caption{Comparison of gradient boosting methods.} \label{tab:related_work}}
\end{table}     
\section{Background}
\label{sec:background}

Consider a data space $\DataSpace \subset \InputSpace \times \TargetSpace$
with input space $\InputSpace \subset \Rbb^{\NInput}$ and target space
$\TargetSpace \subset \Rbb^{\NTarget}$, and let $y: \InputSpace \to \TargetSpace$
be the input-target pairing function (e.g., a labeling function for
classification).
We seek a function $f: \InputSpace \to \Rbb^{\NTarget}$ that satisfies $f(\bfx) =
y(\bfx)\in \Rbb^{\NTarget}$ for all input-target pairs $(\bfx, y(\bfx))\in
\DataSpace$. 
We assume $f$ belongs to a hypothesis class $\FctnClass \coloneqq L^2(\InputSpace, \mu; \Rbb^{\NTarget})$, the Hilbert space
of square-$\mu$-integrable, $\Rbb^{\NTarget}$-valued functions on $\InputSpace$.
We pose the learning problem as \begin{align}
  \label{eq:expected_loss}
  \min_{f\in \FctnClass} \ObjFunctional[f]
  \equiv \LossFunctional[f] + \lambda \RegFunctional[f]
  \quad \text{with} \quad
  \LossFunctional[f]
  = \int_{\InputSpace} \lossfctn\bigl(f(\bfx), y(\bfx)\bigr)\, d\mu(\bfx),
\end{align}
where $\ObjFunctional: \FctnClass \to \Rbb$ is the objective functional,
$\LossFunctional: \FctnClass \to \Rbb$ is the loss functional, 
$\lossfctn: \Rbb^{\NTarget} \times \TargetSpace \to \Rbb$ is the loss function per datum, 
$\RegFunctional: \FctnClass \to \Rbb$ is the regularization functional.
Here, $\LossFunctional$ measures the approximation error and $\RegFunctional$ promotes desirable function properties (e.g., smoothness).   
The regularization parameter $\lambda > 0$ balances the data fit and regularization.

This work considers a deterministic variant of \eqref{eq:expected_loss} via
a sample average approximation (SAA).
Given a finite training set $\DataSpace_N = \{(\bfx_i, \bfy_i)\}_{i=1}^N \subset \DataSpace$ with
$\bfy_i = y(\bfx_i)$ that is large and representative of the data space, we
replace $\mu$ by the empirical measure $\mu_N = \frac{1}{N}\sum_{i=1}^N
\delta_{\bfx_i}$ and redefine our hypothesis class as $\FctnClass \coloneqq
L^2(\InputSpace, \mu_N; \Rbb^{\NTarget})$.
The resulting SAA objective is
\begin{equation}\label{eq:saa_objective_functional}
\begin{alignedat}{3}
  & \min_{f\in \FctnClass}
  &&\SAA\ObjFunctional[f]
  &&\equiv \SAA\LossFunctional[f] + \lambda \RegFunctional[f]\\
  &\text{where} \quad
  &&\SAA\LossFunctional[f]
  &&= \int_{\InputSpace} \lossfctn\bigl(f(\bfx), y(\bfx)\bigr)\, d\mu_N(\bfx)
     = \frac{1}{N}\sum_{i=1}^N \lossfctn\bigl(f(\bfx_i), \bfy_i\bigr).
\end{alignedat}
\end{equation}
We impose the following smoothness and curvature conditions on the loss function, all of which are satisfied for common machine learning losses, including mean squared error and cross-entropy. 
\begin{assumption}[Loss Regularity]
  \label{assump:loss_functional}
  We assume the loss function $\ell: \Rbb^{\NTarget} \times \TargetSpace \to
  \Rbb$,
  \begin{enumerate}[label={\bf (L\arabic*)}, leftmargin=*]
    \item \label{assump:ell_bounded_below}
      is bounded from below, i.e., there exists some constant
      $\kappa_{\rm low}>-\infty$ such that
      $\lossfctn(\hat{\bfy}, \bfy) \ge \kappa_{\rm low}$ for all
      $(\hat{\bfy}, \bfy) \in \Rbb^{\NTarget} \times \TargetSpace$, and
    \item \label{assump:ell_twice_continuously_differentiable}
      is twice continuously differentiable with respect to
      $\hat{\bfy}$; i.e., $\nabla_{\hat{\bfy}}^2 \lossfctn(\hat{\bfy}, \bfy)$
      exists and is continuous in $\hat{\bfy}$ for every $(\hat{\bfy}, \bfy) \in
      \Rbb^{\NTarget} \times \TargetSpace$, and
    \item \label{assump:ell_convexity}
      is convex in its first argument, i.e.,
      $\nabla_{\hat{\bfy}}^2 \lossfctn(\hat{\bfy}, \bfy) \succeq 0$ for all
      $(\hat{\bfy}, \bfy) \in \Rbb^{\NTarget} \times \TargetSpace$, and
    \item \label{assump:ell_bounded_hessian}
      has a uniformly bounded Hessian with respect to the first 
      argument, i.e.,
      $\|\nabla_{\hat{\bfy}}^2 \lossfctn(\hat{\bfy}, \bfy)\| \le \beta$ for all
      $(\hat{\bfy}, \bfy) \in \Rbb^{\NTarget} \times \TargetSpace$ and some
      $\beta > 0$.
  \end{enumerate}
  These imply matching differentiability, convexity, and bounded-curvature
  properties on $\SAA\LossFunctional$ under the empirical measure.
\end{assumption}
Under Assumption~\ref{assump:loss_functional}, $\SAA\LossFunctional$ is
Fr\'{e}chet differentiable on $\FctnClass$, and by the Riesz representation
theorem, its derivative is represented by a unique
functional gradient $\nabla \SAA\LossFunctional[f] \in \FctnClass$, see
Appendix~\ref{app:functional_derivatives}. 
One can show, under the empirical measure $\mu_N$, the directional Fr\'{e}chet derivative
in direction $h \in \Fcal$ reduces to
\begin{align}\label{eq:functional_gradient_inner_product}
  \langle \nabla \SAA\LossFunctional[f],\, h \rangle
  = \frac{1}{N}\sum_{i=1}^N
    \nabla_{\hat{\bfy}}\lossfctn\bigl(f(\bfx_i), \bfy_i\bigr)^\top h(\bfx_i).
\end{align}

\subsection{Gradient Boosting 101}
\label{sec:gradient_boosting}

Ensemble learning is an approach for approximating a solution to
\eqref{eq:saa_objective_functional} by constructing a weighted additive model of $M + 1$ learners
$h^{(i)} \in \FctnClass$ for  $i = 0, \dots, M$; that is
\begin{align*}f \approx f^{(M)} = h^{(0)} + \sum_{m=1}^M \alpha^{(m)} h^{(m)}.
\end{align*} 
where $\alpha^{(m)} > 0$ for  $m = 1,\dots, M$ is the \emph{boosting rate} for the $m^{\rm th}$ weak learner (i.e., a step size in function space). 
Each additive component, $h^{(m)}$, is designed to be simple to optimize at the cost of
limited predictive power; hence, we call the components \emph{weak learners}.
Gradient boosting algorithms determine weak learners greedily. 
The procedure often commences by assigning the initial learner to be an optimal constant $h_0(\cdot) \equiv \bfc_0$ by solving
\begin{align*}
   h^{(0)} \in \argmin_{\text{$h\in \FctnClass^{(0)}$}} \SAA\ObjFunctional[h] \quad \text{where} \quad  \FctnClass^{(0)} = \{h \in \FctnClass \mid h(\cdot) \equiv \bfc \text{ for some } \bfc\in \Rbb^{\NTarget}\}. 
\end{align*}
Subsequent weak learners are trained to minimize the residual
\begin{align}\label{eq:grad_boost_opt_prob}
  h^{(m)} \in \argmin_{h \in \FctnClass^{(m)} \subset \FctnClass} \SAA\LossFunctional[f^{(m-1)} + h]+ \lambda^{(m)} \RegFunctional^{(m)}[h],
\end{align}
where $\FctnClass^{(m)}$ can, in principle, be some other hypothesis class varying in $m$.
We allow for a separate regularization functional, $\RegFunctional^{(m)}$, and
regularization parameter, $\lambda^{(m)} > 0$, per weak learner. 
In principle, one can view weak learners as descent directions in function space relative to the loss functional.
A natural choice of descent direction is $-\nabla \SAA\LossFunctional[f^{(m-1)}]$, hence the term ``gradient boosting'' is effectively a form of gradient descent in the function space.

\begin{wrapfigure}{r}{0.3\linewidth}
    \centering
    \includegraphics[width=\linewidth]{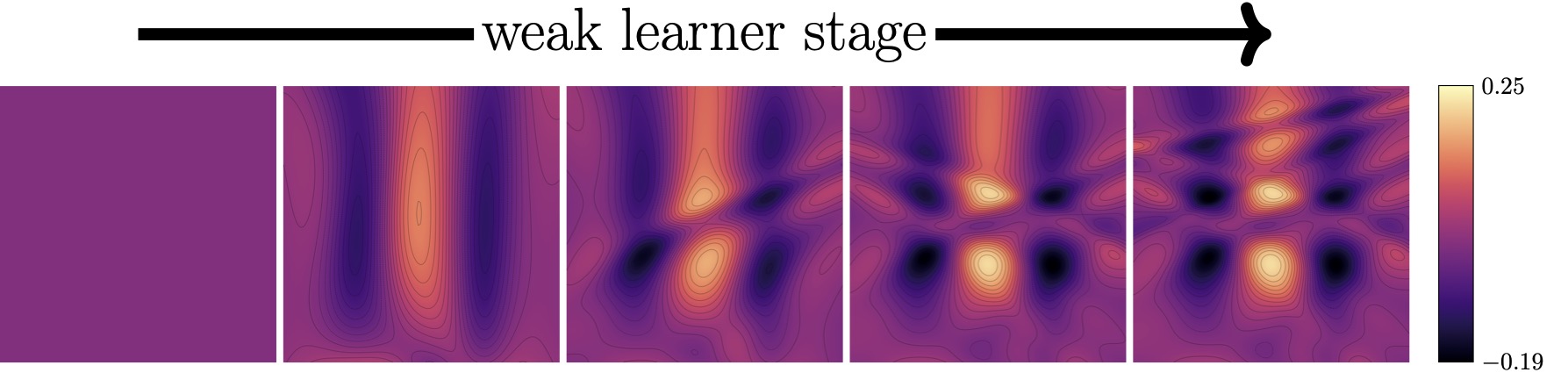}
    \caption{Weak learners progressively capture higher frequencies.}
    \label{fig:single_training_vp}
\end{wrapfigure}
This architecture, and method of construction, has proven fruitful over years of research.
As noted above, learners are typically designed to be small, primarily for the implicit practical purpose of regularization and combating overfitting.
Hence, compared to large-scale learners, training a gradient boosting ensemble is far less memory intensive, needing merely a single weak learner in memory at any given time.
Moreover, the individual optimization problems across the stages has relatively low dimensionality, often resulting in significantly less complex optimization landscapes. 
Finally, note that weak learners essentially minimize the residual at stage $m$.  
This loosely suggests that weak learners  decrease approximation error of progressively higher frequencies (Figure~\ref{fig:single_training_vp}). 
Similar ideologies inform residual learning~\citep{he_deep_2016}, multigrid~\citep{trottenberg_multigrid_2001}, multiscale learning~\citep{haber_learning_2018}, fine tuning~\citep{hu_lora_2021}, iterative architectures~\citep{he_iresnet_2025}, and more.

\subsection{Separable Weak Learners}
\label{sec:separable_models}

In order to solve \eqref{eq:grad_boost_opt_prob} in-practice, one often
selects a subclass of functions in $\Fcal$ that can be well-represented
computationally. 
The specific subclass is flexible, ranging from nonparametric decision trees
\citep{friedman_greedy_2001} to modern neural networks
\citep{badirli_gradient_2020}. 
In this work, we consider a specialized class of smooth, \emph{separable} parameterized
models, which consists of a nonlinear feature extractor followed by a linear
mapping. 
A linear combination of basis functions or neural networks with a linear head are
prototypical examples of a separable model.
We formalize separability in function space before appealing to the finite-dimensional
point of view used in practice.

Let $\FeaturizerMat_{\bftheta}: \InputSpace \to \Rbb^{\NTarget \times \NW}$
denote a nonlinear feature-extraction map, parameterized by $\bftheta \in
\Rbb^{\NTheta}$, that returns a linear operator acting on $\bfw \in \Rbb^{\NW}$.
Henceforth, we consider learners of the form
\begin{equation*}
h(\bfx; \bftheta, \bfw) = \FeaturizerMat_{\bftheta}(\bfx) \bfw.
\end{equation*} 
Here, $\FeaturizerMat_{\bftheta}$ is understood as an element of the Bochner
operator space $\Bcal = L^2(\InputSpace, \mu_{N};
\Rbb^{\NTarget\times\NW})$, equipped with the operator norm on the target space
$\Rbb^{\NTarget \times \NW}$.
Note, 
by the definition of the induced operator norm, one can show that separable models enjoy a 
submultiplicative property under these choices of norm.
\begin{lemma}\label{lem:operator_norm}
   Given Bochner space $\Bcal$ with induced operator norm
  $\|\FeaturizerMat_{\bftheta}\|_{\Bcal} \coloneqq \sup_{\|\bfw\|_2=1}
  \|\FeaturizerMat_{\bftheta}(\cdot)\bfw\|_{\FctnClass}$,
  \begin{align*}\|\FeaturizerMat_{\bftheta}(\cdot)\bfw\|_{\FctnClass}
    \le \|\FeaturizerMat_{\bftheta}\|_{\Bcal}\,\|\bfw\|_2,
    \qquad \text{for all $(\bfw, \FeaturizerMat_{\bftheta}) \in \Rbb^{\NW} \times \Bcal$}.
\end{align*}
\end{lemma}
Note, for the rest of this paper, we do not explicitly indicate which norm is
being considered. 
We always use the default norm per space: the $L^2$-norm for function space, the operator norm for Bochner space, and the $2$-norm for $\Rbb^n$. 
See Appendix~\ref{app:bochner} for details on the associated spaces.

\subsection{Optimizing Separable Weak Learners with Variable Projection (VarPro)}
\label{sec:varpro}

By parameterizing separable weak learners with weights $(\bfw, \bftheta)$, one can approximate the infinite-dimensional weak learner training problem in  \eqref{eq:grad_boost_opt_prob} by a finite-dimensional deterministic optimization problem in parameter space
\begin{equation}
\begin{alignedat}{3}\label{eq:separable_nn_objective}
  &\min_{\bfw\in \Rbb^{\NW}, \bftheta \in \Rbb^{\NTheta}} 
  &&\SAA\ObjFctn(\bfw, \bftheta) 
  &&\equiv \SAA\LossFctn(\bfw, \bftheta) + \lambda_w \RegFctn_w(\bfw) + \lambda_\theta \RegFctn_\theta(\bftheta)\\
  &\quad \text{where} \quad 
  &&\SAA\LossFctn(\bfw, \bftheta) 
  &&= \frac{1}{\NSamp} \sum_{i=1}^\NSamp \lossfctn(\FeaturizerMat_{\bftheta}(\bfx_i)\bfw, \bfy_i) .
\end{alignedat}
\end{equation}
Here, $\SAA \ObjFctn: \Rbb^{\NW} \times \Rbb^{\NTheta} \to \Rbb$ is the objective function, $\RegFctn_w: \Rbb^{\NW} \to \R$ and $\RegFctn_\theta: \Rbb^{\NTheta} \to
\R$ are appropriately defined regularization functions.  

Strategically exploiting the separability of NNs has been shown to accelerate
training and produce more
accurate models
\citep{newman_train_2021, cyr_robust_2020}. 
This is because the two blocks of weights, $\bfw$ and $\bftheta$, are
strongly coupled; that is, a good choice of linear mapping depends on the extracted features. 
Variable projection (VarPro)  explicitly
captures this coupling through partial optimization of the linear weights. 
Formally, VarPro transforms the full optimization
problem~\eqref{eq:separable_nn_objective} into a reduced, bi-level optimization
problem of the form
\begin{subequations}\label{eq:separable_nn_reduced_prob}
  \begin{alignat}{3}
    \label{eq:separable_nn_objective_reduced}
    &\min_{\bftheta \in \Rbb^{\NTheta}} \quad &\Reduced{\SAA\ObjFctn}(\bftheta) 
    &&&\equiv \SAA\ObjFctn(\WOpt(\bftheta), \bftheta)
    \\
    \label{eq:w_opt}
    & \subjectto \quad & \WOpt(\bftheta) 
    &&&\in \argmin_{\bfw\in \Rbb^{\NW}} \SAA\ObjFctn(\bfw, \bftheta).
  \end{alignat}
\end{subequations}
The function $\Reduced{\SAA\ObjFctn}$ is commonly referred to as
the reduced objective, and depends only on $\bftheta$.
\begin{wrapfigure}{t}{0.3\linewidth}
\centering
\includegraphics[width=\linewidth]{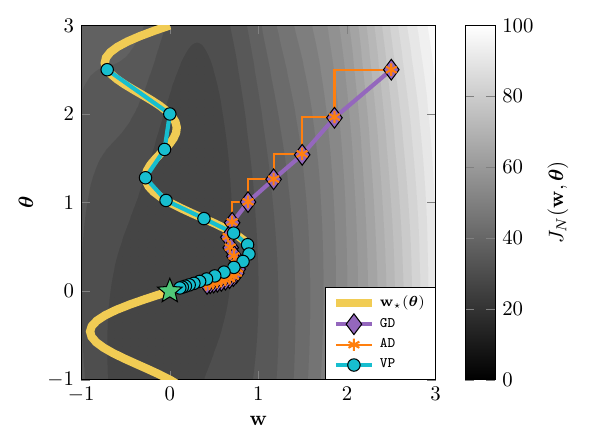}
\caption{Trajectories of gradient descent (\texttt{GD}), alternating directions (\texttt{AD}), and variable projection (\texttt{VP}). 
Geometrically, \texttt{VP} iterates traverse the curve $\bfw_{\star}(\bftheta)$.}
\label{fig:varpro_illustrative_example}
\end{wrapfigure}
Optimizing with $\Reduced{\SAA\ObjFctn}$ amounts to iterating over
a projection of $\Rbb^{\NW}\times \Rbb^{N_\theta}$ onto $\Rbb^{N_\theta}$ where each point in
the space is paired with its corresponding linear parameters in $\Rbb^{\NW}$, hence
the name ``variable projection.'' 
We provide an illustrative example to demonstrate the geometric
interpretation of optimization with VarPro in
Figure~\ref{fig:varpro_illustrative_example}.

Originally, VarPro was designed for separable nonlinear least
squares problems \citep{golub_differentiation_1973, kaufman_variable_1975,
oleary_variable_2013}.
It has since found success in inverse problems
\citep{chung_numerical_2006, espanol_variable_2023}, and has been extended to
manifolds \citep{noferini_nearest_2025}, rational approximation
\citep{hokanson_least_2018}, and nonsmooth settings
\citep{van_leeuwen_variable_2021}.
Neural network training has also benefited from VarPro in both deterministic
\citep{newman_train_2021, cyr_robust_2020} and stochastic settings
\citep{newman_slimtrain---stochastic_2022}.

Under appropriate assumptions on the objective function, the optimal linear weights,
$\WOpt(\bftheta)$ from~\eqref{eq:w_opt}, are uniquely defined and continuously
differentiable. 
\begin{assumption}[Objective Regularity] \label{assump:vp}
  We will assume $\SAA\ObjFctn: \Rbb^{\NW} \times \Rbb^{\NTheta} \to \Rbb$
	\begin{enumerate}[label={\bf (J\arabic*)}, leftmargin=*]
	\item  \label{assump:vp_smoothness} is twice continuously differentiable and $K$-smooth; that is, there exists a constant $K > 0$ such that 
        \begin{align*}
            \|\nabla \SAA\ObjFctn(\bfw, \bftheta) - \nabla \SAA\ObjFctn(\bfw', \bftheta')\|^2 \le K^2\left(\|\bfw - \bfw'\|^2 + \|\bftheta - \bftheta'\|^2\right),
        \end{align*}
    for any $(\bfw, \bftheta), (\bfw', \bftheta')\in \Rbb^{\NW} \times \Rbb^{\NTheta}$. 

  \item \label{assump:vp_convexity} is $\lambda_w$-strongly convex in its
    first argument; 
    that is, there exists $\lambda_w > 0$
such that
    \begin{align*}
      \nabla_{\bfw}^2 \SAA\ObjFctn(\bfw, \bftheta)
      \succeq \lambda_w \bfI_{\NW} \succ 0,
    \end{align*}
    for all $(\bfw, \bftheta) \in \Rbb^{\NW} \times \Rbb^{\NTheta}$. 
	\end{enumerate}
\end{assumption}

The continuous differentiability of $\bfw_{\star}(\bftheta)$ is a consequence of the Implicit Function Theorem
\citep{krantz_implicit_2002} applied to the gradient operator
$\nabla_{\bfw}\SAA\ObjFctn: \Rbb^{\NW} \times \Rbb^{\NTheta} \to \Rbb^{\NW}$. 
This leads to a powerful relationship between the gradients of the full and reduced objective functions.

\begin{restatable}[Gradient of Reduced Objective Function $\Reduced{\SAA\ObjFctn}$]{mylem}{reducedgradientequality}\label{lem:reduced_gradient_equality}
  Under Assumptions~\ref{assump:vp_smoothness} and~\ref{assump:vp_convexity},
  the gradient of the reduced objective function $\Reduced{\SAA\ObjFctn}$ is
  equal to the gradient of the full objective function $\SAA\ObjFctn$,
  evaluated at the optimal linear weights; that is, 
	\begin{align*}
    \nabla\Reduced{\SAA\ObjFctn}(\bftheta) 
    = \left.
      \nabla_{\bftheta} {\SAA\ObjFctn}(\bfw, \bftheta)
    \right|_{\bfw = \WOpt(\bftheta)}.
	\end{align*}
\end{restatable}

A key consequence of Lemma~\ref{lem:reduced_gradient_equality} is that
first-order derivative-based optimization algorithms for $\bftheta$ need
not differentiate through the optimal linear parameters $\WOpt(\bftheta)$.
It suffices to evaluate gradients of $\SAA\ObjFctn$ with respect to $\bftheta$
at $\WOpt(\bftheta)$.
Appendix~\ref{app:vp_vs_gd_proofs} provides a comprehensive suite of VarPro theory, including proof that the reduced optimization problem becomes no worse, and often better, conditioned than the full joint problem.

\begin{remark}
  Unlike classic gradient descent theory, we need not assume strong convexity in
  $\bftheta$ or jointly in $(\bfw,\bftheta)$; only in the linear block $\bfw$.
  Additionally, the notation for the strong-convexity constant $\lambda_w$
  intentionally overlaps with the regularization strength for $\bfw$ in
  \eqref{eq:separable_nn_objective}.
  Indeed, in Assumption~\ref{assump:loss_functional}, we assume that the loss
  functional is convex in its first argument.
 In the case of  Tikhonov regularization, $\Rcal_w(\bfw) = \|\bfw\|^2$, the strong convexity constant is precisely $\lambda_w$. 
\end{remark}

\subsection{A Note on Notation}

The central goal of this work is to synthesize gradient boosting and VarPro for effective ensemble generation. 
In doing so, we must go between infinite-dimensional function space, where gradient boosting operates, and finite-dimensional parameter space, where VarPro applies. 
To delineate the two spaces as clearly as possible, we use calligraphic letters to denote sets/spaces (e.g.,
$\DataSpace$,  $\FctnClass$, $\InputSpace$, $\TargetSpace$) or
functionals (e.g., $\ObjFunctional$, $\LossFunctional$, $\RegFunctional$) and lowercase Latin letters
denote functions (e.g., $f$, $h$, $\ell$). 
We reserve bold lowercase letters (Latin or Greek) to denote finite-dimensional vectors
(e.g., $\bfx$, $\bfy$,  $\bfw$,   $\bftheta$), bold uppercase letters
denote matrices (e.g., $\bfH$), and uppercase Latin letters denote functions operating on parameter space (e.g., $\ObjFctn$, $\LossFctn$, $\RegFctn$).

With few exceptions, common notation conventions will be used throughout the paper. 
Specifically  lowercase (Latin or Greek) letters will denote 
 scalars (e.g., $c$, $r$,  $\alpha$, $\beta$, $\gamma$) and indexing integers (e.g., $i$, $j$, $k$, $m$, $n$).  
Whenever possible, $i$ will index data samples, $j$ and $m$ will index weak learners, and $k$ will index optimization iterations. 
In general, we will use $n$ with an informative subscript to denote the
dimension of a particular vector space. 
We denote the $n\times n$ identity matrix by $\bfI_n$.  
We use $N$ to enumerate the training data samples and indicate the use of an SAA formulation. 
The gradient operator  in both function and parameter space will be denoted as $\nabla$. 
Whenever possible, we use $f$ to denote an ensemble and $h$ to denote a weak learner. 
Special notation of an uppercase Latin $\FeaturizerMat$ will be reserved for a matrix-valued featurizer.

\section{VarPro Gradient Boosting (\VPBoost)}
\label{sec:vpboost}

Variable projection gradient boosting (\VPBoost) greedily constructs an ensemble of separable weak learners whose linear weights are optimal; i.e., a \VPBoost ensemble at stage $m$ is of the form 
	\begin{align*}
		f^{(m)}(\cdot) = \bfc_0 + \sum_{j=1}^{m-1}\FeaturizerMat_{\bftheta}^{(j)}(\cdot)\bfw_\star^{(j)}(\bftheta).
	\end{align*}
We will call any separable weak learner with optimal linear weights a VarPro weak learner. 

Building a \VPBoost ensemble consists of two nested algorithms, one at the  \emph{weak learner/parameter space} level (Section~\ref{sec:weak_learner_training}) and one at the \emph{ensemble/function space} level (Section~\ref{sec:ensemble_training}).  
To distinguish between the two levels, we will use subscripts to denote dependence on parameter space iteration and superscripts to denote dependence on the ensemble stage. 
For example, at the $k^{\rm th}$ training iteration, a separable weak learner with weights $(\bfw_k, \bftheta_k)$ will be denoted as $\FeaturizerMat_{\bftheta_k}(\cdot) \bfw_k$. 
At the $m^{\rm th}$ weak learner stage,  we will remove the iteration counters on $\bfw$ and $\bftheta$ and write a separable weak learner as $\caAm \bfw^{(m)}$. 
Because the optimal linear weights will depend explicitly on the current ensemble, $f^{(m)}$, we will always indicate the dependence on the weak learner stage; i.e., $\bfw_{\star}^{(m)}(\bftheta_k)$ at the $k^{\rm th}$ training iteration and $\caW$ at the ensemble level.  
For later convenience, we will denote the regularization parameter for the linear weights as $\caLambda$, indicating a dependence on weak learner stage.
The rationale for this notation will be made clear in Section~\ref{sec:ensemble_training}.

\subsection{Training Weak Learners with VarPro}
\label{sec:weak_learner_training}
At the weak learner level, \VPBoost adopts a second-order strategy championed by the decision-tree-based \XGBoost~\citep{chen_xgboost_2016}. 
The main idea is to approximate the loss functional about the current ensemble, $f^{(m)}$, with a quadratic model functional $\caQ: \Fcal \to \Rbb$ defined as
	\begin{align}\label{eq:quadratic_functional}
		\Lcal_N[f^{(m)} + h] \approx \caQ[h] \coloneqq \caL + \langle \caDL, h \rangle + \tfrac{1}{2} \langle h, \caDDL h \rangle. 
	\end{align}
To determine the $m^{\rm th}$ separable weak learner, $h^{(m)}(\cdot) = \FeaturizerMat_{\bftheta}^{(m)}(\cdot) \bfw^{(m)}$, we minimize the quadratic functional combined with Tikhonov regularization,\footnote{Decision-tree specific regularization (e.g., number of leaves) is proposed in \XGBoost.} i.e., $\Rcal[h] = \|h\|^2$. 
Specifically,  we solve
	\begin{align}\label{eq:quadratic_functional_def}
		\FeaturizerMat_{\bftheta}^{(m)}(\cdot) \bfw^{(m)} \in \argmin_{\caA \bfw\in \Fcal_m}  \caQ[\caA \bfw] + \frac12 \lambda \|\caA \bfw\|^2
	\end{align}
with regularization parameter $\lambda > 0$. 
An advantage of the separable model formulation is that the quadratic functional is equivalent to a quadratic \emph{function} in $\bfw$; that is, $\caQ[\caA \bfw] \equiv \caQmat(\bfw, \bftheta)$ where $\caQmat: \Rbb^{\NW}  \times \Rbb^{\NTheta} \to \Rbb$ is given by
	\begin{align}\label{eq:quadratic_function}
		\caQmat(\bfw, \bftheta) \equiv \caL + (\caG)^\top \bfw + \tfrac{1}{2} \bfw^\top \caH \bfw.
	\end{align}
Here,  the reduced gradient $\caG \in \Rbb^{\NW}$ and reduced Hessian $\caH \in \Rbb^{\NW\times \NW}$  are, respectively,
	\begin{subequations}\label{eq:grad_hess_shorthand}
	\begin{alignat}{3}
	\caG  &= \langle \FeaturizerMat_{\bftheta}, \nabla \SAA\LossFunctional[f^{(m)}] \rangle 
    &&= \frac{1}{N} \sum_{i=1}^N \FeaturizerMat_{\bftheta}(\bfx_i)^\top  \nabla_{\widehat{\bfy}} \lossfctn(f^{(m)}(\bfx_i), \bfy_i) \qquad \text{and} \label{eq:gbar_formula}\\ 
\caH  &= \langle \FeaturizerMat_{\bftheta}, \nabla^2 \SAA\LossFunctional[f^{(m)}]\FeaturizerMat_{\bftheta}\rangle &&= \frac{1}{N} \sum_{i=1}^N \FeaturizerMat_{\bftheta}(\bfx_i)^\top \nabla_{\widehat{\bfy}}^2 \lossfctn(f^{(m)}(\bfx_i), \bfy_i)\FeaturizerMat_{\bftheta}(\bfx_i). \label{eq:hbar_formula}
	\end{alignat}
	\end{subequations}

Following the VarPro formulation in \eqref{eq:separable_nn_objective}, we translate \eqref{eq:quadratic_functional_def} to a regularized optimization problem of \eqref{eq:quadratic_function} in parameter space 
	\begin{align}\label{eq:quadratic_function_regularized}
		(\bfw^{(m)}, \bftheta^{(m)}) \in \argmin_{\bfw\in \Rbb^{\NW}, \bftheta\in \Rbb^{\NTheta}} \caQmat(\bfw, \bftheta) + \tfrac{1}{2}\lambda_w^{(m)} \|\bfw\|^2 + \lambda_\theta R_{\theta}(\bftheta)
	\end{align}
for $\lambda_w^{(m)}, \lambda_\theta > 0$ and regularizer $R_{\theta}: \Rbb^{\NTheta} \to \Rbb$.

To train the weak learner weights, \VPBoost solves \eqref{eq:quadratic_function_regularized} using a gradient-based approach for $\bftheta$ and variable projection to eliminate $\bfw$,  
The cornerstone of \VPBoost is that \eqref{eq:quadratic_function_regularized} is now quadratic in $\bfw$, and thus admits a \emph{unique, closed-form solution} for the optimal linear weights, 
	\begin{align}\label{eq:w_opt_form}
		\caW = -\left(\caHLam\right)^{-1} \caG.
	\end{align} 
Notably, the closed-form solution holds for any convex loss function, $\ell$ (Assumption~\ref{assump:loss_functional}:~\ref{assump:ell_convexity}); the contribution of the loss is hidden in  the reduced derivatives \eqref{eq:grad_hess_shorthand}. 
Uniqueness is guaranteed because $\caHLam$ is a symmetric positive definite matrix due to the convexity of $\ell$  and the positivity of $\caLambda$. 
The resulting VarPro-reduced, non-quadratic optimization problem in $\bftheta$ is
	\begin{align}\label{eq:quadratic_solve_for_theta}
		\bftheta^{(m)} \in \argmin_{\bftheta\in \Rbb^{\NTheta}} \caQmat(\caW, \bftheta) + \tfrac{1}{2}\lambda_w^{(m)} \|\caW\|^2 + \lambda_\theta R_{\theta}(\bftheta).
	\end{align}
Algorithm~\ref{alg:weak_learner_training} presents a complete weak learner training
pipeline, comparing gradient descent (\texttt{GD}) and variable projection (\texttt{VP}).

\begin{algorithm}[t]
\caption{Weak Learner Training at Stage $m$: Gradient Descent vs. Variable Projection} \label{alg:weak_learner_training}
\begin{algorithmic}[1] \Statex {\bf Goal:} Minimize $\caQmat(\bfw, \bftheta) + \tfrac{1}{2}\lambda_w^{(m)} \|\bfw\|^2 + \lambda_\theta R_{\theta}(\bftheta)$
	\State {\bf Inputs:} 
		\Statex \begin{tabular}{clcl}
		$\bullet$ & $f^{(m)}: \Xcal \to \Rbb^{\NTarget}$ & : & current ensemble \\
		$\bullet$ & $\lambda_w^{(m)}, \lambda_\theta  > 0$ & : &  regularization parameters\\
		$\bullet$ & $\FeaturizerMat_{\bftheta}: \Xcal \times \Rbb^{\NTheta} \to \Rbb^{\NTarget \times \NW}$ & : & featurizer architecture\\
		$\bullet$ & $\bftheta_0\in \Rbb^{\NTheta}$ & : & initial nonlinear weights\\
		$\bullet$ & $\bfw_0\in \Rbb^{\NW}$ & : & initial linear weights (gradient descent only)
		\end{tabular}

	\State {\bf Output:} Trained weights of the separable weak learner $h^{(m)}(\cdot) = \FeaturizerMat_{\bftheta}^{(m)}(\cdot) \bfw^{(m)}$ 
	\Statex
	
\State Compute gradient and Hessian of $\Lcal_N$ at current ensemble, $\caDL$ and $\caDDL$

\For{$k=0,1,2,\dots$}
	\State Compute reduced gradient and Hessian for current featurizer \Comment{Equation~\eqref{eq:grad_hess_shorthand}}
		\begin{align*}
		\bfg_{\bftheta_k}^{(m)} = \langle \FeaturizerMat_{\bftheta_k}, \caDL \rangle \qquad \text{and} \qquad \bfH_{\bftheta_k}^{(m)} =  \langle \FeaturizerMat_{\bftheta_k}, \caDDL \FeaturizerMat_{\bftheta_k} \rangle
		\end{align*}
    \Statex
    \State Select step size/learning rate $\eta_k > 0$
    \State Update linear weights using gradient descent (\texttt{GD}) or variable projection (\texttt{VP}) \label{alg_line:update_linear}
    
     \Statex  \parbox[t]{0.48\linewidth}{
     \begin{tcolorbox}[boxrule=0pt, frame hidden, sharp corners, enhanced, borderline west={2pt}{0pt}{EmoryBlue}, colback=white, parbox=false, notitle,  fonttitle=\bfseries,  
	boxsep=2pt, left=2pt, right=0pt, top=2pt, bottom=2pt, 
	detach title,
overlay={
        \node[rotate=90, minimum width=1cm, anchor=south] at (frame.west) {\color{EmoryBlue}\texttt{GD}};
    }]
    \setcounter{mylineno}{0}
    \myline{7}{$\bfr_k = \bfg_{\bftheta_k}^{(m)} + \bfH_{\bftheta_k}^{(m)}\bfw_k$\vphantom{$\WOptTaylor(\bftheta_k) = -\left( \bfH_{\bftheta_k}^{(m)} + \caLambda \bfI_{\NW}\right)^{-1}  \bfg_{\bftheta_k}^{(m)}$}}\\
\myline{7}{$\bfw_{k+1} = \bfw_k - \eta_k (\bfr_k + \lambda_w^{(m)} \bfw_k)$}
    
        \end{tcolorbox}
    }\hfill\hfill\parbox[t]{0.54\linewidth}{\begin{tcolorbox}[boxrule=0pt, frame hidden, sharp corners, enhanced, borderline west={2pt}{0pt}{EmoryGold}, colback=white, parbox=false, notitle,  fonttitle=\bfseries,  
	boxsep=2pt, left=2pt, right=0pt, top=2pt, bottom=2pt, 
	detach title,
overlay={
        \node[rotate=90, minimum width=1cm, anchor=south] at (frame.west) {\color{EmoryGold}\texttt{VP}};
    }]
       \setcounter{mylineno}{0}
        \myline{7}{$\WOptTaylor(\bftheta_k) = -\left( \bfH_{\bftheta_k}^{(m)} + \caLambda \bfI_{\NW}\right)^{-1}  \bfg_{\bftheta_k}^{(m)}$}\\ \myline{7}{$\bfw_{k} \gets \WOptTaylor(\bftheta_k)$}
     \end{tcolorbox}
    }
    \State Update nonlinear weights using any first-order method, e.g., \Comment{\texttt{VP} benefits from Lemma~\ref{lem:reduced_gradient_equality}} \label{alg_line:update_nonlinear}
	\begin{align*}
	\bftheta_{k+1} = \bftheta_k - \eta_k \left(\nabla_{\bftheta} \caQmat(\bfw_{k}, \bftheta_k)+ \lambda_\theta \nabla R_{\theta}(\bftheta_k) \right) 
	\end{align*}
\EndFor
\end{algorithmic}
\end{algorithm}
 
We share a few comments on nomenclature, subtle notation implications, and natural extensions of Algorithm~\ref{alg:weak_learner_training}. 
Notationally, the residual $\bfr_k$ in Algorithm~\ref{alg:weak_learner_training}, Line~\ref{alg_line:update_linear}(i) for \texttt{GD} corresponds to the gradient of the quadratic function with respect to $\bfw$; i.e., $\bfr_k = \nabla_{\bfw} \caQmat(\bfw_{k}, \bftheta_k)$. 
An important subtlety in Algorithm~\ref{alg:weak_learner_training}, Line~\ref{alg_line:update_linear}(ii) is the index of $\bfw$. 
For \texttt{GD}, the current linear weights, $\bfw_k$, are updated and stored in $\bfw_{k+1}$. 
For \texttt{VP}, the optimal linear weights are tied to the current $\bftheta_k$, and hence stored in $\bfw_k$.  
It is the current linear weights, $\bfw_k$, that are used to update the nonlinear weights, $\bftheta_{k+1}$.  
Because of first-order optimality guarantees on $\caW$ (Lemma~\ref{lem:reduced_gradient_equality}), 
any updating strategy for $\bftheta$ based on $\nabla_{\bftheta} \caQmat(\bfw_{k}, \bftheta_k)$ is permissible without modifying the current algorithm template.   
Extensions to second-order methods for $\bftheta$ are possible, but may require differentiation through the optimal linear weights in the VarPro setting \citep{newman_train_2021}.

\subsection{\VPBoost: A Functional Trust-Region Perspective}
\label{sec:ensemble_training}

At the ensemble level, \VPBoost aggregates weak learners trained by Algorithm~\ref{alg:weak_learner_training} to minimize an objective functional \eqref{eq:saa_objective_functional}. 
Like standard gradient-based optimization, boosting methods must have a mechanism to control the contribution of each 
weak learner in order to guarantee convergence. 
A traditional approach is to rescale weak learners by a boosting rate, analogous to step sizes in finite-dimensional optimization. 
For \VPBoost, the separable structure enables the optimal linear weights to control the directionality and scale of each weak learner.  
This motivates a more natural interpretation of \VPBoost as  a trust-region algorithm in function space.

It is well-known that a Tikhonov-regularized quadratic problem can be reformulated through constrained optimization  \citep{rojas_trust-region_2002}. 
Thus, the closed-form solution for $\caW$ in \eqref{eq:w_opt_form} is equivalent to the solution of the constrained partial minimization problem
	\begin{align}\label{eq:trust_region_for_w}
		\min_{\bfw \in \Rbb^{\NW}} \caQmat(\bfw, \bftheta)
\quad \subjectto \quad \|\bfw\| \le\frac{\Delta(\lambda_w^{(m)})}{\|\caAmat\|}
	\end{align}
for some radius $\Delta(\lambda_w^{(m)}) > 0$ that is related to the regularization parameter. 
The radius prescribes a region over which one trusts the local quadratic function. 
Thus, one can interpret the Tikhonov-regularized quadratic function as a trust-region method~\citep{conn_trust_2000, nocedal_numerical_2006}.  
In the language of constrained optimization, $\caLambda$ is the dual variable. 
Furthermore, because $\caLambda > 0$ by assumption, it is inversely proportional\footnote{For $\lambda > 0$, the derivative $d\Delta(\lambda) / d\lambda$ is strictly negative and monotonically increasing. Thus, the two parameters are inversely proportional: a small trust-region radius implies a large regularization parameter and vice versa.} to $\Delta(\caLambda)$ and admits a one-to-one mapping 
	\begin{align}\label{eq:trust_region_radius_formula}
		\caDeltaLam = \|\caAmat\|\|\caW\| = \|\caAmat\|\left\|\left(\caH + \lambda_w^{(m)} \bfI_{\NW}\right)^{-1} \caG\right\|.
	\end{align} 
Due to submultiplicativity of norms (Lemma~\ref{lem:operator_norm}), a VarPro weak learner is automatically bounded by the trust-region radius, i.e.,  $\|\caAmat(\cdot)\caW\| \le \|\caAmat\|\|\caW\|  = \caDeltaLam$, and thus satisfies the constraint of the functional trust-region subproblem
	\begin{align}\label{eq:trust_region_for_weak_learner}
		\min_{\caAmat(\cdot) \bfw \in \Fcal_m} \caQ[\caAmat(\cdot)\bfw] \qquad \subjectto \quad \|\caAmat(\cdot) \bfw\| \le \caDeltaLam.
	\end{align}
For this reason, \VPBoost elegantly lends itself to a functional trust-region interpretation. 

In the spirit of the bi-level optimization VarPro formulation \eqref{eq:separable_nn_reduced_prob}, one can consider the quadratic problem in a finite-dimensional parameter space \eqref{eq:trust_region_for_w}  to be a reduced version of its infinite-dimensional functional counterpart \eqref{eq:trust_region_for_weak_learner}. 
However, the change of spaces necessitates a modification of the trust-region constraint. 
Specifically, in function space \eqref{eq:trust_region_for_weak_learner}, the entire weak learner is constrained, while in parameter space \eqref{eq:trust_region_for_w}, the constraint is applied directly to the linear weights. 
To connect the two spaces, we modify the constraint in \eqref{eq:trust_region_for_w} by re-scaling the radius based on the norm of the featurizer. 
The consequence of this ``submultiplicativity gap''  is that the VarPro weak learner $\cah$ will likely
satisfy the strict inequality $\|\cah\| < \caDeltaLam$ and be conservative.  
Hence, additional assumptions are needed to ensure that \VPBoost makes sufficient progress at the ensemble level (Section~\ref{sec:vpboost_assumptions}).

We now turn to our core algorithm \VPBoost, described as a functional trust-region method. 
Denote the $m^{\rm th}$ VarPro weak learner as $h_\star^{(m)} = \caAm \caW$, indicating the linear weights were optimized based on data from ensemble $f^{(m)}$ and the nonlinear weights were optimized by solving \eqref{eq:quadratic_solve_for_theta}.
The central actions of a basic trust-region algorithm are (i) accepting or rejecting a trial update, $h_\star^{(m)}$, and (ii) updating the trust-region radius or equivalently the regularization parameter. 
Both actions are determined by the ratio of the actual reduction to the predicted reduction
	\begin{align}\label{eq:reduction_ratio}
		\rho^{(m)} = \frac{\caL - \SAA\LossFunctional[f^{(m)} + h_\star^{(m)}]}{\caQ[0] - \caQ[h_\star^{(m)}]}.
	\end{align} 

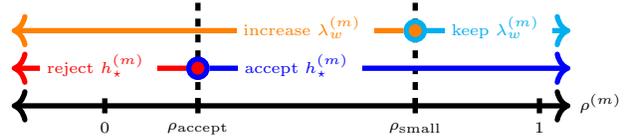
\begin{wrapfigure}{r}{0.5\linewidth}
\vspace{-8pt}
\centering

\begin{tikzpicture}[line width=2pt]
	\tiny
	
	\coordinate (left) at (-0.45\linewidth, 0);
	\coordinate (right) at (0.45\linewidth, 0);
	\coordinate (zero) at (-0.3\linewidth, 0);
	\coordinate (one) at (0.4\linewidth, 0);
	\coordinate (accept) at (-0.15\linewidth, 0);
	\coordinate (small) at (0.2\linewidth, 0);
	
	\draw[<->] (left) -- 
		node[at end, right] {$\rho^{(m)}$  \vphantom{$\rho_{(m)}$}}  
(right);
	
	\def\b{1.4}
	\draw[dashed] (accept) -- ($(accept) + (0, \b)$); 
	\draw[dashed] (small) -- ($(small) + (0, \b)$); 
	
	\def\b{0.1}
	\draw ($(zero) + (0, -\b)$) -- node[below, at start] (zero_node) {$0$} ($(zero) + (0, \b)$); 
	\draw ($(accept) + (0, -\b)$) -- node[below, at start] (accept_node) {$\rho_{\rm accept}$}  ($(accept) + (0, \b)$);  
	\draw ($(small) + (0, -\b)$) -- node[below, at start] (small_node) {$\rho_{\rm small}$}  ($(small) + (0, \b)$); 
	\draw ($(one) + (0, -\b)$) -- node[below, at start] (small_node) {$1$}  ($(one) + (0, \b)$); 
	
	\def\b{0.5}
	\def\r{0.025}
	\node[fill, circle, minimum width=\r cm, red] at ($(accept) + (0, \b)$) {};
	\node[draw, circle, minimum width=\r cm, blue]   (cutoff_accept) at ($(accept) + (0, \b)$) {};
	\draw[->, red]  (cutoff_accept) -- node[midway, fill=white] {reject $\cah$ \vphantom{$h_{(m)}^\star$}} ($(left) + (0, \b)$);  
	\draw[->, blue]  (cutoff_accept) -- node[near start, fill=white] {accept $\cah$  \vphantom{$h_{(m)}^\star$}} ($(right) + (0, \b)$); 
	
	\def\b{1}
	\node[fill, circle, minimum width=\r cm, orange] at ($(small) + (0, \b)$) {};
	\node[draw, circle, minimum width=\r cm, cyan]   (cutoff_small) at ($(small) + (0, \b)$) {};
	\draw[->, orange]  (cutoff_small) -- node[near start, fill=white] (increase) {increase $\caLambda$ \vphantom{$\lambda_{(m)}^w$}} ($(left) + (0, \b)$);  
	\draw[->, cyan]  (cutoff_small) -- node[midway, fill=white] {keep $\caLambda$  \vphantom{$\lambda_{(m)}^w$}} ($(right) + (0, \b)$); 
	
\end{tikzpicture}

\caption{Reduction ratio cutoffs and regularization parameter update in Algorithm~\ref{alg:trust_region_skeleton_vpboost}.}
\label{fig:reduction_ratio_range}

\vspace{-5pt}
\end{wrapfigure} Intuitively, $\rho^{(m)} \approx 1$ indicates the quadratic model represents the true functional well and one can safely accept the trial point. In practice, a trial point is accepted when $\rho^{(m)} > \rho_{\rm accept}$ where $\rho_{\rm accept} \in [0, 1)$ is a user-defined acceptance ratio.

A hallmark of trust-region algorithms is the automatic adaptation of the size of the region, controlled by the regularization parameter. 
Importantly, increasing $\caLambda$ leads to smaller, more conservative steps, which ensures that the algorithm proceeds cautiously if the model is a poor approximation. 
In practice, one increases the regularization parameter by a multiplicative factor $\gamma_{\rm up} > 1$ if the trial point is rejected (i.e., $\rho^{(m)} < \rho_{\rm accept}$) or if the model is overconfident (i.e., $\rho_{\rm accept} < \rho^{(m)} < \rho_{\rm small}$ for some safety cutoff ratio $\rho_{\rm small} < 1$).\footnote{The regularization parameter may also be decreased when the reduction ratio is
sufficiently large, which can improve convergence rates but is unnecessary for the
convergence guarantees; we omit this detail for simplicity.}   
Figure~\ref{fig:reduction_ratio_range} illustrates the algorithm behavior based on reduction ratio cutoffs. 
We formally describe \VPBoost in Algorithm~\ref{alg:trust_region_skeleton_vpboost}.

\begin{algorithm}[t]
\caption{\VPBoost: A Functional Trust-Region Method}\label{alg:trust_region_skeleton_vpboost}
\begin{algorithmic}[1] \Statex {\bf Goal:} Converge to a stationary point of $\SAA\LossFunctional: \Fcal \to \Rbb$
	\State {\bf Inputs:} 
		\Statex \begin{tabular}{clcl}
		$\bullet$ & $f^{(0)}: \Xcal \to \Rbb^{\NTarget}$ & : & initial ensemble \\
		$\bullet$ & $0 < \lambda_{\rm low} \le \lambda_w^{(0)}$ & : &  initial regularization parameter \\
		$\bullet$ & $0 \le \rho_{\rm accept} < \rho_{\rm small}  < 1$ & : & user-defined reduction ratio cutoffs\\
		$\bullet$ & $1 < \gamma_{\rm up}$ & : &  user-defined regularization parameter multiplicative update
		\end{tabular}

	\State {\bf Output:} VarPro ensemble $f^{(m)}: \Xcal \to \Rbb^{\NTarget}$
	\Statex
\For{$m=0, 1, 2, \dots$} 	 
		\State Train \texttt{VP} weak learner $\caAm \caW$ with optimal $\caW$ from \eqref{eq:w_opt_form} \label{alg_line:weak_learner_training} \Comment{Algorithm~\ref{alg:weak_learner_training}}
\State Compute the actual-versus-predicted reduction ratio $\rho^{(m)}$ from \eqref{eq:reduction_ratio} \label{alg_line:rho}
\Statex \dotfill
		\Statex \colorbox{lightgray!50}{\makebox[0.99\linewidth][l]{Accept or reject weak learner}} 
		\vspace{-18pt} 
		\Statex \dotfill
		\If{{\color{blue} $\rho^{(m)} > \rho_{\rm accept}$}} \Comment{\color{blue} ACCEPT}  \label{alg_line:vpboost_accept} 
			\State Accept trial update and define $f^{(m+1)} = f^{(m)} + \caAm \caW$ \label{alg_line:vpboost_update}
		\Else \Comment{\color{red} REJECT}
			\State Reject trial update and define $f^{(m+1)} = f^{(m)}$
		\EndIf
	
		 \Statex\hspace{\algorithmicindent}\dotfill
		\Statex\hspace{\algorithmicindent}\colorbox{lightgray!50}{\makebox[0.95\linewidth][l]{Update regularization parameter}} 
		\vspace{-18pt} 
		\Statex\hspace{\algorithmicindent}\dotfill
		\If{{\color{orange}$\rho^{(m)} < \rho_{\rm small}$}} \Comment{\color{orange}model too confident} \label{alg_line:model_confidence}
			\State Increase regularization parameter  $\lambda_w^{(m+1)}  = \gamma_{\rm up} \lambda_w^{(m)}$ 							\label{alg_line:reg_increase}

		\Else  \Comment{\color{cyan} model reasonable}
			\State Keep regularization parameter $\lambda_w^{(m+1)} = \lambda_w^{(m)}$ 
		\EndIf
	\EndFor
\end{algorithmic}
\end{algorithm}

\begin{remark}
Trust-region algorithms are traditionally formulated in terms of a trust-region radius updating strategy \citep[Algorithm 4.1]{nocedal_numerical_2006} \citep[Algorithm 6.4.1]{conn_trust_2000}. 
However, because the regularization parameter appears directly in the closed-form solution for $\caW$ in \eqref{eq:w_opt_form} and can be uniquely identified with a trust-region radius, it is natural for us to formulate \VPBoost in terms of $\caLambda$. 
\end{remark}

\section{\VPBoost Convergence Analysis}
\label{sec:convergence_analysis}

The ultimate goal of this section is to prove that \VPBoost (Algorithm~\ref{alg:trust_region_skeleton_vpboost}) converges to a stationary point of the loss functional; that is,  
	\begin{align}\label{eq:stationary_point}
		\lim_{m\to \infty} \|\caDL\| = 0.
	\end{align}
We proceed as follows. 
Section~\ref{sec:descent_direction} proves that VarPro weak learners are automatically
descent directions in function space.
Section~\ref{sec:vpboost_assumptions} identifies the \VPBoost assumptions needed overcome the submultiplicativity gap.
Section~\ref{sec:vpboost_lemmas} serves as a keystone for \VPBoost theory and presents three core lemmas to ensure \texttt{VP} weak learners achieve sufficient model reduction.  
Section~\ref{sec:global_convergence} proves global convergence of \VPBoost by combining the new lemmas with established trust-region arguments.

In reality, the presented theoretical results hold for any training paradigm of separable weak learners (e.g., \texttt{GD} or \texttt{VP} in Algorithm~\ref{alg:weak_learner_training}), provided the linear weights are optimized just before the ensemble update (e.g., prior to Algorithm~\ref{alg:trust_region_skeleton_vpboost}, Line~\ref{alg_line:vpboost_update}). 
The specific advantages of \VPBoost are realized at a practical level for weak learner training and at a conceptual level for ensemble construction.  
Practically,  \VPBoost employs VarPro to accelerate weak learner training  (Algorithm~\ref{alg:weak_learner_training}) and improve the conditioning of the optimization problem; see Appendix~\ref{app:vp_vs_gd_proofs}.
Conceptually, \VPBoost acts as a bridge between training weak learners in parameter space and constructing an ensemble in function space; see Figure~\ref{fig:vpboost_cartoon} for an illustration. 
Because VarPro optimizes the linear weights at every step during weak learner training, each update in parameter space simultaneously makes progress toward a stationary point in function space.  
\begin{remark}
    Our theoretical analysis strongly adheres to the proof structure in \cite[Section 4.3]{nocedal_numerical_2006}. 
    Our original contributions primarily stem from how we connect function and parameter space using \VPBoost. 
\end{remark}

\begin{figure}[!htp]
	\centering
	\begin{tikzpicture}
		
		\node (fspace) {\includegraphics[width=0.4\linewidth, trim={1cm 0cm 0cm 0cm}, clip]{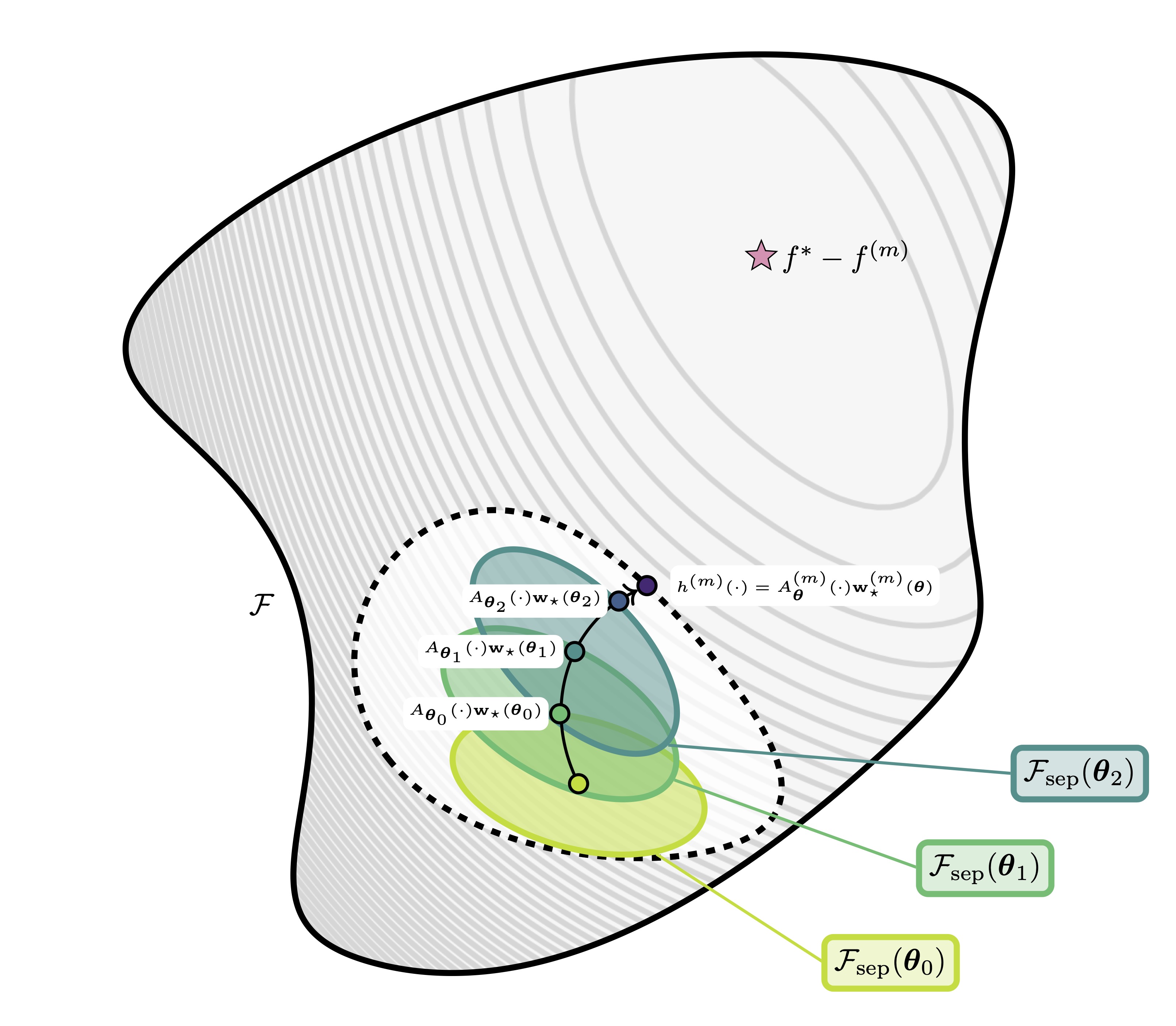}};
		\node[above=0.0cm of fspace.north, anchor=south, fill=lightgray!50, draw, rounded corners=3pt] {\scriptsize \begin{tabular}{c} Function Space $\Fcal$ \\ \emph{Ensemble Level} \end{tabular}};
		 \node[right=0.15\linewidth of fspace.east, anchor=west] (pspace) {\includegraphics[width=0.3\linewidth]{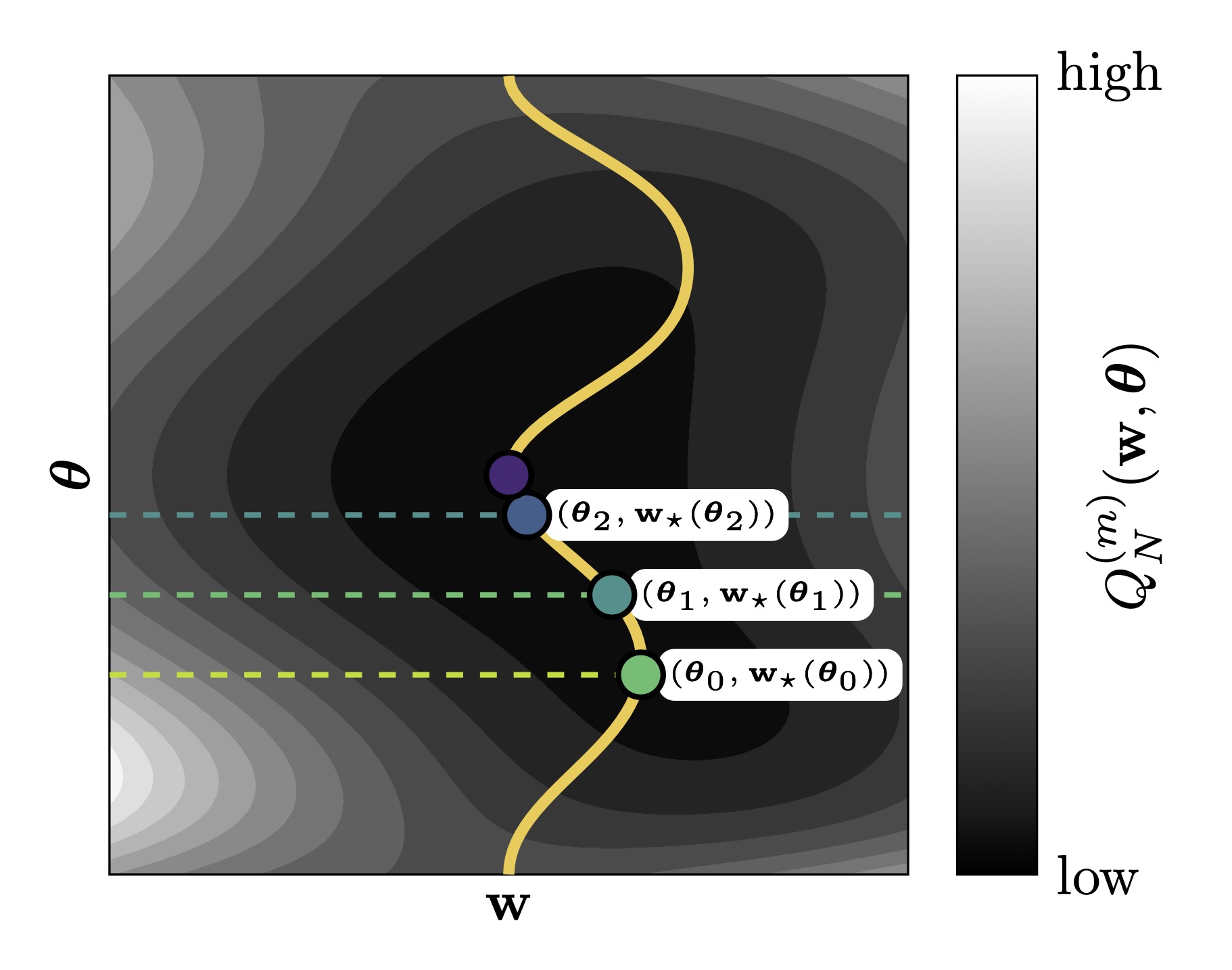}}; 
		 \node[above=0.0cm of pspace.north, anchor=south, fill=lightgray!50, draw, rounded corners=3pt] {\scriptsize \begin{tabular}{c} Parameter Space $\Rbb^{\NW} \times \Rbb^{\NTheta}$ \\ \emph{Weak Learner Level} \end{tabular}};
		 
		 \def\yy{0.5}
		 \def\xx{0.5}
		 \draw[->, line width=3pt] ($(fspace.east)+(-\xx,\yy)$) to[bend left=30] node[above, midway] {\scriptsize progress in $\Rbb^{\NTheta}$ with \texttt{VP}} ($(pspace.west)+(-\xx,\yy)$) ; 
		 \draw[->, line width=3pt] ($(pspace.west)+(-\xx,-\yy)$)  to[bend left=30]  node[below, midway] {\scriptsize simultaneous progress in $\Fcal$}  ($(fspace.east)+(-\xx,-\yy)$) ; 
	\end{tikzpicture}
	
	\caption{Illustration of the \VPBoost bridge between function space $\Fcal$ (left) and parameter space $\Rbb^{\NW} \times \Rbb^{\NTheta}$ (right). 
	The dashed region is the subset of $\Fcal$ containing separable weak learners with specific featurizer architecture, $A_{\bftheta}$.   
	 Each colorful ellipse represents a linear subspace of $\Fcal$ for fixed $\bftheta \in \Rbb^{\NTheta}$ defined as $\Fcal_{\rm sep}(\bftheta) = \{A_{\bftheta}(\cdot) \bfw \in \Fcal \mid \bfw\in \Rbb^{\NW}\}$. 
As VarPro traverses along $\bfw_{\star}(\bftheta)$ in parameter space and progresses toward a minimum (colorful circles along yellow curve), the corresponding linear subspaces $\Fcal_{\rm sep}(\bftheta)$ evolve simultaneously in function space and progress toward the target, $f^* - f^{(m)}$.}
	\label{fig:vpboost_cartoon}
\end{figure}

\subsection{\VPBoost Guarantees Descent in Function Space}
\label{sec:descent_direction}
The first step toward convergence is to show that VarPro weak learners are descent
directions in function space.
We prove that optimal linear weights alone are sufficient to guarantee descent, provided that that the reduced gradient~\eqref{eq:gbar_formula} is nonzero. 
This condition implies that the optimal linear weights, $\caW$, are nonzero following the closed-form solution \eqref{eq:w_opt_form}. 
\begin{lemma}[VarPro Guarantees Descent]\label{lem:vpboost_descent}
A VarPro weak learner, $\cah(\cdot) = \caAm\caW$, is guaranteed to be a descent direction at the current ensemble $f^{(m)}$ provided the reduced gradient is nonzero; that is, $\langle \cah,  \caDL \rangle < 0$ whenever
$\caG = \langle \caAmatm, \caDL \rangle   \not={\bf0}_{\NW \times 1}$. 
\end{lemma}

\begin{proof}
The proof follows directly from the formulas for $\caW$ in~\eqref{eq:w_opt_form} and the reduced derivatives in~\eqref{eq:grad_hess_shorthand}.  
We compute the inner product \eqref{eq:functional_gradient_inner_product} and simplify via 
\begin{align*}
		\langle \caDL, \cah \rangle
&=\frac{1}{N}\sum_{i=1}^N\left(  \nabla_{\widehat{\bfy}} \ell(f^{(m)}(\bfx_i),\bfy_i)^\top \caAmat^{(m)}(\bfx_i)\caW\right)\\
			&={\color{gray}\underbrace{\color{black}\left(\frac{1}{N}\sum_{i=1}^N  \nabla_{\widehat{\bfy}} \lossfctn(f^{(m)}(\bfx_i),\bfy_i)^\top\caAmat^{(m)}(\bfx_i) \right) }_{(\caG)^\top}}\caW\\
			&=-(\caG)^\top (\caHLam)^{-1}\caG.
	\end{align*}
Because $\caHLam$ is symmetric positive definite, $\langle\caDL, \cah \rangle < 0$ whenever $\caG \not= \bf0$. 
By definition, $\cah$ is a descent direction. 
\end{proof}

 A strength of Lemma~\ref{lem:vpboost_descent} is that \texttt{VP} weak learners can yield descent directions in function space even with suboptimal featurizers.  
In practice, better optimization of $\bftheta$ leads to more informative weak learners and better ensemble approximations (Section~\ref{sec:numerical_experiments}).

 	\subsection{Assumptions for \VPBoost Trust-Region Convergence Analysis}
\label{sec:vpboost_assumptions}

The submultiplicativity gap of the \VPBoost trust-region formulation \eqref{eq:trust_region_for_w} necessitates additional assumptions to ensure \VPBoost makes sufficient progress. 
This leads to per-weak-learner conditions that connect derivative information in parameter and function spaces. 
We adopt the term \emph{subspace regularity} to encompass all \VPBoost assumptions. 
For the sake of exposition, we state the assumptions first and subsequently interpret and discuss the reasonableness of each condition. 

\begin{assumption}[\VPBoost Assumptions] \label{assump:vpboost}  
We require four mild assumptions per weak learner stage $m$ with VarPro weak learner  $h_\star^{(m)}(\cdot) = \caAm \caW$ and current ensemble $f^{(m)}$. 

\begin{enumerate}[label={\bf (VP\arabic*)}, leftmargin=*]
	\item  \label{vpboost_assump:gradient_alignment} {\bf Gradient Alignment:} We assume the featurizer is aligned sufficiently with the gradient; i.e., there exists a fixed constant $\kappa_{\rm align}  \in (0,1]$ such that, for all $m$, 
		\begin{align*}
			\kappa_{\rm align} \le \frac{\left\|\langle \caAm, \caDL \rangle \right\|}{\|\caAm\|\|\caDL\|}.
		\end{align*}
	The numerator is equivalent to the norm of the reduced gradient, $\caG $, in \eqref{eq:gbar_formula}. 
	\item \label{vpboost_assump:curvature_capture}  {\bf Curvature Capture:} 
	We assume that $h_{\star}^{(m)}$ captures sufficient curvature information along the gradient direction; i.e., there exists fixed constant $\kappa_{\rm curve} > 0$ such that, for all $m$, 
	\begin{align*}
		\kappa_{\rm curve} \langle \caDL, \caDDL \caDL \rangle 
			&\le  \langle h_\star^{(m)}, \caDDL h_\star^{(m)} \rangle. 
	\end{align*}
The right-hand side of the inequality is equivalent to a curvature condition on the reduced Hessian \eqref{eq:hbar_formula} in the direction of $\caW$; specifically,  
	\begin{align*}
	\langle h_\star^{(m)}, \caDDL h_\star^{(m)} \rangle = \caW^\top \caH \caW. 
	\end{align*}

	\item \label{vpboost_assump:bounded_featurizer} {\bf Bounded Featurizer:} We assume that the featurizers are uniformly bounded; i.e., there exists fixed constants $0 < \alpha_{\rm low} < \alpha_{\rm high}$ such that,  for all $m$,
		\begin{align*}
			\alpha_{\rm low} \le \|\caAmatm\|\le \alpha_{\rm high}.
		\end{align*}
		
		\item \label{vpboost_assump:sufficient_regularization} {\bf Sufficient Regularization:} We assume that the regularization parameter is uniformly bounded away from zero; i.e., there exists a fixed constant $\lambda_{\rm low} > 0$ such that, for all $m$,
		\begin{align*}
			0 < \lambda_{\rm low} \le \caLambda.
		\end{align*}
\end{enumerate}
\end{assumption}

We detail the motivation and restrictiveness of each \VPBoost assumption and contrast the conditions to those in recent gradient boosting literature. 

\paragraph{Gradient Alignment Condition~\ref{vpboost_assump:gradient_alignment}} 
One can interpret the gradient alignment ratio as the cosine similarity between the featurizer and $\caDL$.  
As such, the similarity metric is guaranteed to be a fraction. 
The additional requirement is that the similarity metric must be uniformly bounded away from zero, independent of the weak learner stage $m$. This non-degeneracy condition ensures the growing ensemble decreases $\caL$ by a quantity that does not vanish as $m\to \infty$, a stronger assumption than requiring a descent direction in Lemma~\ref{lem:vpboost_descent}.

We note that gradient alignment differs from the Minimal Cosine Angle (MCA) condition introduced in \citep[Definition 4.2]{lu_randomized_2020}.  
The MCA condition is an ensemble-level assumption regarding the density of weak learners in function space, which leads to a sufficiently expressive ensemble. 
In contrast, gradient alignment is a weak-learner-level assumption that quantifies how well $\caA$ aligns with a single direction, $\caDL$, at weak learner stage $m$.  

The ensemble-level MCA condition and the weak-learner-level gradient alignment condition each come with advantages and limitations. 
On the one hand, the MCA condition is less restrictive; it does not impose structure on weak learners and thereby encompasses a wider range of potential ensembles than \VPBoost. 
In comparison, gradient alignment requires weak learners to be separable with a particular, albeit mild, directionality criterion. 
On the other hand, gradient alignment is prescriptive and guides weak learner training on the fly to ensure sufficient algorithmic progress. 
In contrast, the MCA condition applies to a pre-determined dictionary of candidate learners and thus needs to embed some a priori knowledge of the loss functional landscape.

\paragraph{Curvature Capture Condition~\ref{vpboost_assump:curvature_capture}}  Because trust-region methods incorporate second-order information, we require that appropriate curvature information is captured by the reduced Hessian, $\caH$, from~\ref{eq:hbar_formula}.  
At each stage $m$, the VarPro weak learner must capture some positive proportion of the curvature information only in the  direction of $\caDL$. 
This assumption precludes weak learners from becoming arbitrarily close to lying in the kernel of $\caDDL$ and ensures that sufficient curvature information is retained in $\caH$. 

Because $\caA$ is already assumed to exhibit some non-degenerate alignment with $\caDL$ from \ref{vpboost_assump:gradient_alignment}, the \texttt{VP} weak learner, $\cah(\cdot) = \caA \caW$, will naturally retain some gradient information. 
It is thus reasonable to expect that curvature information along the gradient direction will be retained. 
The curvature capture condition also requires that the amount of retained Hessian information does not vanish as $m\to \infty$.

\paragraph{Bounded Featurizer Condition~\ref{vpboost_assump:bounded_featurizer}}
These bounds on $\|\caA\|$ serve two purposes. 
First, combined with the uniform-boundedness of $\caDDL$ (Assumption~\ref{assump:loss_functional}:~\ref{assump:ell_bounded_hessian}), the upper bound ensures that the reduced Hessian is also uniformly bounded by
	\begin{align}\label{eq:reduced_hessian_bound}
		\|\caH\| \le \|\caAmatm\|^2 \|\caDDL\| =\alpha_{\rm high}^2 \beta.
	\end{align}
    Hence, this condition can be viewed as an extension of the standard uniform Hessian bound trust region theory typically assumes.
The second purpose addresses the submultiplicativity gap in the trust-region subproblem for $\bfw$ \eqref{eq:trust_region_for_w}. 
The boundedness of $\|\caA\|$ ensures the size of the radius $\caDeltaLam / \|\caAmatm\|$ is ultimately controlled by $\caLambda$ and not by the scale of the featurizer.

The bounds  $\alpha_{\rm low}$ and $\alpha_{\rm high}$ are not restrictive. 
In fact, for ensembles of finitely many weak learners, the bounds can be derived a posteriori. 
Often in machine learning, the bounds are also satisfied automatically by the class of featurizers (e.g., bounded activation functions in NNs). 
Effectively, the assumption only becomes necessary for potentially unbounded featurizers in the asymptotic case of infinitely many weak learners.

\paragraph{Sufficient Regularization Condition~\ref{vpboost_assump:sufficient_regularization}} 
Bounding the regularization parameter away from zero uniformly guarantees uniqueness of $\caW$ in \eqref{eq:w_opt_form}.  
Because \VPBoost only permits increases of $\caLambda$ (Algorithm~\ref{alg:trust_region_skeleton_vpboost}, Line~\ref{alg_line:reg_increase}), this assumption is automatically satisfied algorithmically.

\begin{remark}
The optimal linear weights reveal an intrinsic relationship between regularization parameter and featurizer scale.
Omitting the dependence on $m$, let $\bfw_\star(\bftheta; \lambda_w)$ solve the trust-region subproblem \eqref{eq:trust_region_for_w} with featurizer $\caAmat$ and $\lambda_w > 0$. 
Suppose we rescale $\caAmat$ by a nonzero scalar, i.e., $\caAmat^\alpha = \alpha \caAmat$ for $\alpha \not=0$. 
Then, the solution to the corresponding trust-region subproblem is
	\begin{align*}
        \begin{split}
		\bfw_\star^{\alpha}(\bftheta; \lambda_w)
			= -\left(\alpha^2 \caH + \lambda_w  \bfI_{\NW}\right)^{-1}(\alpha \caG)
			=-\tfrac{1}{\alpha}\left(\caH + \tfrac{\lambda_w}{\alpha^2} \bfI_{\NW} \right)^{-1}\caG
			=\tfrac{1}{\alpha}\bfw_\star(\bftheta; \tfrac{\lambda_w}{\alpha^2}). 
        \end{split}
	\end{align*} 
Effectively, $\bfw_\star^{\alpha}(\bftheta; \lambda_w)$ is equal to a rescaled solution of the original problem with a modified regularization parameter. 
The bounds on $\|\caAmat\|$ in \ref{vpboost_assump:bounded_featurizer} prevent the rescaling from becoming arbitrarily small or large and enable $\lambda_w$ to control the magnitude of the optimal linear weights. 
\end{remark}

 	\subsection{\VPBoost Model Reduction and Algorithmic Bounds}
\label{sec:vpboost_lemmas}

Each \VPBoost iteration trains a weak learner (Algorithm~\ref{alg:trust_region_skeleton_vpboost}, Line~\ref{alg_line:weak_learner_training})  to decrease the local quadratic functional in \eqref{eq:trust_region_for_weak_learner}. 
We thereby need to relate the model reduction obtained from a VarPro weak learner, $\caQ[0] - \caQ[\cah]$, to the norm of the gradient of the loss functional, $\|\caDL\|$, in order to show \VPBoost converges to a stationary point. 
Lemma~\ref{lem:vp_sufficient_model_reduction} and Lemma~\ref{lem:varpro_reduction_lower_bound} formulate this relationship by connecting VarPro weak learners to a specific function, the Cauchy point $h_C^{(m)}$, that has a known relationship with $\|\caDL\|$. 
We additionally prove an algorithmic-based result in Lemma~\ref{lem:varpro_regularization_parameter_upper_bound} to ensure Algorithm~\ref{alg:trust_region_skeleton_vpboost} converges to a stationary point in the limit when $m \to \infty$.

We begin by defining the Cauchy point as a nonnegative multiple of the negative gradient that satisfies the trust-region constraint  $\|h_C^{(m)}\|\le \caDeltaLam$ in \eqref{eq:trust_region_for_weak_learner}. 
\begin{definition}[Cauchy Point]\label{def:cauchy_point}

The Cauchy point $h_C^{(m)}$ is the function parallel to the negative gradient that minimizes the quadratic functional \eqref{eq:quadratic_functional}; that is, 
	\begin{align*}
		h_C^{(m)} = -\gamma_C^{(m)} \caDL
	\end{align*}
where $\gamma_C^{(m)} > 0$ is defined as
	\begin{align*}
		\gamma_C^{(m)} =\min\left\{
						\dfrac{\caDeltaLam}{\|\caDL\|} \;, \; 
						\dfrac{\|\caDL\|^2}{\langle \caDL, \caDDL\caDL\rangle}
					\right\}
	\end{align*}
and defaults to the first scalar if $\langle \caDL, \caDDL\caDL\rangle \le 0$. 

\end{definition}

Armed with the Cauchy point and \VPBoost assumptions (Assumption~\ref{assump:vpboost}), we present the central lemma of \VPBoost convergence analysis.

\begin{lemma}[VarPro vs. Cauchy Model Reduction]\label{lem:vp_sufficient_model_reduction}

Suppose $\cah(\cdot) = \caAm \caW$ is a VarPro weak learner that satisfies Assumption~\ref{assump:vpboost}. 
Then, $\cah$ achieves sufficient model reduction relative to the Cauchy point, $h_C^{(m)}$; that is, 
	\begin{align*}c_2\left(\caQ[0] - \caQ[h_C^{(m)}]\right) \le \caQ[0] - \caQ[ \caAm \caW].
	\end{align*}
for some fixed fraction $c_2 \in (0, 1]$ that holds for all $m$. 
\end{lemma}

\begin{proof}
We can view $c_2$ as a lower bound of the relative model reduction obtained from VarPro weak learner, $r_\star^{(m)}$, to  the reduction from the Cauchy point, $r_C^{(m)}$, respectively defined  as 
	\begin{subequations}
	\begin{align}
		r_\star^{(m)} &\coloneqq \caQ[0] - \caQ[\caAm \caW] \label{eq:vapro_model_reduction}\quad \text{and}\\
		r_C^{(m)} &\coloneqq \caQ[0] - \caQ[h_C^{(m)}]. \label{eq:cauchy_model_reduction_main}
	\end{align} 
	\end{subequations}
The goal of our proof is to construct a lower bound of the ratio $c_2 \le r_\star^{(m)} / r_C^{(m)}$ such that $c_2$ is independent of $m$ and $c_2\in (0,1]$. 

We first derive exact formulas for each model reduction. 
For Cauchy model reduction \eqref{eq:cauchy_model_reduction_main}, we evaluate the quadratic functional at the Cauchy point $h_C^{(m)} = -\gamma_C^{(m)} \caDL$ and obtain
	\begin{align}\label{eq:reduction_cauchy}
r_C^{(m)}= \gamma_C^{(m)} \langle \caDL, \caDL \rangle  - \tfrac{1}{2}(\gamma_C^{(m)})^2  \langle \caDL, \caDDL(\caDL) \rangle.
\end{align}
We compute the model reduction for the separable weak learner using the trust region derivation and
recalling the formulas from \eqref{eq:grad_hess_shorthand}; that is,  
	\begin{align*}
			 r_\star^{(m)} &=-(\caG)^\top \caW - \tfrac{1}{2}\caW^\top \caH \caW\\
		 	&= {\color{gray} \underbrace{\color{black}\caW^\top \left(\caHLam\right)}_{-(\caG)^\top}}\caW - \tfrac{1}{2}\caW^\top \caH \caW
	\end{align*}
where the final equality substituted  $\caG = -\left(\caHLam\right)\caW$ using the closed-form solution in \eqref{eq:w_opt_form}. 
Combining like terms, the separable model reduction formula \eqref{eq:vapro_model_reduction} becomes
	\begin{align}\label{eq:reduction_separable}
r_\star^{(m)} = \caW^\top\left(\tfrac{1}{2} \caH  + \lambda_w^{(m)} \bfI_{\NW}\right)\caW.
\end{align}
We bound the ratio of the reductions by considering two cases based on the Cauchy point step size.

\bigskip

\noindent  \dotfill {\bfseries Case 1: Full Negative Gradient Step} \dotfill

\medskip

\noindent Following Definition~\ref{def:cauchy_point}, we first consider the case when $\gamma_C^{(m)} = \caDeltaLam / \|\caDL\|$, which corresponds to following the negative gradient to the trust region boundary.   
	Effectively, the Cauchy point disregards curvature information. 
	Because the Cauchy point considers only first-order information, we will determine model reduction bounds by ignoring second-order contributions.  
	As a result, we will derive a lower bound that relies on our first-order assumption, the gradient alignment condition (Assumption~\ref{vpboost_assump:gradient_alignment}). 
	
	For the Cauchy reduction \eqref{eq:reduction_cauchy}, we construct an upper
bound by omitting the $(\gamma_C^{(m)})^2$ term: \begin{align}\label{eq:cauchy_model_reduction_upper_bound_case1}
			r_C^{(m)} \le  \gamma_C^{(m)} \|\caDL\|^2 = \caDeltaLam  \|\caDL\|.
		\end{align}
	For the VarPro reduction \eqref{eq:reduction_separable}, we lower bound by
ignoring the curvature contribution of $\caH$:
		\begin{align}\label{eq:varpro_model_reduction_lower_bound_case1}
			r_\star^{(m)} \ge \lambda_w^{(m)} \|\caW\|^2. 
		\end{align} 
	Both bounds utilized the positive semidefiniteness of $\caDDL$ (Assumption~\ref{assump:ell_convexity}).  
	
	We now show that the ratio of these model reduction bounds \eqref{eq:cauchy_model_reduction_upper_bound_case1} and \eqref{eq:varpro_model_reduction_lower_bound_case1} is itself bounded below by a fixed positive fraction. 
	To connect the two bounds explicitly, recall from the reduced trust region subproblem \eqref{eq:trust_region_for_w} that $\caDeltaLam = \|\caAm\| \|\caWLam\|$. 
	Then, the ratio becomes
		\begin{align}\label{eq:reduction_ratio_case1}
			\frac{r_\star^{(m)}}{r_C^{(m)}} 
				&\ge \frac{ \lambda_w^{(m)} \|\caW\|^2 }{\caDeltaLam  \|\caDL\|}
				= \frac{ \lambda_w^{(m)} \|\caW\|}{\|\caAm\|\|\caDL\|}.
		\end{align}
	Using the analytic formula for $\caW$ in \eqref{eq:w_opt_form}, the inner product definition of the reduced gradient, $\caG$, in \eqref{eq:gbar_formula}, and the upper bound of the reduced Hessian in \eqref{eq:reduced_hessian_bound}, a lower bound of the norm of the optimal linear weights is 
		\begin{align}\label{eq:w_opt_lower_bound_case1}
		 \|\caW\| 
		 	= \left\|\left(\caHLam\right)^{-1} \caG\right\| 
		 	\ge \frac{\left\|\langle \caAm, \caDL \rangle \right\|}{\alpha_{\rm high}^2\beta +  \lambda_w^{(m)}}. 
		\end{align}
	Substituting \eqref{eq:w_opt_lower_bound_case1} into \eqref{eq:reduction_ratio_case1}, the lower bound of the reduction ratio becomes
		\begin{align}\label{eq:ratio_lower_bound_part1_case1}
			\frac{r_\star^{(m)}}{r_C^{(m)}}  \ge \frac{ \lambda_w^{(m)}}{\alpha_{\rm high}^2\beta  +  \lambda_w^{(m)}} \cdot  \frac{\left\|\langle \caAm, \caDL \rangle \right\|}{\|\caAm\|\|\caDL\|}. 
		\end{align}
By the gradient alignment condition (Assumption~\ref{vpboost_assump:gradient_alignment}), the second fraction in \eqref{eq:ratio_lower_bound_part1_case1} is uniformly bounded below by $\kappa_{\rm align} \in (0, 1]$. 
	Thus, the lower bound becomes
		\begin{align}\label{eq:c2_tilde_part1}
			\frac{r_\star^{(m)}}{r_C^{(m)}}  \ge \frac{\lambda_w^{(m)}}{\alpha_{\rm high}^2 \beta + \lambda_w^{(m)}}  \cdot \kappa_{\rm align}. 
		\end{align}
	The first fraction on the right-hand side of \eqref{eq:c2_tilde_part1} is a positive, strictly increasing fraction as $\lambda^{(m)} \to \infty$ for $\lambda_w^{(m)} > 0$. 
	Thus, the smallest feasible regularization parameter will serve as a lower bound.
	By  Assumption~\ref{vpboost_assump:sufficient_regularization}, the regularization parameter satisfies $0 < \lambda_{\rm low} \le \lambda_w^{(m)}$. 
	Thus, the reduction ratio is bounded below by 
	\begin{align}\label{eq:c2_tilde}
		\frac{r_\star^{(m)}}{r_C^{(m)}}  \ge\frac{\lambda_{\rm low}}{\alpha_{\rm high}^2 \beta + \lambda_{\rm low}} \cdot \kappa_{\rm align}.
	\end{align}
	Consequently, we have found a positive constant, independent of $m$, that
lower bounds the relative model reduction, and since
$\alpha_{\rm high}^2\beta > 0$ and $\kappa_{\rm align} \in (0,1]$,  the lower bound is also a fraction.

\bigskip

\noindent \dotfill {\bfseries Case 2: Partial Negative Gradient Step} \dotfill

\medskip

\noindent In the presence of significant curvature, the full negative gradient step to the trust region boundary may be suboptimal.   
	Because the Cauchy step is informed by curvature information in this case, we will bound the model reductions while retaining second-order information.    
	As a result, we will derive a lower bound that relies on our second-order assumption, the curvature capture condition (Assumption~\ref{vpboost_assump:curvature_capture}). 
	
	Following Definition~\ref{def:cauchy_point}, the Cauchy step size is
		\begin{align}\label{eq:cauchy_step_size_small}
		\gamma_C^{(m)} = \frac{\|\caDL\|^2}{\langle \caDL, \caDDL(\caDL) \rangle}.
		\end{align}
	For this choice of $\gamma_C^{(m)}$, the Cauchy model reduction in \eqref{eq:reduction_cauchy} simplifies to
		\begin{align}\label{eq:cauchy_model_reduction_upper_bound_case2}
			r_C^{(m)} = \tfrac{1}{2}\gamma_C^{(m)} \|\caDL\|^2.
		\end{align}
	Note that this is the exact reduction and we will not construct an upper bound in this case. 
	
	We now choose the lower bound of the VarPro model reduction in \eqref{eq:reduction_separable} that incorporates curvature information from $\caH$, namely
		\begin{align}\label{eq:varpro_model_reduction_upper_bound_case2}
			r_\star^{(m)} \ge \tfrac{1}{2}\caW^\top \caH \caW.
		\end{align}
	As before, we seek a fixed, positive lower bound of the ratio of \eqref{eq:cauchy_model_reduction_upper_bound_case2} and \eqref{eq:varpro_model_reduction_upper_bound_case2}, derived as 
		\begin{align}\label{eq:ratio_model_reduction_case2}
            \begin{split}
			\frac{r_\star^{(m)}}{r_C^{(m)}} 
				&\ge \frac{\tfrac{1}{2}\caW^\top \caH \caW}{\tfrac{1}{2}\gamma_C^{(m)} \|\caDL\|^2}\\
				&\ge  \frac{\kappa_{\rm curve}\langle \caDL, \caDDL(\caDL) \rangle}{\gamma_C^{(m)} \|\caDL\|^2}\\
				&= \frac{\kappa_{\rm curve}}{(\gamma_C^{(m)})^2}
            \end{split}
		\end{align}
	where the second inequality follows from the curvature capture requirement (Assumption~\ref{vpboost_assump:curvature_capture}) and the last equality is a simplification using the Cauchy step size in \eqref{eq:cauchy_step_size_small}. 
		
	The final step is to obtain an $m$-independent lower bound of the above inequality.  
	We approach this by finding a uniform upper bound of the denominator, $(\gamma_C^{(m)})^2$. 
	Following Definition~\ref{def:cauchy_point}, the step size $\gamma_C^{(m)}$ in \eqref{eq:cauchy_step_size_small} must be less than or equal the largest allowed step size,
	\begin{align}\label{eq:gamma_small_upper_bound}
		\gamma_C^{(m)} \le \frac{\caDeltaLam}{\|\caDL\|}.
	\end{align}
	Using \eqref{eq:trust_region_radius_formula}, we can express trust-region radius in terms of the separable learner components and obtain an upper bound using the closed-form solution for $\caW$ in \eqref{eq:w_opt_form} via
		\begin{align}\label{eq:delta_upper_bound}
            \begin{split}
			\caDeltaLam 
&= \|\caAm\| \left\|\left(\caHLam\right)^{-1} \caG\right\|\\
				&\le \frac{\|\caAm\|  \|\langle \caAm, \caDL\rangle \|}{\lambda_w^{(m)}}\\
				&\le \frac{\alpha_{\rm high}^2 \|\caDL\|}{\lambda_{\rm low}}.
            \end{split}
		\end{align}
	where the last step follows from the  sufficient regularization and bounded featurizer conditions (Assumption~\ref{vpboost_assump:sufficient_regularization}~\ref{vpboost_assump:bounded_featurizer}).
	Combining \eqref{eq:delta_upper_bound} with \eqref{eq:gamma_small_upper_bound}, we obtain the upper bound for the Cauchy step size
		\begin{align}\label{eq:cauchy_step_size_upper_bound}
			\gamma_C^{(m)} \le \frac{\alpha_{\rm high}^2}{\lambda_{\rm low}} 
		\end{align}
	Inserting \eqref{eq:cauchy_step_size_upper_bound} into the denominator of the model reduction ratio in \eqref{eq:ratio_model_reduction_case2} yields a uniform, positive lower bound 
		\begin{align}\label{eq:lower_bound_case2}
			\frac{r_\star^{(m)}}{r_C^{(m)}}  &\ge \frac{\lambda_{\rm low}^2}{\alpha_{\rm high}^4} \cdot \kappa_{\rm curve}.
		\end{align}

\bigskip

\noindent\dotfill {\bfseries Final Fixed Fraction} \dotfill

\medskip

In our final step, we determine a global lower bound on the reduction ratio by selecting the minimum of Case 1 in \eqref{eq:c2_tilde} and Case 2 in \eqref{eq:lower_bound_case2}; that is,
	\begin{align*}
		c_2 = \min\left\{\frac{\lambda_{\rm low}}{\alpha_{\rm high}^2 \beta + \lambda_{\rm low}} \cdot \kappa_{\rm align}, \frac{\lambda_{\rm low}^2}{\alpha_{\rm high}^4} \cdot \kappa_{\rm curve}\right\}.
	\end{align*}
	Because the first constant is a fraction, we guarantee $c_2 \in (0, 1]$, as desired. 
    \end{proof}
 
The key implication of  Lemma~\ref{lem:vp_sufficient_model_reduction} is a relationship between the VarPro model reduction and the norm of the gradient (see Lemma~\ref{lem:cauchy_point_reduction} and Theorem~\ref{thm:sufficient_model_reduction} in Appendix~\ref{app:trust_region_proofs}) given by 
	\begin{align}\label{eq:vp_reduction_vs_gradient_norm}
		\begin{split}
		\caQ[0] - \caQ[h_\star^{(m)}]  &\ge c_2 \left(\caQ[0] - \caQ[h_C^{(m)}]\right) \\
			&\ge c_1 \|\caDL\| \min\left\{\caDeltaLam, \frac{\|\caDL\|}{\|\caDDL\|}\right\}
		\end{split}
	\end{align}
where $c_1 = c_2 / 2 \in (0,1)$.  
The \VPBoost conditions (Assumption~\ref{assump:vpboost}) connect the trust-region radius to yield a simplified, more conservative lower bound than \eqref{eq:vp_reduction_vs_gradient_norm}.

\begin{lemma}[VarPro Model Reduction Lower Bound]\label{lem:varpro_reduction_lower_bound}
Under the \VPBoost conditions (Assumption~\ref{assump:vpboost}) and the uniform boundedness of the Hessian (Assumption~\ref{assump:ell_bounded_hessian}), the VarPro model reduction is bounded below by 
	\begin{align*}\caQ[0] - \caQ[h_\star^{(m)}]  &\ge c_1 \|\caDL\|^2 \cdot \frac{\kappa_{\rm align}}{\beta + \caLambda / \alpha_{\rm low}^2}.
	\end{align*}
\end{lemma} 

\begin{proof}
First, we leverage the relationship between the trust-region radius and the norm of the gradient  in \eqref{eq:trust_region_radius_formula}, $\caDeltaLam = \|\caAmatm\|\|\caWLam\|$, to construct a lower bound   \begin{align}\label{eq:delta_lower_bound_vp_model_reduction}
\caDeltaLam
= \|\caAmat^{(m)}\| \left\|\left(\caHLam\right)^{-1} \caG\right\| 
			\ge \frac{ \|\caAmat^{(m)}\|\|\langle \caAm, \caDL\rangle \|}{\|\caAmat^{(m)}\|^2\beta + \lambda_w^{(m)}}
\end{align}
where the inequality follows from the reduced Hessian bound in \eqref{eq:reduced_hessian_bound}.
From the gradient alignment condition (Assumption~\ref{vpboost_assump:gradient_alignment}), we have $\kappa_{\rm align} \|\caAmat^{(m)}\|\|\caDL\| \le \|\langle \caAm, \caDL\rangle \|$, which subsequently bounds \eqref{eq:delta_lower_bound_vp_model_reduction} below by
	\begin{align}\label{eq:delta_lower_bound_part2}
        \begin{split}
	\caDeltaLam  
		&\ge \frac{\kappa_{\rm align} \|\caAmat^{(m)}\|^2 \|\caDL\|}{\|\caAmat^{(m)}\|^2\beta + \lambda_w^{(m)}} \\
		&= \frac{\kappa_{\rm align} \|\caDL\|}{\beta +\lambda_w^{(m)} / \|\caAmat^{(m)}\|^2} \\
		&\ge \frac{\kappa_{\rm align} \|\caDL\|}{\beta + \lambda_w^{(m)} / \alpha_{\rm low}^2}
       \end{split}
	\end{align}
where the last inequality in \eqref{eq:delta_lower_bound_part2} comes from the uniform lower bound of the featurizer (Assumption~\ref{vpboost_assump:bounded_featurizer}).
The resulting lower-bound on the VarPro model reduction is 
	\begin{align*}
		\caQ[0] - \caQ[h_\star^{(m)}]  &\ge c_1 \|\caDL\|^2 \min\left\{ \frac{\kappa_{\rm align}}{\beta + \lambda_w^{(m)} / \alpha_{\rm low}^2}, \frac{1}{\beta}\right\}. 
	\end{align*}
Because $\kappa_{\rm align} \in (0,1]$ and $\lambda_w^{(m)} > 0$ due to the sufficient regularization condition (Assumption~\ref{vpboost_assump:sufficient_regularization}), the first fraction in the lower bound will be strictly less than $1 / \beta$ for all $m$. 
Thus, the lower bound of the model reduction simplifies to the desired result. \end{proof}

The lower bound on VarPro model reduction incorporates information about the steepness of loss landscape and the curvature captured by the reduced Hessian, $\caH$, at the current ensemble. 
Together, these inform the magnitude of the reduction one can potentially achieve from the trust-region minimizer $\cah$.
The simplified lower bound is a consequence of the submultiplicativity constraint modification in \eqref{eq:trust_region_for_w}.

We present a final lemma that, following the mechanics of the \VPBoost
trust-region algorithm in Algorithm~\ref{alg:trust_region_skeleton_vpboost}, the regularization
parameter $\lambda_w^{(m)}$ has a uniform upper bound. 

\begin{lemma}[Regularization Parameter Upper Bound]\label{lem:varpro_regularization_parameter_upper_bound}
Assume uniform boundedness of the featurizer  (Assumption~\ref{vpboost_assump:bounded_featurizer}) and the Hessian (Assumption~\ref{assump:ell_bounded_hessian}). 
Then, starting from weak learner stage $M$, \VPBoost following Algorithm~\ref{alg:trust_region_skeleton_vpboost} yields an upper bound on the regularization parameter $\caLambda$ for all $m \ge M$ given by 
	\begin{align*}
		\caLambda \le \max \left\{\lambda_w^{(M)}, \gamma_{\rm up} \cdot \tfrac{1}{2}\beta\alpha_{\rm high}^2 \cdot \frac{1 + \rho_{\rm small}}{1 - \rho_{\rm small}}\right\}.
	\end{align*}
\end{lemma} 

\begin{proof}
Suppose we are currently at weak learner stage $M$ in Algorithm~\ref{alg:trust_region_skeleton_vpboost}. 
We will show that the regularization parameter cannot become arbitrarily large by relating the reduction ratio $\caRho$ to the regularization parameter $\caLambda$ and the model confidence threshold $\rho_{\rm small}$ in Algorithm~\ref{alg:trust_region_skeleton_vpboost}, Line~\ref{alg_line:model_confidence}. 
In particular, we will show that for sufficiently large $\caLambda$, the reduction ratio will exceed $\rho_{\rm small}$, in which case the regularization parameter will no longer be increased (i.e., Algorithm~\ref{alg:trust_region_skeleton_vpboost}, Line~\ref{alg_line:reg_increase} will not be reached, the orange range in Figure~\ref{fig:reduction_ratio_range}). 
To this end, we will find a lower bound for $\rho^{(m)}$ in terms of $\caLambda$, which will lead to an upper bound of $\caLambda$ in terms of $\rho_{\rm small}$. 

We start by constructing a lower bound of $\rho^{(m)}$ in \eqref{eq:reduction_ratio} by decreasing the numerator, the actual reduction, $\caL - \Lcal_N[f^{(m)} + \cah]$. 
To make the numerator smaller, we obtain an upper bound for the new loss functional value using the uniform bound of the Hessian (Assumption~\ref{assump:loss_functional}:~\ref{assump:ell_bounded_hessian})
	\begin{align}\label{eq:next_loss_upper_bound}
		\Lcal_N[f^{(m)} + \cah] \le \caL + \langle \caDL, \cah \rangle + \tfrac{1}{2}\beta \| \cah \|^2.
	\end{align} 
Next, using the separability of the weak learner, $\cah = \caAm \caW$, submultiplicativity of norms, and the reduced derivatives in \eqref{eq:grad_hess_shorthand}, and the featurizer bounds in Assumption~\ref{vpboost_assump:bounded_featurizer}, we build a subsequent upper bound of \eqref{eq:next_loss_upper_bound} in terms of the $\caW$ via
	\begin{align}\label{eq:next_loss_upper_bound_in_w}
        \begin{split}
		\Lcal_N[f^{(m)} + \cah]  
			&\le \caL + (\caG)^\top \caW +  \tfrac{1}{2}\beta \|\caAmatm\|^2 \| \caW \|^2\\
			\begin{split}
			&=\caL -\caW^\top \left(\caH + \caLambda \bfI_{\NW}\right) \caW+  \tfrac{1}{2}\beta \alpha_{\rm high}^2 \| \caW \|^2
			\end{split}\\
			&\le \caL - \left(\caLambda - \tfrac{1}{2}\beta \alpha_{\rm high}^2 \right) \| \caW \|^2.
        \end{split}
	\end{align} 
We truncated the non-positive term $-\caW^\top \caH\caW$ to obtain the final inequality in \eqref{eq:next_loss_upper_bound_in_w}. 
Combining this upper bound and the VarPro model reduction formula in \eqref{eq:reduction_separable}, we obtain a lower bound for the reduction ratio
	\begin{align}\label{eq:rho_upper_bound_by_lambda}
        \begin{split}
		\caRho 
			&=\frac{\caL - \Lcal_N[f^{(m)} + \cah]}{\caQ[0] - \caQ[\cah]} \\
			&\ge \frac{ \left( \caLambda -  \tfrac{1}{2}\beta \alpha_{\rm high}^2\right) \| \caW \|^2}{\caW^\top \left(\tfrac{1}{2}\caHLam\right) \caW}\\
			&\ge  \frac{ \left( \caLambda -  \tfrac{1}{2}\beta\alpha_{\rm high}^2\right) \| \caW \|^2}{\left(\tfrac{1}{2}\beta\alpha_{\rm high}^2 + \caLambda\right)\|\caW\|^2} \\
			&=  \frac{\caLambda -  \tfrac{1}{2}\beta\alpha_{\rm high}^2}{\tfrac{1}{2}\beta\alpha_{\rm high}^2 + \caLambda}.
        \end{split}
	\end{align}
The last inequality in \eqref{eq:rho_upper_bound_by_lambda} follows from the bound on the reduced Hessian \eqref{eq:reduced_hessian_bound}. 
The final equality simplified the expression by canceling $\|\caW\|^2$. 
Algebraic manipulation reveals that if 
	\begin{align*}\caLambda \ge\tfrac{1}{2}\beta\alpha_{\rm high}^2 \cdot \frac{1 + \rho_{\rm small}}{1 - \rho_{\rm small}},
	\end{align*}
then $\rho^{(m)} \ge \rho_{\rm small}$, in which case the weak learner will be accepted and the regularization parameter will not be increased (cyan region in Figure~\ref{fig:reduction_ratio_range}). 
Thus, following the procedure outlined in Algorithm~\ref{alg:trust_region_skeleton_vpboost}, for all $m \ge M$, the regularization parameter is bounded above by
	\begin{align*}\caLambda \le \max \left\{\lambda_w^{(M)}, \gamma_{\rm up} \cdot \tfrac{1}{2}\beta\alpha_{\rm high}^2 \cdot \frac{1 + \rho_{\rm small}}{1 - \rho_{\rm small}}\right\}.
	\end{align*}  
\end{proof}

Recall, by \eqref{eq:trust_region_radius_formula},  $\caLambda > 0$ is inversely proportional to the trust-region radius, $\caDeltaLam$. 
Thus, Lemma~\ref{lem:varpro_regularization_parameter_upper_bound} effectively prohibits $\caDeltaLam$ from becoming arbitrarily small in \eqref{eq:trust_region_for_w}, the quadratic problem in $\bfw$. 
This guarantees the optimal linear weights, and by extension the separable weak learner, will not become too small due to an overly-regularized problem.

  	\subsection{Global Convergence of \VPBoost}
\label{sec:global_convergence}

We conclude our discussion of \VPBoost theory with main convergence results: that \VPBoost converges to a stationary point \eqref{eq:stationary_point} and, under stronger assumptions, achieves a superlinear rate of convergence. 
Theorem~\ref{thm:vpboost_convergence} proves that in the ``accept any reduction'' regime of Algorithm~\ref{alg:trust_region_skeleton_vpboost}, there exists a subsequence of iterates that converges to a stationary point. 
Theorem~\ref{thm:vpboost_convergence_part2} provides a stronger result that  Algorithm~\ref{alg:trust_region_skeleton_vpboost} is globally-convergent provided non-negligible reduction is accepted. 
Theorem~\ref{thm:vpboost_convergence_rate} proves \VPBoost converges superlinearly with a sufficiently expressive featurizer. 

\begin{remark}
We note that the theoretical results can be relaxed to hold only locally based on a level set \citep{nocedal_numerical_2006} and more general trust-region radius constraints. 
However, in the machine learning setting, the loss functionals are often well-behaved globally. 
We therefore, for expositional purposes, present the convergence proof without the weaker locality assumptions. 
The mechanics of the presented proofs hold under the weaker assumptions. 
\end{remark}

\begin{restatable}[\VPBoost Has a Stationary Limit Point]{mythm}{vpboostconvergenceA}\label{thm:vpboost_convergence}
Let $\rho_{\rm accept} = 0$ in Algorithm~\ref{alg:trust_region_skeleton_vpboost} 
and assume standard smoothness and boundedness of the loss functional (Assumption~\ref{assump:loss_functional}). 
Further assume the each VarPro weak learner, $h_\star^{(j)} = \caAmat^{(j)}(\cdot) \bfw_\star^{(j)}(\bftheta)$, satisfies the \VPBoost conditions (Assumption~\ref{assump:vpboost}).    
Then, \VPBoost will converge to a stationary point in the sense that
	 	\begin{align*}
			\liminf_{m\to \infty} \|\caDL\| = 0
		\end{align*}
where  $f^{(m)}(\cdot) = \sum_{j=0}^m h_\star^{(j)}(\cdot)$ is a VarPro-based ensemble. 
\end{restatable}

\begin{proof}
See Appendix~\ref{sec:proof_vpboost_convergence}. 
\end{proof}

Note that convergence of Theorem~\ref{thm:vpboost_convergence} does not hold for any sequence of iterates generated by Algorithm~\ref{alg:trust_region_skeleton_vpboost}; it only guarantees convergence of a subsequence.  
The subtlety arises because $\rho_{\rm accept} = 0$, which accepts any weak learner that produces an arbitrarily small actual reduction.

\begin{restatable}[\VPBoost Converges to a Stationary Point]{mythm}{vpboostconvergenceB}\label{thm:vpboost_convergence_part2}
Let $\rho_{\rm accept} \in (0,1)$ in Algorithm~\ref{alg:trust_region_skeleton_vpboost} and assume standard smoothness and boundedness of the loss functional (Assumption~\ref{assump:loss_functional}). 
Further assume the each VarPro weak learner, $h_\star^{(j)} = \caAmat^{(j)}(\cdot) \bfw_\star^{(j)}(\bftheta)$, satisfies the \VPBoost conditions (Assumption~\ref{assump:vpboost}).    
Then, \VPBoost will converge to a stationary point; that is, 
	 	\begin{align*}
			\lim_{m\to \infty} \|\caDL\| = 0
		\end{align*}
for a greedy, VarPro-based ensemble $f^{(m)}(\cdot) = \sum_{j=0}^m h_\star^{(j)}(\cdot)$. 
\end{restatable}

\begin{proof}
See Appendix~\ref{sec:proof_vpboost_convergence_part2}.
\end{proof}

Up until this point, we have assumed mild, reasonable conditions (Assumption~\ref{assump:vpboost}) to guarantee convergence of \VPBoost. 
To obtain a rate of convergence, we require much stronger, less practical assumptions on the VarPro weak learners.  
In particular, we must assume a trust-region-based condition that the VarPro weak learners eventually are asymptotically similar\footnote{We say $p$ is asymptotically similar to $q$ if $\lim_{m\to \infty} p(m) / q(m) = 0$, meaning $p$ decays to zero faster than $q$. We can express this succinctly in little-o notation by $p = o(q)$.} 
 to Newton weak learners, $\caDDL^{-1} \caDL$. 
The need for more stringent assumptions for convergence rate analysis is well-founded based on existing literature, which often restricts analysis to convergence within a subclass of functions \citep{atsushi_nitanda_functional_2020, bickel_theory_2006} or requires an ensemble-level weak learnability condition \citep{lu_randomized_2020}. 
For completeness, we provide a convergence rate result followed by a discussion on the limitations of the result in the context of boosting with weak learners.

\begin{restatable}[\VPBoost Superlinear Rate of Convergence]{mythm}{vpboostconvergencerate}\label{thm:vpboost_convergence_rate} 
Let $\{f^{(m)}\}_{m=0}^\infty$ be a sequence of ensembles that converges to a stationary point, $f^* \in \Fcal$ with $\nabla \SAA\LossFunctional[f^*] = 0$.  
Assume that the Hessian at the stationary point is positive
definite;\footnote{This implies $f^*$ is a local minimizer of
$\SAA\LossFunctional$ by the second-order sufficiency conditions \citep[Theorem
2.4]{nocedal_numerical_2006}.} that is, $\nabla^2\SAA\LossFunctional[f^*] \succ
0$. 
Further assume, for sufficiently large $m$, that the VarPro weak learner, $\cah$, is asymptotically similar to the Newton weak learner. 
Then, the \VPBoost iterates converge superlinearly to $f^*$.

\end{restatable}

\begin{proof}
See Appendix~\ref{sec:proof_vpboost_convergence_rate}.
\end{proof}

Asymptotic similarity to Newton weak learners enable \VPBoost to capitalize on the characteristic quadratic convergence of Newton iterates. 
The crucial question of Theorem~\ref{thm:vpboost_convergence_rate} is: under what conditions do \VPBoost learners satisfy this asymptotic similarity condition?
This translates to a condition on the featurizer, $\caAmat$, about which, thus far, we have made few assumptions. 
At a high level, one needs to choose a featurizer such that $\caA \caW \approx  \caDDL^{-1} \caDL$ in the long run. 
As a result, the featurizer must become expressive enough to capture substantial curvature information along a particular direction. 
Expressibility counteracts the weakness of the learners, and hence is an unrealistic assumption in the context of boosting. 
Furthermore, the convergence rate is achieved at the infinite limit of Algorithm~\ref{alg:trust_region_skeleton_vpboost}. 
In reality, boosting constructs an ensemble of a small,  finite number of weak learners. 
We therefore would not expect to see a superlinear rate in practice.

\section{Numerical Experiments}
\label{sec:numerical_experiments}

We present numerical experiments demonstrating the strong practical performance of \VPBoost.  
Our experiments encompass a range of machine learning tasks, featurizer architectures, and problem scales. 
We organize the section as follows. 
Section~\ref{subsec:implementation_details} describes the implementation details and 
Section~\ref{subsec:experimental_protocol} describes the protocols used to assess the robustness,
efficiency, and scalability of the proposed techniques. 
Section~\ref{sec:toy_experiments} uses small scale synthetic examples to assess and visualize the performance of \VPBoost for both regression and classification tasks. 
Section~\ref{subsec:real_world_benchmarks} collects three real-world
benchmarks: MNIST image classification, CDR regression, and Higgs binary
classification.

\subsection{Implementation Details and Practical Considerations}
\label{subsec:implementation_details}

\paragraph{Acceptance Heuristics.}
In practice, we work with a prescribed finite budget of weak learners, so the
full trust-region acceptance logic can be simplified without much loss. A
natural choice is to set $\rho_{\rm accept} = 0$ and take $\rho_{\rm small}$
to be very small. Then any trial weak learner producing positive actual
reduction is accepted, and the regularization parameter is only increased when
the quadratic model is exceptionally unreliable. This is analogous to the fixed
shrinkage (boosting-rate) heuristics common in gradient boosting. In the
\VPBoost setting, this practical choice remains consistent with the theoretical
results established earlier: Theorem~\ref{thm:vpboost_convergence} gives the
relevant convergence guarantee for the $\rho_{\rm accept}=0$ regime, while
Lemma~\ref{lem:vpboost_descent} ensures that VarPro weak learners provide a
descent direction under mild conditions.

\paragraph{Kronecker Structure and Reduced Derivatives.}
Beyond these algorithmic simplifications, the implementation also benefits
from the algebraic structure of the weak learner itself. 
A separable neural networks is defined by the mapping $\bfx \mapsto \bfW \Featurizer_{\bftheta}(\bfx)$, where
$\bfW \in \Rbb^{\NTarget\times\NFeat}$ is dense matrix and $\Featurizer_{\bftheta}(\bfx) \in
\Rbb^{\NFeat}$ is a vector-valued feature extractor.
In our notation, this corresponds to the featurizer matrix
$\FeaturizerMat_{\bftheta}(\bfx)=\Featurizer_{\bftheta}(\bfx)^\top\otimes
\bfI_{\NTarget}$, where $\otimes$ is the Kronecker product and
$\bfw=\myvec(\bfW) \in \Rbb^{\NTarget\NFeat}$ is vectorized column-wise~\cite[Section~10.2]{petersen_matrix_2012}.
In the implementation, we
exploit this structure directly rather than constructing the full featurizer $\caAmat$. In particular, the reduced derivatives in
\eqref{eq:gbar_formula}--\eqref{eq:hbar_formula} are assembled in the einsum notation
\begin{align*}
[\caG]_{ab} &=
\frac{1}{\NSamp}\sum_{i=1}^{\NSamp}
[\Featurizer_{\bftheta}(\bfx_i)]_b \cdot [\nabla_{\widehat{\bfy}}\lossfctn(f^{(m)}(\bfx_i),\bfy_i)]_a , \\
[\caH]_{ac,bd} &=
\frac{1}{\NSamp}\sum_{i=1}^{\NSamp}
[\Featurizer_{\bftheta}(\bfx_i)]_a
\cdot
[\Featurizer_{\bftheta}(\bfx_i)]_b
\cdot
[\nabla_{\widehat{\bfy}}^2\lossfctn(f^{(m)}(\bfx_i),\bfy_i)]_{cd}.
\end{align*}

\paragraph{Gradient Handling in the Reduced Model.}
A final implementation detail concerns how gradients are computed for the
reduced model. When optimizing the reduced VarPro objective with respect to the
nonlinear parameters $\bftheta$, the optimal linear weights are first computed
from \eqref{eq:w_opt_form} and then treated as fixed when differentiating the
loss. That is, the gradient step for $\bftheta$ need only account for the
explicit dependence of the reduced model on the featurizer and should not
backpropagate through the linear solve itself. Operationally, this means the
computational graph is truncated after computing the optimal weights. In
practice, this can be implemented using autodiff primitives such as
\texttt{jax.lax.stop\_gradient} in JAX or \texttt{torch.no\_grad()} in PyTorch.

\paragraph{Initial Constant Ensemble.}
We initialize the ensemble with the optimal constant predictor, meaning the
constant vector $\bfc_0\in\Rbb^{\NTarget}$ that minimizes the empirical loss
over the training set,
\[
\bfc_0=\argmin_{\bfc\in\Rbb^{\NTarget}} \frac{1}{\NSamp}\sum_{i=1}^{\NSamp}
\lossfctn(\bfc,\bfy_i).
\]
In other words, the initial ensemble is the best constant approximation before
any weak learner is added. 
For the losses used in this work, this attains familiar closed forms: for
squared loss, $\bfc_0=\frac{1}{\NSamp}\sum_i \bfy_i$; for binary cross-entropy
with logit output, $c_0=\log(\bar y/(1-\bar y))$ where $\bar
y=\frac{1}{\NSamp}\sum_i y_i$; and for multiclass cross-entropy the
loss-minimizing constant satisfies $\sigma(\bfc_0)=\bar{\bfy}$, so one may
take\footnote{In principle, there are many solutions to $\sigma(\bfc_0)=
\bar{\bfy}$ due to softmax translation invariance.} $\bfc_0=\log(\bar{\bfy})$.
Alternatively, one could directly minimize for an optimal constant prior to
training as it is a fairly simple low-dimensional convex optimization
problem.
Optimal constant initialization is discussed further
in~\citet{friedman_greedy_2001}.

\subsection{Experimental Protocol}
\label{subsec:experimental_protocol}

We assess \VPBoost on both synthetic and real-world datasets spanning
regression and classification tasks to evaluate its performance,
efficiency, and scalability.
We begin with toy regression and classification problems to illustrate the
fundamental capabilities of \VPBoost.
Then, we proceed to more challenging real-world datasets, including a
scientific machine learning regression task for a
convection-diffusion-reaction system \citep{newman_train_2021} and large-scale
binary classification of the Higgs boson dataset \citep{whiteson_higgs_2014}.
Aside from the Higgs boson dataset, which reserves the last 500,000 samples for
testing, and MNIST which reserves 10,000 samples for the same purpose,
all datasets are randomly shuffled and split into training, validation, and test
sets with proportions of 70\%, 15\%, and 15\%, respectively.

The validation set is used for hyperparameter tuning and early stopping, while
the testing set is reserved for final performance evaluation.
Hyperparameter selection is performed with Optuna \citep{akiba_optuna_2019}
using 100 trials per experiment, where the learning rate and regularization
strength are jointly tuned. After identifying the best hyperparameter
configuration, we rerun each experiment across 10 distinct random seeds and
report performance as the mean and standard deviation over these runs.

The primary comparison in all experiments is between \VPBoost and standard gradient-descent weak learner training (\GDBoost), isolating the effect of variable projection.
Both methods have the same neural network architectures, loss functions, and optimization algorithms (Adam) for training the weak learners.
Under the same protocol, we will also compare against modifications of the \GDBoost strategy, referred to as {VP@Start}, {VP@End}, and {VP@Start+End}.
\VPBoost performs variable projection throughout training; these algorithms isolate variable projection to a single step within training, at the start, end, or both.
Finally, in all numerical experiments, we compare \VPBoost against the state-of-the-art gradient boosting implementation \XGBoost \citep{chen_xgboost_2016}.
Because \XGBoost is tree-based, a direct like-for-like comparison at the weak
learner level is not possible. 
Nevertheless, in all experiments we compare against \XGBoost ensembles whose total parameter count is matched as closely as possible to the neural network ensemble.
Assuming $M$ learners and a depth per-tree of $d$, the parameter count is conservatively estimated as $\NTarget \cdot M \cdot (2^d - 1)$.
To mitigate the impact of this constraint, we additionally compare against the AUC from the original \XGBoost paper \citep{chen_xgboost_2016} for the Higgs classification dataset.
This serves as an external reference point outside the matched-capacity protocol.
All baseline methods are tuned using the same validation metric and protocol to ensure fair comparison.

All experiments are implemented in Python 3.13 using
JAX~\citep{bradbury_jax_2018} and Equinox~\citep{kidger_equinox_2021} for
neural network components and
\texttt{scikit-learn}~\citep{pedregosa_scikit-learn_2011} for pre-processing and
evaluation utilities.
Experiments were run on hardware provided by the SLURM-based high-performance computing cluster at the Center for Computation and Visualization of Brown University.
The code is available upon request and will be released publicly after review of
the manuscript is complete.
 \subsection{Illustrative 2D Experiments}
\label{sec:toy_experiments}

We present three two-dimensional toy machine learning experiments to assess the
performance of \VPBoost across different loss functions. 
We briefly describe the setup of each experiment and then present the results
showing convergence during training, final ensemble approximations, and final
metrics per task in Figure~\ref{fig:toy_experiment_training},
Figure~\ref{fig:toy_experiment_approx}, and
Table~\ref{tab:toy_experiment_metrics}, respectively.  
The problems were motivated by similar toy examples in
\citet{ruthotto_stable_2017}. 

\subsubsection{2D Experiments Problem Setup}
Each experiment implements a fully-connected neural network, a composite function made of lightweight functions called \emph{layers}. 
Each fully-connected layer consists of an affine transformation followed by a nonlinear activation function; that is, if $u^{\rm FC}: \Rbb^p \to \Rbb^q$ is a fully-connected layer, then $u^{\rm FC}(\bfz) = \sigma(\bfK \bfz + \bfb)$ where $\bfK\in \Rbb^{q\times p}$ is a dense weight matrix and $\bfb\in \Rbb^q$ is an additive bias vector. The nonlinear parameters are $\bftheta = \{(\bfK_i, \bfb_i)\}_{i=1}^d$, collecting the weights and biases across all
$d$ layers of the featurizer.
The capacity of a fully-connected layer is the number of entries in the weight matrix and bias vector, $pq + q$.  
The activation function $\sigma: \Rbb \to \Rbb$ is applied elementwise. 
In the following experiments, we use the smooth hyperbolic tangent activation
function.

\paragraph{Task 1: Regression (MSE)} 
First, we consider approximating a smooth, oscillatory function
	\begin{equation}
		y(\bfx) = x_1(1-x_1)\cos(4\pi x_1)\sin^2(4\pi x_2^2)
	\end{equation}
for $\bfx \equiv (x_1, x_2) \in [0,1]^2$. 
The goal is to learn the smooth function $y: \Rbb^2 \to \Rbb$ from $N$ observations $\{(\bfx_i,  y_i \equiv y(\bfx_i))\}_{i=1}^N$. 
We train our models using mean square error (MSE) loss and $N=1200$ training data points randomly sampled uniformly across the domain $[0,1]^2$.  
Each weak learner featurizer consists of $2$ hidden layers of width $4$ (see Figure~\ref{fig:toy_experiment_weak_learner_regression}). 
The capacity per weak learner is $\NTheta = 52$ and $\NW = 5$ and the ensemble of $M=10$ weak learners has a capacity of $570$. 

	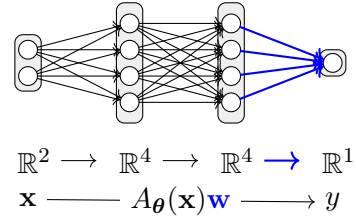
\begin{wrapfigure}{r}{0.3\linewidth}
		\begin{tikzpicture}
		\def\s{0.1}
		\def\r{0.25}
		
		\pgfmathsetmacro{\X}{0}
		\pgfmathsetmacro{\Y}{0}
		
		\draw[rounded corners=3pt, fill=lightgray!25] (\X-0.5*\r-0.5*\s, \Y-1.5*\r-0*\s) rectangle (\X+0.5*\r+0.5*\s, \Y+1.5*\r+0*\s);
		
		\foreach[count=\a] \i in {-0.5, 0.5}{
			\node[circle, minimum width=\r cm, draw, inner sep=0pt, outer sep=0pt, fill=white] (n\a) at (\X,\Y-\i*\r-\i*\s) {};
		}

		\node at (\X, \Y-5*\r-0.5*\s) (rr) {$\phantom{{}^2}\Rbb^2$}; 
		\node[below=0.0cm of rr.south] (xx) {$\bfx$};
		
		\pgfmathsetmacro{\X}{\X+\r+\s+1}
		
		\draw[rounded corners=3pt, fill=lightgray!25] (\X-0.5*\r-0.5*\s, \Y-3*\r-0.5*\s) rectangle (\X+0.5*\r+0.5*\s, \Y+3*\r+0.5*\s);
		
		\foreach[count=\a] \i in {-1.5, -0.5, ..., 1.5}{
			\node[circle, minimum width=\r cm, draw, inner sep=0pt, outer sep=0pt, fill=white] (o\a) at (\X,\Y-\i*\r-\i*\s) {};
		}

		\node at (\X, \Y-5*\r-0.5*\s) (rtmp) {$\phantom{{}^4}\Rbb^4$};
		
		\draw[->] (rr) -- (rtmp);
		
		\pgfmathsetmacro{\X}{\X+\r+\s+1}
		
		\draw[rounded corners=3pt, fill=lightgray!25] (\X-0.5*\r-0.5*\s, \Y-3*\r-0.5*\s) rectangle (\X+0.5*\r+0.5*\s, \Y+3*\r+0.5*\s);
		\foreach[count=\a] \i in {-1.5, -0.5, ..., 1.5}{
			\node[circle, minimum width=\r cm, draw, inner sep=0pt, outer sep=0pt, fill=white] (p\a) at (\X,\Y-\i*\r-\i*\s) {};
		}	
		
		\node at (\X, \Y-5*\r-0.5*\s) (rtmp2) {$\phantom{{}^4}\Rbb^4$};
		\draw[->] (rtmp) -- (rtmp2);
			
		\pgfmathsetmacro{\X}{\X+\r+\s+1}
		
		\draw[rounded corners=3pt, fill=lightgray!25] (\X-0.5*\r-0.5*\s, \Y-0.5*\r-0.5*\s) rectangle (\X+0.5*\r+0.5*\s, \Y+0.5*\r+0.5*\s);
		
		\foreach[count=\a] \i in {0}{
			\node[circle, minimum width=\r cm, draw, inner sep=0pt, outer sep=0pt, fill=white] (q\a) at (\X,\Y-\i*\r-\i*\s) {};
		}

		\node at (\X, \Y-5*\r-0.5*\s) (rr)  {$\phantom{{}^1}\Rbb^1$};
		\draw[->, thick, blue] (rtmp2) -- (rr);
		
		\node[below=0.0cm of rr.south] (yy) {$y$};
		\draw[->] (xx) -- node[midway, fill=white] {$A_{\bftheta}(\bfx) {\color{blue} \bfw}$} (yy);
		
		\foreach \a in {1, ..., 4}{
			\draw[->, black] (n1) -- (o\a);
			\draw[->, black] (n2) -- (o\a);
			\draw[->, thick, blue] (p\a) -- (q1);
			
			\foreach \b in {1, ..., 4}{
				\draw[->, black] (o\a) -- (p\b);
			}
		}
	\end{tikzpicture} 	\caption{Fully-connected NN with two hidden layers. Thicker blue arrows indicate the final linear map. }
	\label{fig:toy_experiment_weak_learner_regression}
\end{wrapfigure}

\paragraph{Task 2: Binary Classification (BCE)} 

The second task we consider is to classify data generated from the Swiss roll function \citep[Section 6.2]{ruthotto_stable_2017}. 
Data points are uniformly sampled from one of two paths, $\zeta_0$ and $\zeta_1$, given by 
	\begin{align}
		\zeta_0(r, \omega) = r \begin{pmatrix} \cos \omega \\ \sin\omega \end{pmatrix}  \quad \text{and} \quad 
		\zeta_1(r, \omega) = (r + 0.2) \begin{pmatrix} \cos \omega \\ \sin\omega \end{pmatrix}
	\end{align}
for $r \in [0,1]$  and $\omega \in [0,4\pi]$. 
The objective is to learn the labeling function $y: \Rbb^2 \to \{0,1\}$ where $y(\bfx) = 0$ if $\bfx$ lies along path $\zeta_0$ and $y(\bfx) = 1$ if $\bfx$ lies along path $\zeta_1$. 
We train our models using the binary cross entropy (BCE) loss with $N=800$ training data points per class. 
Each weak learner featurizer consists of $1$ hidden layer of width $4$ (see a similar architecture in Figure~\ref{fig:toy_experiment_weak_learner_regression}). 
The capacity per weak learner is $\NTheta = 32$ and $\NW = 5$. 
We construct an ensemble of $M=5$ weak learners, resulting in an ensemble capacity of $185$. 

\paragraph{Task 3: Multi-class Classification (MCE)}

Lastly, we consider a multi-class classification problem constructed from level sets of the $\text{MATLAB}^{\textregistered}$ \texttt{peaks} function, formed by transforming and combining Gaussian distributions via
	\begin{align}\label{eq:peaks}
		f(\bfx) = 3(1-x_1)^2 e^{-x_1^2 - (x_2 +1)^2} - 10(\tfrac{1}{5}x_1 - x_1^3 - x_2^5)e^{-x_1^2 - x_2^2} - \tfrac{1}{3}e^{-(x_1+1)^2 - x_2^2}
	\end{align}
for $\bfx \equiv (x_1, x_2)\in [-3, 3]^2$.   
We construct $\NTarget = 5$ level sets of the \texttt{peaks} function by forming a uniform partition of the range of $f$ \citep[Section 6.3]{ruthotto_stable_2017}.  
Data points are generated by over-sampling uniformly over the domain and sub-selecting to form balanced classes.

We train our models using multi-class cross entropy (MCE) loss and $N=800$ training data points per class. 
Each weak learner featurizer consists of $2$ hidden layers of width $4$ (Figure~\ref{fig:toy_experiment_weak_learner_regression}). 
The capacity per weak learner is $\NTheta = 52$ and $\NW = 5$ and the ensemble of $M=10$ weak learners has a capacity of $570$. 

\subsubsection{Performance of \VPBoost on 2D Experiments}
We present training and validation metrics in
Figure~\ref{fig:toy_experiment_training}, final ensemble approximations in
Figure~\ref{fig:toy_experiment_approx}, test performance in
Table~\ref{tab:toy_experiment_metrics},
and a timing benchmark in Figure~\ref{fig:time_vs_loss}.
Across all three tasks and all three loss functions, MSE, BCE, and MCE, \VPBoost attains the lowest final loss and strongest test metrics among all boosting methods, demonstrating robustness to the choice of loss without modification.
On Tasks~1 and~2, all VP variants consistently outperform \GDBoost, confirming that variable projection provides a reliable improvement over standard gradient descent weak learners.
Task~3 is more challenging due to rank-deficiency of the MCE Hessian (discussed below), but \VPBoost still achieves the best performance among the boosting methods.
We also compare against full NN baselines, discussed separately at the end of this section.

\begin{figure}[!htp]
\centering

\begin{tikzpicture}

\scriptsize

	\def\resultsDirA{data/toy_regression/ensemble_metrics}
	\def\ymaxA{7e-3}
	\def\yminA{5e-5}
	\def\xminA{0}
	\def\xmaxA{10}
	\def\ylabelA{MSE {\bfseries (Task 1)}}
	
	\def\resultsDirB{data/swiss_roll/ensemble_metrics_summary}
	\def\yminB{3e-2}
	\def\ymaxB{1}
	\def\xminB{0}
	\def\xmaxB{5}
	\def\ylabelB{BCE  {\bfseries (Task 2)}}
	
	\def\resultsDirC{data/level_sets/ensemble_metrics_summary}
	\def\yminC{1e-2}
	\def\ymaxC{2}
	\def\xminC{0}
	\def\xmaxC{10}
	\def\ylabelC{MCE  {\bfseries (Task 3)}}

\begin{groupplot}[group style={group size=2 by 3, 
		ylabels at=edge left, yticklabels at=edge left, 
		xlabels at=edge bottom, 
horizontal sep=0.025\linewidth, vertical sep=0.05\linewidth}, 
scale only axis,
width=0.4\linewidth, height=0.15\linewidth, 		
xlabel={Boosting Iteration}, 
		ylabel style={anchor=north, yshift=0.6cm}, 
ymode=log, 
ticklabel style={/pgf/number format/fixed, /pgf/number format/precision=0}, 
grid=major,
		grid style={dashed, gray!30},  
legend to name=sharedlegend, 
		legend columns=6,
		execute at begin axis={
    			\addlegendimage{style_gb}  \addlegendentry{\GDBoost}
			\addlegendimage{style_vp_at_start}   \addlegendentry{VP@Start}
			\addlegendimage{style_vp_at_end}   \addlegendentry{VP@End}
			\addlegendimage{style_vp_at_start_and_end}   \addlegendentry{VP@Start+End}
			\addlegendimage{style_vp}   \addlegendentry{\VPBoost}
			\addlegendimage{style_xgboost}   \addlegendentry{\XGBoost}
		}, 
		legend style={/tikz/every even column/.append style={column sep=0.05cm}}
		]

	\nextgroupplot[title=Train, 		
		xmin=\xminA, xmax=\xmaxA,
ylabel={\ylabelA}, 
		ymin=\yminA, ymax=\ymaxA]

	\def\dataType{train}
		\foreach \methodType in {gb, vp_at_start, vp_at_end, vp_at_start_and_end, vp}{
      		  	\edef\temp{
				\noexpand\targetplotB{\resultsDirA/train/optimal_\methodType_em_summary.csv}{style_\methodType}
			}\temp
    		}
		
		\pgfplotstableread[col sep=comma]{\resultsDirA/train/optimal_xgboost_em_summary.csv}\xgboosttable
		\pgfplotstablegetrowsof{\xgboosttable}
		\pgfmathtruncatemacro{\xgtrainlastrow}{\pgfplotsretval-1}
		\pgfplotstablegetelem{\xgtrainlastrow}{target_metric_mean}\of{\xgboosttable}
		\edef\xgtrainfinal{\pgfplotsretval}
		
		\addplot [style_xgboost, domain=-1:100, forget plot] {\xgtrainfinal};

\nextgroupplot[title=Validation,  
		xmin=\xminA, xmax=\xmaxA,
ymin=\yminA, ymax=\ymaxA]

	\def\dataType{val}
		\foreach \methodType in {gb, vp_at_start, vp_at_end, vp_at_start_and_end, vp}{
      		  	\edef\temp{
				\noexpand\targetplotB{\resultsDirA/val/optimal_\methodType_em_summary.csv}{style_\methodType}
				\noexpand
			}\temp
    		}
		
		\pgfplotstableread[col sep=comma]{\resultsDirA/val/optimal_xgboost_em_summary.csv}\xgboosttable
		\pgfplotstablegetrowsof{\xgboosttable}
		\pgfmathtruncatemacro{\xgtrainlastrow}{\pgfplotsretval-1}
		\pgfplotstablegetelem{\xgtrainlastrow}{target_metric_mean}\of{\xgboosttable}
		\edef\xgtrainfinal{\pgfplotsretval}
		
		\addplot [style_xgboost, domain=-1:100, forget plot] {\xgtrainfinal};

	\nextgroupplot[xmin=\xminB, xmax=\xmaxB,
ylabel={\ylabelB}, 
		ymin=\yminB, ymax=\ymaxB]
	
	\def\dataType{train}
		\foreach \methodType in {gb, vp_at_start, vp_at_end, vp_at_start_and_end, vp}{
      		  	\edef\temp{
				\noexpand\targetplotB{\resultsDirB/train/optimal_\methodType_em_summary.csv}{style_\methodType}
			}\temp
    		}
		
		\pgfplotstableread[col sep=comma]{\resultsDirB/train/optimal_xgboost_em_summary.csv}\xgboosttable
		\pgfplotstablegetrowsof{\xgboosttable}
		\pgfmathtruncatemacro{\xgtrainlastrow}{\pgfplotsretval-1}
		\pgfplotstablegetelem{\xgtrainlastrow}{target_metric_mean}\of{\xgboosttable}
		\edef\xgtrainfinal{\pgfplotsretval}
		
		\addplot [style_xgboost, domain=-1:100] {\xgtrainfinal};

\nextgroupplot[xmin=\xminB, xmax=\xmaxB,
ymin=\yminB, ymax=\ymaxB]
	
	\def\dataType{val}
		\foreach \methodType in {gb, vp_at_start, vp_at_end, vp_at_start_and_end, vp}{
      		  	\edef\temp{
				\noexpand\targetplotB{\resultsDirB/val/optimal_\methodType_em_summary.csv}{style_\methodType}
				\noexpand
			}\temp
    		}
		
		\pgfplotstableread[col sep=comma]{\resultsDirB/val/optimal_xgboost_em_summary.csv}\xgboosttable
		\pgfplotstablegetrowsof{\xgboosttable}
		\pgfmathtruncatemacro{\xgtrainlastrow}{\pgfplotsretval-1}
		\pgfplotstablegetelem{\xgtrainlastrow}{target_metric_mean}\of{\xgboosttable}
		\edef\xgtrainfinal{\pgfplotsretval}
		
		\addplot [style_xgboost, domain=-1:100] {\xgtrainfinal};

	\nextgroupplot[xmin=\xminC, xmax=\xmaxC,
ylabel={\ylabelC}, 
		ymin=\yminC, ymax=\ymaxC]
	
	\def\dataType{train}
		\foreach \methodType in {gb, vp_at_start, vp_at_end, vp_at_start_and_end, vp}{
      		  	\edef\temp{
				\noexpand\targetplotB{\resultsDirC/train/optimal_\methodType_em_summary.csv}{style_\methodType}
			}\temp
    		}
		
		\pgfplotstableread[col sep=comma]{\resultsDirC/train/optimal_xgboost_em_summary.csv}\xgboosttable
		\pgfplotstablegetrowsof{\xgboosttable}
		\pgfmathtruncatemacro{\xgtrainlastrow}{\pgfplotsretval-1}
		\pgfplotstablegetelem{\xgtrainlastrow}{target_metric_mean}\of{\xgboosttable}
		\edef\xgtrainfinal{\pgfplotsretval}
		
		\addplot [style_xgboost, domain=-1:100] {\xgtrainfinal};

	\nextgroupplot[xmin=\xminC, xmax=\xmaxC,
ymin=\yminC, ymax=\ymaxC]
	
	\def\dataType{val}
		\foreach \methodType in {gb, vp_at_start, vp_at_end, vp_at_start_and_end, vp}{
      		  	\edef\temp{
				\noexpand\targetplotB{\resultsDirC/val/optimal_\methodType_em_summary.csv}{style_\methodType}
				\noexpand
			}\temp
    		}
		
		\pgfplotstableread[col sep=comma]{\resultsDirC/val/optimal_xgboost_em_summary.csv}\xgboosttableCval
		\pgfplotstablegetrowsof{\xgboosttableCval}
		\pgfmathtruncatemacro{\xgtrainlastrowCval}{\pgfplotsretval-1}
		\pgfplotstablegetelem{\xgtrainlastrowCval}{target_metric_mean}\of{\xgboosttableCval}
		\edef\xgtrainfinalCval{\pgfplotsretval}
		
		\addplot [style_xgboost, domain=-1:100] {\xgtrainfinalCval};

\end{groupplot}

\node at ($(group c1r3.south west)!0.5!(group c2r3.south east)$) [anchor=north, yshift=-1cm, xshift=-0.6cm] {\pgfplotslegendfromname{sharedlegend}
};

\end{tikzpicture}

\caption{Comparison of \VPBoost, \VPBoost variants, \GDBoost, and \XGBoost on the
      synthetic 2D tasks across boosting iterations.
      Each plot shows mean performance (solid lines) with standard deviation
      (shaded band) over 10 independent trials.
      The \XGBoost baseline is shown as a horizontal dashed line at its final
      ensemble target metric.
      Across all tasks, \VPBoost demonstrates faster convergence to a lower loss for both training and validation.}
\label{fig:toy_experiment_training}
\end{figure}
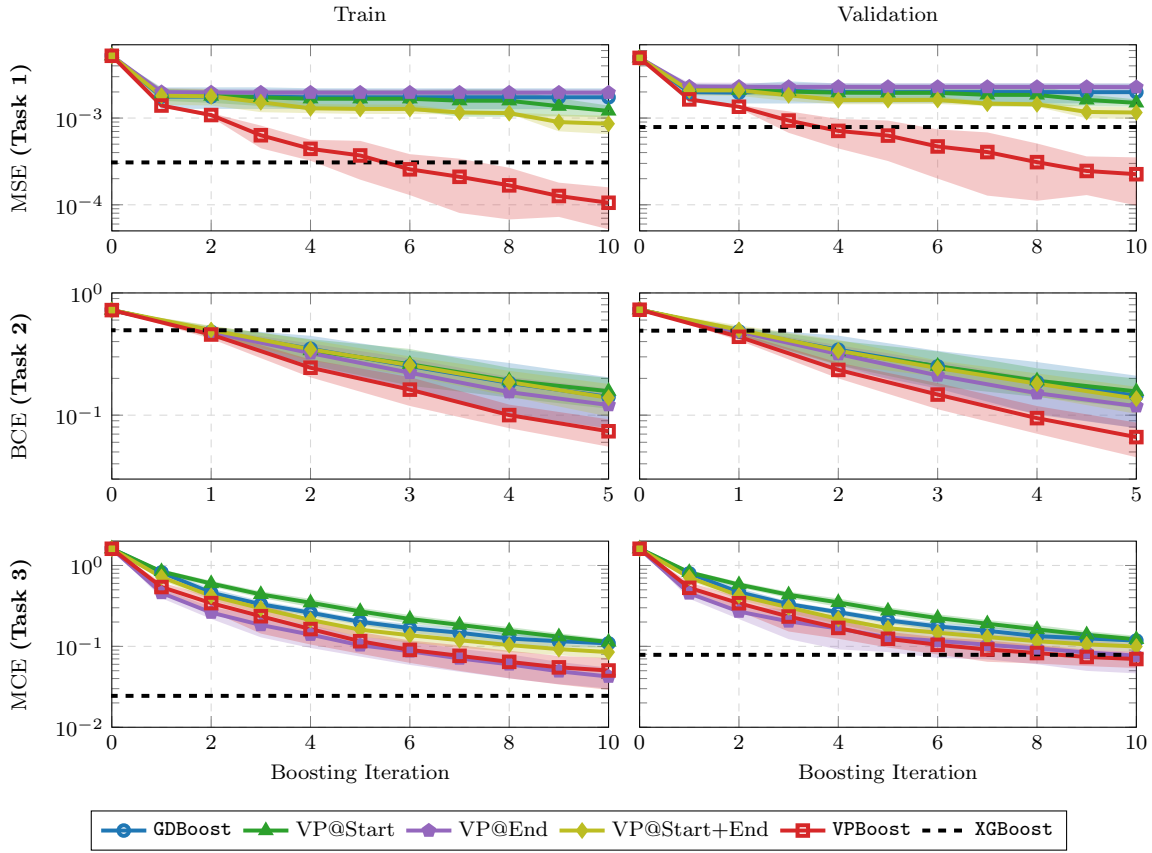

\begin{table}[!htp]
\centering
\footnotesize

\def\resultsDirA{data/toy_regression/ensemble_metrics_summary/test}
    
	\pgfplotstableread[col sep=comma]{\resultsDirA/optimal_gb_em_summary.csv}\loadedtabletestA

\pgfplotsforeachungrouped \method in {vp_at_start, vp_at_end, vp_at_start_and_end, vp, xgboost, full_gb, full_vp} {\pgfplotstableread[col sep=comma]{\resultsDirA/optimal_\method_em_summary.csv}\loadedtabletmpA

        \pgfplotstablevertcat{\loadedtabletestA}{\loadedtabletmpA}
        \pgfplotstableclear{\loadedtabletmpA}
    }

\createcolmean{\loadedtabletestA}{target_metric}{targetmeanstdA}
     \createcolmean{\loadedtabletestA}{skl_r2}{r2meanstdA}
    
\def\resultsDirB{data/swiss_roll/ensemble_metrics_summary/test}
    
	\pgfplotstableread[col sep=comma]{\resultsDirB/optimal_gb_em_summary.csv}\loadedtabletestB

\pgfplotsforeachungrouped \method in {vp_at_start, vp_at_end, vp_at_start_and_end, vp, xgboost, full_gb, full_vp} {\pgfplotstableread[col sep=comma]{\resultsDirB/optimal_\method_em_summary.csv}\loadedtabletmpB

        \pgfplotstablevertcat{\loadedtabletestB}{\loadedtabletmpB}
        \pgfplotstableclear{\loadedtabletmpB}
    }

\createcolmean{\loadedtabletestB}{target_metric}{targetmeanstdB}
    \pgfplotstablecreatecol[copy column from table={\loadedtabletestB}{targetmeanstdB}]{targetmeanstdB}{\loadedtabletestA}

 	\createcolmean{\loadedtabletestB}{skl_auc}{aucB}
        \pgfplotstablecreatecol[copy column from table={\loadedtabletestB}{aucB}]{aucB}{\loadedtabletestA}

 \createcolmean{\loadedtabletestB}{skl_accuracy}{accB}
        \pgfplotstablecreatecol[copy column from table={\loadedtabletestB}{accB}]{accB}{\loadedtabletestA}

\def\resultsDirC{data/level_sets/ensemble_metrics_summary/test}
    
	\pgfplotstableread[col sep=comma]{\resultsDirC/optimal_gb_em_summary.csv}\loadedtabletestC

\pgfplotsforeachungrouped \method in {vp_at_start, vp_at_end, vp_at_start_and_end, vp, xgboost, full_gb, full_vp} {\pgfplotstableread[col sep=comma]{\resultsDirC/optimal_\method_em_summary.csv}\loadedtabletmpC

        \pgfplotstablevertcat{\loadedtabletestC}{\loadedtabletmpC}
        \pgfplotstableclear{\loadedtabletmpC}
    }

\createcolmean{\loadedtabletestC}{target_metric}{targetmeanstdC}
     \pgfplotstablecreatecol[copy column from table={\loadedtabletestC}{targetmeanstdC}]{targetmeanstdC}{\loadedtabletestA}

    \createcolmean{\loadedtabletestC}{skl_auc}{aucC}
        \pgfplotstablecreatecol[copy column from table={\loadedtabletestC}{aucC}]{aucC}{\loadedtabletestA}

	 \createcolmean{\loadedtabletestC}{skl_accuracy}{accC}
        \pgfplotstablecreatecol[copy column from table={\loadedtabletestC}{accC}]{accC}{\loadedtabletestA}

\pgfkeys{
    /pgf/fpu = true,
    /pgf/number format/.cd,
    precision=4,
    fixed,
    fixed zerofill,
1000 sep={.}
}

\pgfplotsset{topentrystyle/.style={/pgfplots/table/@cell content/.add={\cellcolor{lightgray!50}\boldmath}{},}}
	
\def\h{0.5}
    \pgfplotstabletypeset[
    	test_metrics_table_style,
columns={methodnamewfull, targetmeanstdA, r2meanstdA, targetmeanstdB, aucB, accB, targetmeanstdC, aucC, accC},
columns/targetmeanstdA/.style={column type={@{\hspace{\h cm}}c}, string type, column name=\multicolumn{1}{c}{MSE ($\downarrow$)}},
        columns/r2meanstdA/.style={column type=c, string type, column name=\multicolumn{1}{c}{$R^2$ ($\uparrow$)}},
        columns/targetmeanstdB/.style={column type={@{\hspace{\h cm}}c}, string type, column name=\multicolumn{1}{c}{BCE ($\downarrow$)}},
        columns/aucB/.style={column type=c, string type, column name=\multicolumn{1}{c}{AUC ($\uparrow$)}},
        columns/accB/.style={column type=c, string type, column name=\multicolumn{1}{c}{Acc. ($\uparrow$)}},
        columns/targetmeanstdC/.style={column type={@{\hspace{\h cm}}c}, string type, column name=\multicolumn{1}{c}{MCE ($\downarrow$)}},
        columns/aucC/.style={column type=c, string type, column name=\multicolumn{1}{c}{AUC ($\uparrow$)}},
         columns/accC/.style={column type=c, string type, column name=\multicolumn{1}{c}{Acc. ($\uparrow$)}},
every row 4 column 1/.style={
            postproc cell content/.append style={/pgfplots/table/@cell content/.add={\cellcolor{lightgray!50}\boldmath}{}}},
         every row 4 column 2/.style={
            postproc cell content/.append style={/pgfplots/table/@cell content/.add={\cellcolor{lightgray!50}\boldmath}{}}},
     	  every row 4 column 3/.style={
            postproc cell content/.append style={/pgfplots/table/@cell content/.add={\cellcolor{lightgray!50}\boldmath}{}}},
          every row 4 column 4/.style={
            postproc cell content/.append style={/pgfplots/table/@cell content/.add={\cellcolor{lightgray!50}\boldmath}{}}},
every row 4 column 5/.style={
            postproc cell content/.append style={/pgfplots/table/@cell content/.add={\cellcolor{lightgray!50}\boldmath}{}}},
every row 4 column 6/.style={
            postproc cell content/.append style={/pgfplots/table/@cell content/.add={\cellcolor{lightgray!50}\boldmath}{}}},
         every row 4 column 7/.style={
            postproc cell content/.append style={/pgfplots/table/@cell content/.add={\cellcolor{lightgray!50}\boldmath}{}}},
          every row 4 column 8/.style={
            postproc cell content/.append style={/pgfplots/table/@cell content/.add={\cellcolor{lightgray!50}\boldmath}{}}},
every nth row={6}{before row=\hdashline\\[-0.9em]}, 
        every row no 5/.style={after row={\\[-0.9em]}},
every row 7 column 1/.style={
            postproc cell content/.append style={
                /pgfplots/table/@cell content/.add={\cellcolor{lightgray!50}\boldmath}{},
            }
        },
        every row 7 column 2/.style={
            postproc cell content/.append style={
                /pgfplots/table/@cell content/.add={\cellcolor{lightgray!50}\boldmath}{},
            }
        },
        every row 7 column 3/.style={
            postproc cell content/.append style={
                /pgfplots/table/@cell content/.add={\cellcolor{lightgray!50}\boldmath}{},
            }
        }, 
          every row 7 column 4/.style={
            postproc cell content/.append style={
                /pgfplots/table/@cell content/.add={\cellcolor{lightgray!50}\boldmath}{},
            }
        }, 
            every row 7 column 5/.style={
            postproc cell content/.append style={
                /pgfplots/table/@cell content/.add={\cellcolor{lightgray!50}\boldmath}{},
            }
        }, 
            every row 7 column 6/.style={
            postproc cell content/.append style={
                /pgfplots/table/@cell content/.add={\cellcolor{lightgray!50}\boldmath}{},
            }
        }, 
             every row 7 column 7/.style={
            postproc cell content/.append style={
                /pgfplots/table/@cell content/.add={\cellcolor{lightgray!50}\boldmath}{},
            }
        }, 
              every row 7 column 8/.style={
            postproc cell content/.append style={
                /pgfplots/table/@cell content/.add={\cellcolor{lightgray!50}\boldmath}{},
            }
        }, 
every head row/.style={
        before row={\toprule
& \multicolumn{2}{c}{\bfseries Task 1} & \multicolumn{3}{c}{\bfseries Task 2} & \multicolumn{3}{c}{\bfseries Task 3}  \\
            \cmidrule(lr){2-3}  \cmidrule(lr){4-6}  \cmidrule(lr){7-9}
        },
        after row=\midrule },
    ]{\loadedtabletestA}

\caption{Test metrics across three synethatic datasets. All results show the mean over $10$ random featurizer initializations.  The standard derivations of the losses were roughly one order of magnitude smaller than the mean. 
Below the dashed line, we compare to training full NNs with similar capacity to the ensemble. 
The top metrics per column for boosting and full training are bolded and highlighted in gray. 
Over all boosting methods, \VPBoost achieves the lowest loss per task.}
\label{tab:toy_experiment_metrics}

\end{table} 

\begin{figure}[t]
\centering

\begin{subfigure}{\linewidth}
\centering
\includegraphics[width=0.85\linewidth]{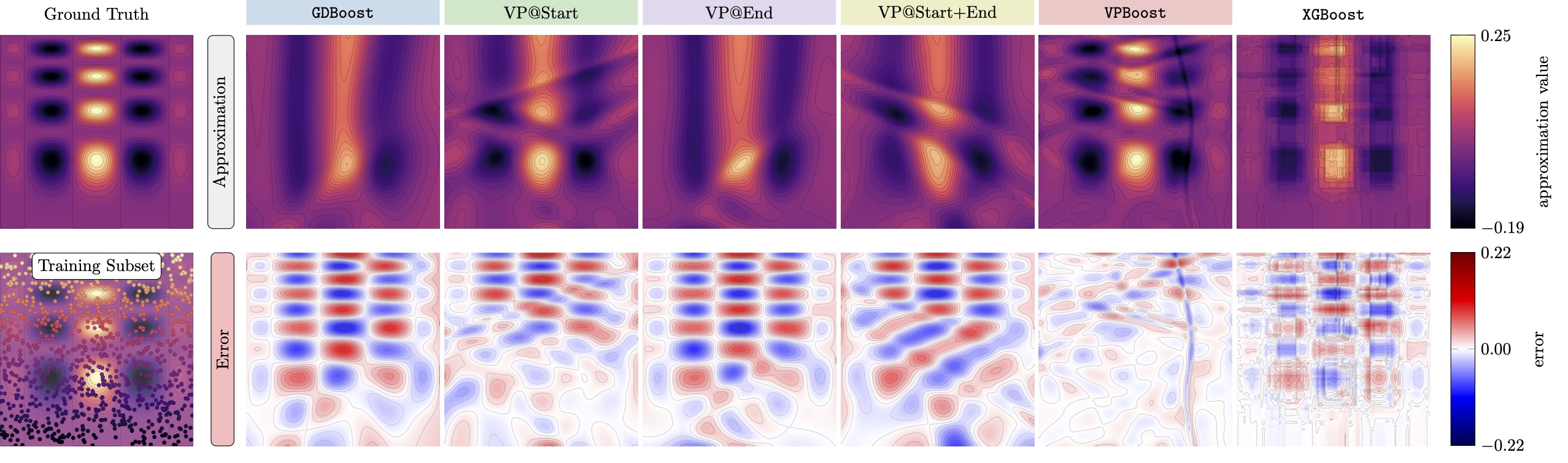}
\caption{{\bfseries (Task 1)} Approximation of 2D smooth function. 
Left column: ground truth and subset of points used during training. 
Top row: approximations from final ensembles for a single seed. 
Bottom row: corresponding $L^2$ error. 
\VPBoost smoothly captures highly oscillatory regions and produces the lowest error.}
\label{fig:toy_regression_task1}
\end{subfigure}

\begin{subfigure}{\linewidth}
\centering
\includegraphics[width=0.85\linewidth]{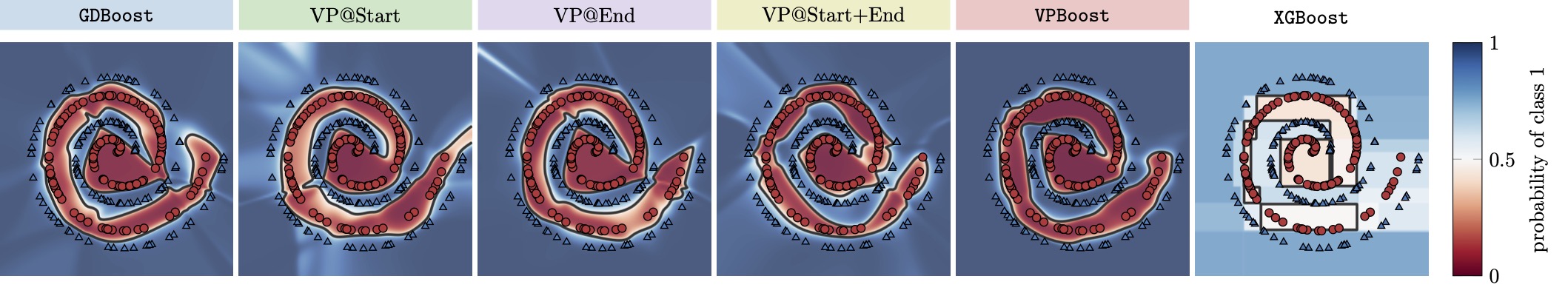}
\caption{{\bfseries (Task 2)} Decision bound of 2D binary classification. 
All NN methods are able to capture the spiral shape well. 
The decision-tree-based \XGBoost struggles due to the lack of regularity of the decision boundary.}
\end{subfigure}

\begin{subfigure}{\linewidth}
\centering
\includegraphics[width=0.85\linewidth]{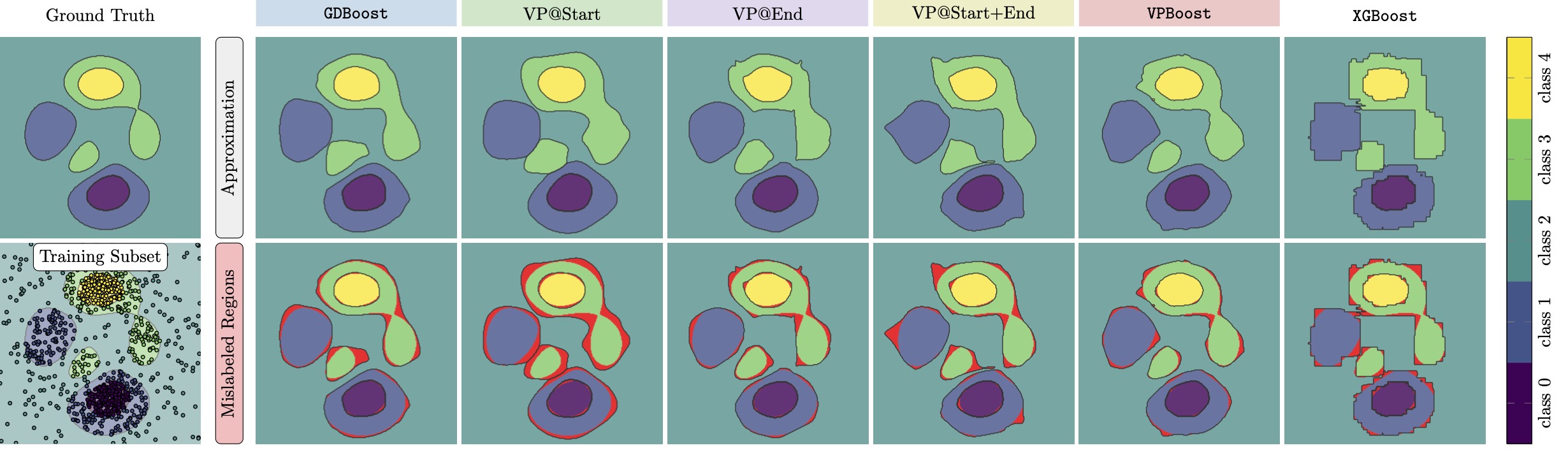} \caption{{\bfseries (Task 3)} Labels of 2D multi-class classification.  
Left column: ground truth and subset of points used during training. 
Top row: label predictions from final ensembles for a single seed. 
Bottom row: predictions with mislabeled regions highlighted in red. 
\VPBoost achieved the lowest misclassification rate. 
}
\end{subfigure}

\caption{Visualization of final boosting approximations across the three synthetic tasks.}
\label{fig:toy_experiment_approx}

\end{figure} 
\paragraph{Sensitivity of VarPro to MCE Loss}  
The use of VarPro on a quadratic approximation of MCE (Task 3) was sensitive to the choice of hyperparameters and featurizer initialization. 
This is because the MCE Hessian and consequently the reduced Hessian, $\caH$, are necessarily rank deficient \citep{kan_lsemink_2023}.  
Thus, without sufficiently large regularization, the optimal linear weights explode in magnitude, leading to an unreasonably large ensemble update. 
The trust-region protects against this case precisely by increasing the regularization parameter (Algorithm~\ref{alg:trust_region_skeleton_vpboost}
Line~\ref{alg_line:model_confidence}).

\paragraph{Comparison to Full NNs}  

To complete a thorough study of \VPBoost on the illustrative examples, we compare to a full neural network baseline (Full NN) and 
its variable projection counterpart (Full NN+VP). 
Each full NN has a comparable parameter count to the corresponding boosted
ensemble and is trained for the same number of epochs used to train a single
weak learner, with hyperparameters tuned using the same validation protocol.
This ensures that the comparison reflects the effect of the boosting strategy
itself rather than differences in model capacity or tuning effort. 
This ensures that the comparison reflects the effect of the boosting 
strategy itself rather than differences in model capacity or tuning effort. 

We stress that the full NN baselines in Table~\ref{tab:toy_experiment_metrics}
 are not directly comparable to the
boosting ensembles: a full NN (with VP) is
trained end-to-end on the true loss, whereas \VPBoost optimizes a quadratic
approximation of the loss at each step.
This distinction is especially consequential for non-quadratic losses: on
Tasks~2 and~3, Full NN+VP must numerically optimize its linear weights at
every gradient evaluation, resulting in
longer training times than \VPBoost
(Figure~\ref{fig:time_vs_loss}).
On Task~1, the MSE loss is quadratic, so Full NN+VP admits a closed-form
solution for its linear weights analogous to \eqref{eq:w_opt_form}, which
moderates its training cost relative to Tasks~2 and~3, as reflected in
Figure~\ref{fig:time_vs_loss}.
We also note that the current \VPBoost implementation is not optimized for
runtime; the timing results in Figure~\ref{fig:time_vs_loss} are therefore a
conservative estimate of the achievable speedup.
\begin{wrapfigure}{r}{0.5\linewidth}
\centering
\begin{tikzpicture}
	\scriptsize 
	\begin{groupplot}[
	group style={group size=3 by 1, 
		ylabels at=edge left, yticklabels at=edge left, 
horizontal sep=0.025\linewidth, vertical sep=0.05\linewidth}, 
scale only axis,
width=0.25\linewidth, height=0.25\linewidth, 		
xlabel={Training Time}, 
		ylabel style={anchor=north, yshift=0.6cm}, 
ymode=log, 
ylabel={Test Loss}, 
grid=major,
		grid style={dashed, gray!30},  
legend to name=sharedlegend, 
		legend columns=2, 
        transpose legend,
		execute at begin axis={
    		\addlegendimage{style_gb, only marks, mark size=3pt}  \addlegendentry{\GDBoost}
			\addlegendimage{style_full_gb, only marks, mark size=3pt}  \addlegendentry{Full NN}
			\addlegendimage{style_vp, only marks, mark size=3pt}  \addlegendentry{\VPBoost}
			\addlegendimage{style_full_vp, only marks, mark size=3pt}  \addlegendentry{Full NN+VP}
			\addlegendimage{style_xgboost, only marks, mark=x, solid, mark size=3pt}   \addlegendentry{\XGBoost}
}, 
		legend style={/tikz/every even column/.append style={column sep=0.25cm}}
	]

	\pgfplotsset{
    m1/.style={mark size=3pt, error bars/.cd, x dir=both, y dir=both, x explicit, y dir=both, y explicit}, 
    m2/.style={mark=x, solid}
    }	\def\resultsDirA{data/toy_regression/ensemble_metrics_summary/test}
		\def\resultsDirB{data/swiss_roll/ensemble_metrics_summary/test}
		\def\resultsDirC{data/level_sets/ensemble_metrics_summary/test}
		
		\nextgroupplot[title={MSE {\bfseries (Task 1)}}]
		\foreach \method in {gb, vp, full_gb, full_vp}{
			\edef\tmp{
				\noexpand\addplot[style_\method, m1]
				table[  x=total_training_time_mean, 
                        x error=total_training_time_std,
                        y=target_metric_mean, 
                        y error=target_metric_std, 
                        col sep=comma] {\resultsDirA/optimal_\method_em_summary.csv};
			}\tmp
		}
		
		\addplot[style_xgboost, m2, m1] 
                table[  x=total_training_time_mean, 
                        x error=total_training_time_std,
                        y=target_metric_mean, 
                        y error=target_metric_std, 
                        col sep=comma] {\resultsDirA/optimal_xgboost_em_summary.csv};
		
		\nextgroupplot[title={BCE {\bfseries (Task 2)}}]
		\foreach \method in {gb, vp, full_gb, full_vp}{
			\edef\tmp{
				\noexpand\addplot[style_\method, m1]
				table[  x=total_training_time_mean, 
                        x error=total_training_time_std,
                        y=target_metric_mean, 
                        y error=target_metric_std, 
                        col sep=comma] {\resultsDirB/optimal_\method_em_summary.csv};
			}\tmp
		}
		
		\addplot[style_xgboost, m2, m1] 
                table[  x=total_training_time_mean, 
                        x error=total_training_time_std,
                        y=target_metric_mean, 
                        y error=target_metric_std, 
                        col sep=comma] {\resultsDirB/optimal_xgboost_em_summary.csv};

        \nextgroupplot[title={MCE {\bfseries (Task 3)}}]
		\foreach \method in {gb, vp, full_gb, full_vp}{
			\edef\tmp{
				\noexpand\addplot[style_\method, m1]
				table[  x=total_training_time_mean, 
                        x error=total_training_time_std,
                        y=target_metric_mean, 
                        y error=target_metric_std, 
                        col sep=comma] {\resultsDirC/optimal_\method_em_summary.csv};
			}\tmp
		}
		
		\addplot[style_xgboost, m2, m1] 
                table[  x=total_training_time_mean, 
                        x error=total_training_time_std,
                        y=target_metric_mean, 
                        y error=target_metric_std, 
                        col sep=comma] {\resultsDirC/optimal_xgboost_em_summary.csv};

	\end{groupplot}
	
\node at ($(group c1r1.south west)!0.5!(group c3r1.south east)$) [anchor=north, yshift=-1cm] {\pgfplotslegendfromname{sharedlegend}
};

\end{tikzpicture} \caption{Mean training time versus test loss per task over $10$ trials. 
Lower left corner is best.
Each point includes error bars showing one standard deviation in each direction. 
}
\label{fig:time_vs_loss}
\end{wrapfigure}
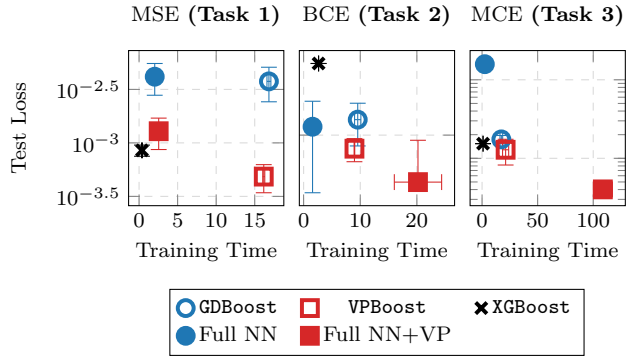

With this caveat in mind, the results in
Table~\ref{tab:toy_experiment_metrics} and Figure~\ref{fig:time_vs_loss}
tell a consistent story. 
Across all tasks, Full NN+VP outperforms Full NN, further evidencing that VarPro is beneficial beyond boosting.
In terms of  test loss, \VPBoost outperforms Full NN on all three tasks and also outperforms Full
NN+VP on Task~1. 
In the latter case, \VPBoost is able to resolve the highly oscillatory structure
of the target function by sequentially fitting the residual, as shown in Figure~\ref{fig:toy_regression_task1}.
On Tasks~2 and~3, Full NN+VP achieves a lower test loss than \VPBoost because applies VarPro directly to the true non-quadratic loss directly.   
In terms of time, \XGBoost and Full NN are consistently the fastest algorithms.\footnote{We would expect \XGBoost to win in terms of time if the Full NN used a larger architecture.} 
The Full NN+VP is among the fastest in Task~1 when the optimal linear weights have a closed-form solution, and significantly slower for the non-quadratic losses. 
\GDBoost and \VPBoost obtain similar speeds, both to each other and across tasks. 
This indicates that optimizing the linear weights with a closed-form solution has minimal overhead (see Section~\ref{subsec:scaling} for more details) and that training time is independent of the loss function. 
However, both sequentially-trained NN boosting strategies are among the slower methods compared.

The trade-off between loss function, test loss,  training time, and reliability is made explicit in Figure~\ref{fig:time_vs_loss}.  
The choice ultimately depends on the  priorities. If speed is the priority, \XGBoost is the clear choice. 
If accuracy is the priority, Full NN+VP may be preferred, though \VPBoost may be the better choice to approximate highly oscillatory functions. 
If speed and smoothness are the goal for a non-quadratic loss, \VPBoost is a strong candidate. 
If convergence is needed, \VPBoost provides the strongest guarantees.

 \subsection{Real-World Benchmarks}
\label{subsec:real_world_benchmarks}

We supplement the toy experiments with three nontrivial real-world benchmarks: MNIST handwritten digit classification~\citep{lecun_mnist_2010}, a scientific machine learning regression task for a convection-diffusion-reaction (CDR) system~\citep{newman_train_2021}, and the Higgs binary classification dataset~\citep{whiteson_higgs_2014}.
Together, these experiments test whether \VPBoost remains effective beyond the small-scale toy examples and across a broader range of featurizer architectures. 

\subsubsection{MNIST}
\label{subsec:mnist}

We consider a convolutional featurizer for the baseline MNIST handwritten digit
classification task \citep{lecun_mnist_2010}. The MNIST dataset consists of
60,000 training and 10,000 test grayscale images of size $28\times 28$. Each
image contains one handwritten digit located roughly in the center of the
domain. The goal is to approximate the labeling function
$y: \Rbb^{28 \times 28} \to \{0,\dots, 9\}$, where $y$ returns the digit
contained in the image. We construct our loss over a $5,000$ image subset of
the entire training dataset, and set aside $10,000$ other training images as a
validation dataset for hyperparameter tuning.

Convolutional neural networks (CNNs) are state-of-the-art for computer vision
tasks \citep{krizhevsky_imagenet_2017}. With $\VPBoost$, we train an ensemble
of weak CNNs to classify MNIST digits. 
We design weak learners to be similar
to the suggested CNN architecture in the Equinox
tutorial\footnote{\url{https://docs.kidger.site/equinox/examples/mnist/}}
\citep{kidger_equinox_2021}, but with approximately three orders of magnitude less capacity. 
In the current implementation, each weak learner uses a stacked convolutional
featurizer mapping to $N_{\rm feat} = 144$ features, as depicted in Figure~\ref{fig:mnist_cnn_weak_learner}.
Under the
dense output map described in
Section~\ref{subsec:implementation_details}, the final linear layer maps these
$144$ features to logits in $\Rbb^{10}$, so $\NW = 10\cdot 144 = 1440$. 
See Figure~\ref{fig:mnist_cnn_weak_learner} for a depiction of the architecture. 
The
convolutional featurizer itself has
$\NTheta = 65 + 17 + 5 = 87$ trainable
parameters, corresponding to the three convolutional layers and their scalar biases.
We build an ensemble of $5$ weak learners, resulting in a total capacity of
$5(\NTheta+\NW)=7635$. This is compared against an \XGBoost ensemble comprised
of $20$ learners with max depth $5$.

	\begin{figure}[!htp]
	\centering
        \scriptsize
		\begin{tikzpicture}[scale=0.75]
			\def\s{2.5}
			\def\b{2.75}

			\pgfmathsetmacro{\X}{0}
			\def\n{3}
			\draw[fill=lightgray!25] (\X, -\n/2, \n/2) -- (\X, -\n/2, -\n/2) -- (\X, \n/2, -\n/2) -- (\X, \n/2, \n/2) -- cycle;  
			\def\p{10}
			\pgfmathtruncatemacro{\q}{\p-1}
			\foreach \i in {1, ..., \q}{
				\draw (\X, -\n/2, \n/2 - \i * \n / \p) -- (\X, \n/2, \n/2 - \i * \n / \p);
				\draw (\X, -\n/2 + \i * \n / \p, \n/2) -- (\X, -\n/2 + \i * \n / \p, -\n/2);
			}
			\def\ii{2} \def\jj{1}
			\coordinate (A) at (\X, -\n/2 + \ii * \n / \p, \n/2 - \jj * \n / \p);
			\coordinate (B) at (\X, -\n/2 + \ii * \n / \p, \n/2 -  \jj * \n / \p - 4 * \n / \p);
			\coordinate (C) at (\X, -\n/2 +  \ii * \n / \p + 4 * \n / \p, \n/2 -  \jj * \n / \p - 4 * \n / \p);
			\coordinate (D) at (\X, -\n/2 +  \ii * \n / \p + 4 * \n / \p, \n/2 - \jj * \n / \p);
			\draw[line width=2pt, color=red, fill=red, fill opacity=0.5, draw opacity=1] (A) -- (B) -- (C) -- (D) -- cycle;
			
			\node at (\X, -\b) (input) {$\Rbb^{28 \times 28}$};
			
			\pgfmathsetmacro{\X}{\X+\s}
			\def\n{2.1}
			\def\p{7}
			\pgfmathtruncatemacro{\q}{\p-1}
			\draw[fill=lightgray!25] (\X, -\n/2, \n/2) -- (\X, -\n/2, -\n/2) -- (\X, \n/2, -\n/2) -- (\X, \n/2, \n/2) -- cycle;  
			\foreach \i in {1, ..., \q}{
				\draw (\X, -\n/2, \n/2 - \i * \n / \p) -- (\X, \n/2, \n/2 - \i * \n / \p);
				\draw (\X, -\n/2 + \i * \n / \p, \n/2) -- (\X, -\n/2 + \i * \n / \p, -\n/2);
			}
			
			\coordinate (a) at (\X, -\n/2 + \ii * \n / \p, \n/2 - \jj * \n / \p);
			\coordinate (b) at (\X, -\n/2 + \ii * \n / \p, \n/2 -  \jj * \n / \p - \n / \p);
			\coordinate (c) at (\X, -\n/2 +  \ii * \n / \p +  \n / \p, \n/2 -  \jj * \n / \p - \n / \p);
			\coordinate (d) at (\X, -\n/2 +  \ii * \n / \p +  \n / \p, \n/2 - \jj * \n / \p);
			
			\fill[red] (a) -- (b) -- (c) -- (d) -- cycle;
			\draw[red] (A) -- (a);
			\draw[red] (B) -- (b);
			\draw[red] (C) -- (c);
			\draw[red] (D) -- (d);

			\def\ii{3} \def\jj{3}
			\coordinate (A) at (\X, -\n/2 + \ii * \n / \p, \n/2 - \jj * \n / \p);
			\coordinate (B) at (\X, -\n/2 + \ii * \n / \p, \n/2 -  \jj * \n / \p - 4 * \n / \p);
			\coordinate (C) at (\X, -\n/2 +  \ii * \n / \p + 4 * \n / \p, \n/2 -  \jj * \n / \p - 4 * \n / \p);
			\coordinate (D) at (\X, -\n/2 +  \ii * \n / \p + 4 * \n / \p, \n/2 - \jj * \n / \p);
			\draw[line width=2pt, color=black, opacity=1] (A) -- (B) -- (C) -- (D) -- cycle;
			
			\node at (\X, -\b) (conv2d) {$\Rbb^{21 \times 21}$};
			
			\pgfmathsetmacro{\X}{\X+\s}
			\def\n{1.2}
			\def\p{4}
			\pgfmathtruncatemacro{\q}{\p-1}
			\draw[fill=lightgray!25] (\X, -\n/2, \n/2) -- (\X, -\n/2, -\n/2) -- (\X, \n/2, -\n/2) -- (\X, \n/2, \n/2) -- cycle; 
			\foreach \i in {1, ..., \q}{
				\draw (\X, -\n/2, \n/2 - \i * \n / \p) -- (\X, \n/2, \n/2 - \i * \n / \p);
				\draw (\X, -\n/2 + \i * \n / \p, \n/2) -- (\X, -\n/2 + \i * \n / \p, -\n/2);
			}
			
			\coordinate (a) at (\X, -\n/2 + \ii * \n / \p, \n/2 - \jj * \n / \p);
			\coordinate (b) at (\X, -\n/2 + \ii * \n / \p, \n/2 -  \jj * \n / \p - \n / \p);
			\coordinate (c) at (\X, -\n/2 +  \ii * \n / \p +  \n / \p, \n/2 -  \jj * \n / \p - \n / \p);
			\coordinate (d) at (\X, -\n/2 +  \ii * \n / \p +  \n / \p, \n/2 - \jj * \n / \p);
			
			\fill[black] (a) -- (b) -- (c) -- (d) -- cycle;
			\draw[black] (A) -- (a);
			\draw[black] (B) -- (b);
			\draw[black] (C) -- (c);
			\draw[black] (D) -- (d);
			
			\draw[rounded corners=3pt, line width=1pt] (-0.75, -2.25) rectangle (\X+0.4, 2.25); 
			\node at (\X+0.5, 0) (tmp0) {};
			\node[anchor=south] at (\X/2-0.25, 2.25) {\scriptsize $\text{Conv2D}_{8\times 8}$ + $\text{MaxPool2D}_{4\times 4}$};
			
			\node at (\X, -\b) (maxpool) {$\Rbb^{18\times 18}$};
			
			\pgfmathsetmacro{\X}{\X+\s}
\node[anchor=center, draw, rotate=90, rounded corners=3pt, line width=1pt, minimum width=3.5cm, outer sep=4pt] (tmp1) at (\X, 0) {\tiny $\text{Conv2D}_{4\times 4}$ + $\text{MaxPool2D}_{2\times 2}$};
			
			\node at (\X, -\b) (convmax1) {$\Rbb^{14\times 14}$};
			
			\pgfmathsetmacro{\X}{\X+\s}
\node[anchor=center, draw, rotate=90, rounded corners=3pt, line width=1pt, minimum width=3.5cm, outer sep=4pt, ] (tmp2) at (\X, 0) {\tiny $\text{Conv2D}_{2\times 2}$ + $\text{MaxPool2D}_{2\times 2}$};
			
			\node at (\X, -\b) (convmax2) {$\Rbb^{12\times 12}$};
			
			\pgfmathsetmacro{\X}{\X+0.5}
			\draw[->, line width=2pt] (\X, 0) --  node[midway, above] {\scriptsize vectorize} (\X+\s - 0.25, 0);

			\pgfmathsetmacro{\X}{\X+\s}
			
			\draw[fill=lightgray!25, rounded corners=1pt] (\X-0.1, 0.3/2*8+0.05) rectangle (\X+0.1, -0.3/2*8-0.05);
			\foreach \i in {1, ..., 16}{
				\draw[fill=white] (\X, 0.3/2*8-\i*0.3/2+0.3/4) circle (0.05);
				\node[outer sep=0pt, inner sep=0pt]  at (\X, 0.3/2*8-\i*0.3/2+0.3/4)  (q\i) {};
			}

			\node at (\X, -\b) (vectorized) {$\Rbb^{144}$};
			
			\pgfmathsetmacro{\X}{\X+\s/2}

			\draw[fill=lightgray!25, rounded corners=1pt] (\X-0.1, 0.3/2*5+0.05) rectangle (\X+0.1, -0.3/2*5-0.05);
			\foreach \i in {1, ..., 10}{
				\draw[fill=white] (\X, 0.3/2*5-\i*0.3/2+0.3/4) circle (0.05);
				\node[outer sep=0pt,  inner sep=0pt]  at (\X, 0.3/2*5-\i*0.3/2+0.3/4) (s\i) {};
			}
			
			\node at (\X, -\b) (linear2) {$\Rbb^{10}$};
			
			\foreach \i in {1, ..., 16}{
				\foreach \j in {1, ..., 10}{
					\draw[blue] (q\i) -- (s\j);
				}
			}

			\draw[->, line width=2pt] (tmp1.south) -- (tmp2.north);
			\draw[->, line width=2pt] (tmp0) -- (tmp1.north);
			
			\draw[->] (input) -- (conv2d);
			\draw[->] (conv2d) -- (maxpool);
			\draw[->] (maxpool) -- (convmax1);
			\draw[->] (convmax1) -- (convmax2);
			\draw[->] (convmax2) -- (vectorized);
			\draw[->, thick, blue] (vectorized) -- (linear2);

		\end{tikzpicture}
	\caption{MNIST CNN weak learner. 
        The featurizer consists of three consecutive blocks, each
performing a convolution, max pool, and a ReLU activation.
Each block respectively applies a single 2D convolutional kernel of sizes $(8,4,2)$
and 2D max pooling of sizes $(4,2,2)$.
All convolutions and pooling operators use stride $1$ with zero padding. 
The vectorized output of the convolutional featurizer is of dimension $\NFeat = 144$.}
	\label{fig:mnist_cnn_weak_learner}
	\end{figure}
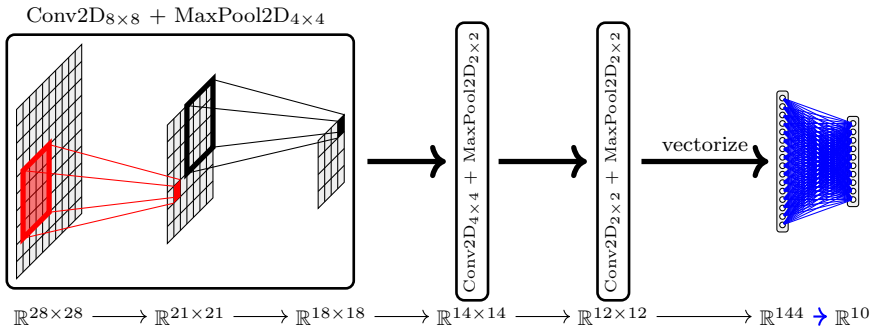 

Figure~\ref{fig:05_04_mnist} shows that \VPBoost transfers cleanly to this
convolutional setting. In particular, the full \VPBoost variant outperforms
all others on all collected test metrics, achieving lower MCE together with
higher AUC and accuracy. This is notable as \XGBoost is a strong baseline on
structured prediction problems, whereas here \VPBoost is built from small CNN
weak learners whose feature maps are learned directly from the image domain.
These results therefore provide strong empirical evidence that the VP
formulation is not tied to a particular featurizer class. The same boosting
mechanism remains effective when the weak learner architecture is changed from
the MLP and residual settings considered elsewhere in the experiments to a
convolutional architecture tailored to vision data. The train and validation
curves also suggest a qualitative advantage in generalization. While \XGBoost
achieves the best training metrics, the gap to its validation performance
remains visibly larger than for \VPBoost. By contrast, \VPBoost degrades more
gracefully, which may reflect the inductive bias of the neural featurizer and
the extra flexibility of learning feature representations jointly with the
linear output map at each boosting step. Recall, all models are trained on the
same $5,000$ image subsample of the entire $50,000$ image MNIST dataset.

\begin{figure}[!htp]
    \centering
\def\resultsDir{data/mnist/ensemble_metrics_summary}
        \def\ymin{5e-3}
        \def\ymax{4e0}
        \def\xmin{0}
        \def\xmax{5}
        \def\ylabel{MCE}

\begin{tikzpicture}

\scriptsize 
\begin{groupplot}[group style={group size=2 by 1, 
		ylabels at=edge left, yticklabels at=edge left, 
		xlabels at=edge bottom, xticklabels at=edge bottom, 
		horizontal sep=0.025\linewidth}, 
scale only axis,
width=0.4\linewidth, height=0.2\linewidth, 		
xlabel={Boosting Iteration}, 
		xmin=\xmin, xmax=\xmax,
ylabel={\ylabel}, 
		ymin=\ymin, ymax=\ymax, 
		ymode=log, 
ticklabel style={/pgf/number format/fixed, /pgf/number format/precision=0}, 
grid=major,
		grid style={dashed, gray!30},  
legend to name=sharedlegend, 
		legend columns=6,
		execute at begin axis={
    			\addlegendimage{style_gb}  \addlegendentry{\GDBoost}
    			\addlegendimage{style_vp}   \addlegendentry{\VPBoost}
			\addlegendimage{style_vp_at_start}   \addlegendentry{VP@Start}
			\addlegendimage{style_vp_at_end}   \addlegendentry{VP@End}
			\addlegendimage{style_vp_at_start_and_end}   \addlegendentry{VP@Start+End}
			\addlegendimage{style_xgboost}   \addlegendentry{\XGBoost}
		}, 
		legend style={/tikz/every even column/.append style={column sep=0.25cm}}
		]

\nextgroupplot[title=Train]
	\def\dataType{train}
		\foreach \methodType in {gb, vp, vp_at_start, vp_at_end, vp_at_start_and_end}{
      		  	\edef\temp{
				\noexpand\targetplotB{\resultsDir/\dataType/optimal_\methodType_em_summary.csv}{style_\methodType}
			}\temp
    		}
		
		\pgfplotstableread[col sep=comma]{\resultsDir/\dataType/optimal_xgboost_em_summary.csv}\xgboosttable
		\pgfplotstablegetrowsof{\xgboosttable}
		\pgfmathtruncatemacro{\xgtrainlastrow}{\pgfplotsretval-1}
		\pgfplotstablegetelem{\xgtrainlastrow}{target_metric_mean}\of{\xgboosttable}
		\edef\xgtrainfinal{\pgfplotsretval}
		
		\addplot [style_xgboost, domain=-1:100] {\xgtrainfinal};

	\nextgroupplot[title=Validation]
	\def\dataType{val}
		\foreach \methodType in {gb, vp, vp_at_start, vp_at_end, vp_at_start_and_end}{
      		  	\edef\temp{
				\noexpand\targetplotB{\resultsDir/\dataType/optimal_\methodType_em_summary.csv}{style_\methodType}
				\noexpand
			}\temp
    		}
		
		\pgfplotstableread[col sep=comma]{\resultsDir/\dataType/optimal_xgboost_em_summary.csv}\xgboosttable
		\pgfplotstablegetrowsof{\xgboosttable}
		\pgfmathtruncatemacro{\xgtrainlastrow}{\pgfplotsretval-1}
		\pgfplotstablegetelem{\xgtrainlastrow}{target_metric_mean}\of{\xgboosttable}
		\edef\xgtrainfinal{\pgfplotsretval}
		
		\addplot [style_xgboost, domain=-1:100] {\xgtrainfinal};

\end{groupplot}

\node at ($(group c1r1.south west)!0.5!(group c2r1.south east)$) [anchor=north, yshift=-1cm, xshift=-0.6cm] {\pgfplotslegendfromname{sharedlegend}
};

\end{tikzpicture}

    \caption{
      Comparison of \VPBoost variants, \GDBoost, and \XGBoost on MNIST across
      5 boosting iterations. The top panel shows training and validation
      performance, and the bottom panel reports the test metrics of the best
      validation-selected ensemble for each method.
    }
    \label{fig:05_04_mnist}
\end{figure}
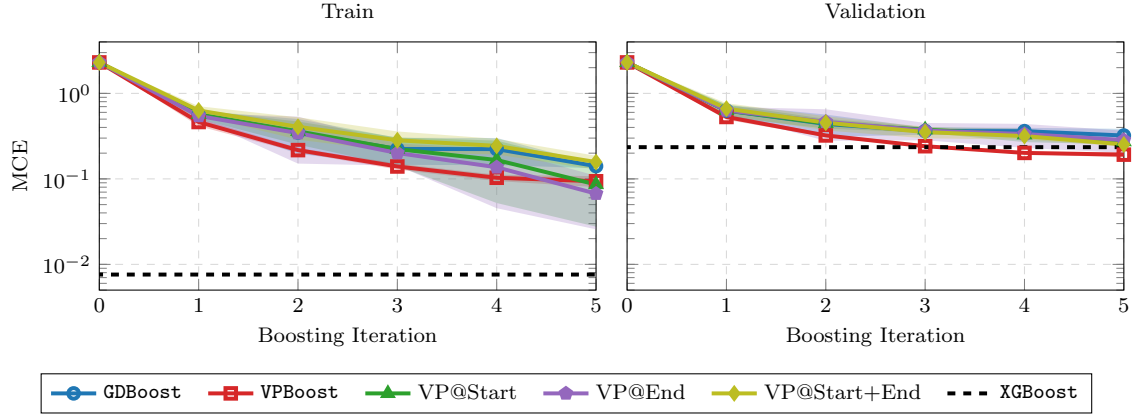

\subsubsection{Convection Diffusion Reaction Surrogate Modeling}
\label{subsec:cdr}
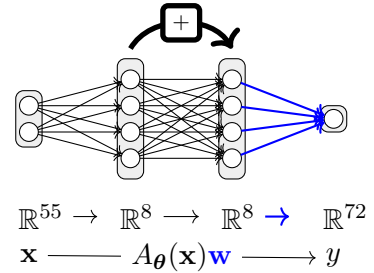
\begin{wrapfigure}{r}{0.32\linewidth}
\centering
    		\begin{tikzpicture}
		\def\s{0.1}
		\def\r{0.25}
		
		\pgfmathsetmacro{\X}{0}
		\pgfmathsetmacro{\Y}{0}
		
		\draw[rounded corners=3pt, fill=lightgray!25] (\X-0.5*\r-0.5*\s, \Y-1.5*\r-0*\s) rectangle (\X+0.5*\r+0.5*\s, \Y+1.5*\r+0*\s);
		
		\foreach[count=\a] \i in {-0.5, 0.5}{
			\node[circle, minimum width=\r cm, draw, inner sep=0pt, outer sep=0pt, fill=white] (n\a) at (\X,\Y-\i*\r-\i*\s) {};
		}

		\node at (\X, \Y-5*\r-0.5*\s) (rr) {$\phantom{{}^{55}}\Rbb^{55}$}; 
		\node[below=0.0cm of rr.south] (xx) {$\bfx$};
		
		\pgfmathsetmacro{\X}{\X+\r+\s+1}
		
		\draw[rounded corners=3pt, fill=lightgray!25] (\X-0.5*\r-0.5*\s, \Y-3*\r-0.5*\s) rectangle (\X+0.5*\r+0.5*\s, \Y+3*\r+0.5*\s);
		
		\foreach[count=\a] \i in {-1.5, -0.5, ..., 1.5}{
			\node[circle, minimum width=\r cm, draw, inner sep=0pt, outer sep=0pt, fill=white] (o\a) at (\X,\Y-\i*\r-\i*\s) {};
		}	
		
		\node at (\X, \Y+3*\r+0.5*\s) (top1) {};
		\node at (\X, \Y-5*\r-0.5*\s) (rtmp) {$\phantom{{}^8}\Rbb^8$};
		
		\draw[->] (rr) -- (rtmp);
		
		\pgfmathsetmacro{\X}{\X+\r+\s+1}
		
		\draw[rounded corners=3pt, fill=lightgray!25] (\X-0.5*\r-0.5*\s, \Y-3*\r-0.5*\s) rectangle (\X+0.5*\r+0.5*\s, \Y+3*\r+0.5*\s);
		\foreach[count=\a] \i in {-1.5, -0.5, ..., 1.5}{
			\node[circle, minimum width=\r cm, draw, inner sep=0pt, outer sep=0pt, fill=white] (p\a) at (\X,\Y-\i*\r-\i*\s) {};
		}	
		
		\node at (\X, \Y+3*\r+0.5*\s) (top2) {};
		
		\node at (\X, \Y-5*\r-0.5*\s) (rtmp2) {$\phantom{{}^8}\Rbb^8$};
		\draw[->] (rtmp) -- (rtmp2);
			
		\pgfmathsetmacro{\X}{\X+\r+\s+1}
		
		\draw[rounded corners=3pt, fill=lightgray!25] (\X-0.5*\r-0.5*\s, \Y-0.5*\r-0.5*\s) rectangle (\X+0.5*\r+0.5*\s, \Y+0.5*\r+0.5*\s);

		\foreach[count=\a] \i in {0}{
			\node[circle, minimum width=\r cm, draw, inner sep=0pt, outer sep=0pt, fill=white] (q\a) at (\X,\Y-\i*\r-\i*\s) {};
		}

		\node at (\X, \Y-5*\r-0.5*\s) (rr)  {$\phantom{{}^{72}}\Rbb^{72}$};
		\draw[->, thick, blue] (rtmp2) -- (rr);
		
		\node[below=0.0cm of rr.south] (yy) {$y$};
		\draw[->] (xx) -- node[midway, fill=white] {$A_{\bftheta}(\bfx) {\color{blue} \bfw}$} (yy);
		
		\foreach \a in {1, ..., 4}{
			\draw[->, black] (n1) -- (o\a);
			\draw[->, black] (n2) -- (o\a);
			\draw[->, thick, blue] (p\a) -- (q1);
			
			\foreach \b in {1, ..., 4}{
				\draw[->, black] (o\a) -- (p\b);
			}
		}
		
		\draw[->, line width=2] (top1.north) to[out=60, in=120, looseness=1] node[midway, fill=white, draw, rounded corners=2pt] {\scriptsize $+$}(top2.north);

	\end{tikzpicture}     \caption{Small ResNet weak learner with a single residual block.
    The upper arrow indicates the residual connection and
    the thicker blue arrow indicates the final linear mapping.}
    \label{fig:05_05_highdim_cdr_weak_learner}
\end{wrapfigure}

The CDR dataset probes a different aspect of the method: here the goal is not
only to handle a moderately large input space, but also to predict a
higher-dimensional output field. This example also provides a testbed for
moving beyond plain MLP weak learners. Prior work on this CDR benchmark uses
variable projection to train a single neural ODE model
\citep{newman_train_2021}. Here, for the same problem we use a similar
architecture, namely a small ResNet featurizer \citep{he_deep_2016}, to better
understand \VPBoost with non-MLP weak learners.

In the notation of \Cref{sec:varpro}, each CDR weak learner has the separable
form $h(\bfx)=\FeaturizerMat_{\bftheta}(\bfx)\bfw$, or equivalently
$h(\bfx)=\bfW\Featurizer_{\bftheta}(\bfx)$ with $\bfW\in\Rbb^{72\times 8}$.
The nonlinear featurizer first projects the $55$-dimensional input to width
$8$, applies a single width-$8$ residual block, and outputs an $8$-dimensional
feature vector; the linear map then lifts these features to the
$72$-dimensional target. The capacity per weak learner is $\NTheta = 520$ and
$\NW = 648$. Each boosting round introduces a newly initialized weak ResNet of
this form. We construct an ensemble of $M=30$ weak learners, resulting in an
ensemble capacity of $35{,}040$. This gives a simple residual analogue of a
neural ODE while preserving the separable VP structure used throughout the
paper.

In \Cref{fig:05_05_highdim_cdr}, the validation curves show that \VPBoost
reduces the regression error much more aggressively than \GDBoost across the
full 30 boosting rounds, and the final gap on the test set is substantial. The
gap to the \XGBoost baseline, comprised of $18$ learners of a max-depth of
$5$, is even larger. Among the methods considered here, these results suggest
that variable projection with neural networks is particularly effective when
each weak learner must fit a richer multivariate response, even when the
featurizer is a small residual network rather than the MLPs used in the
earlier experiments.

\begin{figure}[!htp]
    \centering
    \def\resultsDir{data/cdr/ensemble_metrics_summary}
    \def\ymin{2e-2}
    \def\ymax{32}
    \def\xmin{0}
    \def\xmax{30}
    \def\ylabel{MSE}

\begin{tikzpicture}

\scriptsize 
\begin{groupplot}[group style={group size=2 by 1, 
		ylabels at=edge left, yticklabels at=edge left, 
		xlabels at=edge bottom, xticklabels at=edge bottom, 
		horizontal sep=0.025\linewidth}, 
scale only axis,
width=0.4\linewidth, height=0.2\linewidth, 		
xlabel={Boosting Iteration}, 
		xmin=\xmin, xmax=\xmax,
ylabel={\ylabel}, 
		ymin=\ymin, ymax=\ymax, 
		ymode=log, 
ticklabel style={/pgf/number format/fixed, /pgf/number format/precision=0}, 
grid=major,
		grid style={dashed, gray!30},  
legend to name=sharedlegend, 
		legend columns=6,
		execute at begin axis={
    			\addlegendimage{style_gb}  \addlegendentry{\GDBoost}
    			\addlegendimage{style_vp}   \addlegendentry{\VPBoost}
			\addlegendimage{style_vp_at_start}   \addlegendentry{VP@Start}
			\addlegendimage{style_vp_at_end}   \addlegendentry{VP@End}
			\addlegendimage{style_vp_at_start_and_end}   \addlegendentry{VP@Start+End}
			\addlegendimage{style_xgboost}   \addlegendentry{\XGBoost}
		}, 
		legend style={/tikz/every even column/.append style={column sep=0.25cm}}
		]

\nextgroupplot[title=Train]
	\def\dataType{train}
		\foreach \methodType in {gb, vp, vp_at_start, vp_at_end, vp_at_start_and_end}{
      		  	\edef\temp{
				\noexpand\targetplotB{\resultsDir/\dataType/optimal_\methodType_em_summary.csv}{style_\methodType}
			}\temp
    		}
		
		\pgfplotstableread[col sep=comma]{\resultsDir/\dataType/optimal_xgboost_em_summary.csv}\xgboosttable
		\pgfplotstablegetrowsof{\xgboosttable}
		\pgfmathtruncatemacro{\xgtrainlastrow}{\pgfplotsretval-1}
		\pgfplotstablegetelem{\xgtrainlastrow}{target_metric_mean}\of{\xgboosttable}
		\edef\xgtrainfinal{\pgfplotsretval}
		
		\addplot [style_xgboost, domain=-1:100] {\xgtrainfinal};

	\nextgroupplot[title=Validation]
	\def\dataType{val}
		\foreach \methodType in {gb, vp, vp_at_start, vp_at_end, vp_at_start_and_end}{
      		  	\edef\temp{
				\noexpand\targetplotB{\resultsDir/\dataType/optimal_\methodType_em_summary.csv}{style_\methodType}
				\noexpand
			}\temp
    		}
		
		\pgfplotstableread[col sep=comma]{\resultsDir/\dataType/optimal_xgboost_em_summary.csv}\xgboosttable
		\pgfplotstablegetrowsof{\xgboosttable}
		\pgfmathtruncatemacro{\xgtrainlastrow}{\pgfplotsretval-1}
		\pgfplotstablegetelem{\xgtrainlastrow}{target_metric_mean}\of{\xgboosttable}
		\edef\xgtrainfinal{\pgfplotsretval}
		
		\addplot [style_xgboost, domain=-1:100] {\xgtrainfinal};

\end{groupplot}

\node at ($(group c1r1.south west)!0.5!(group c2r1.south east)$) [anchor=north, yshift=-1cm, xshift=-0.6cm] {\pgfplotslegendfromname{sharedlegend}
};

\end{tikzpicture}

     \caption{
      Comparison of \VPBoost, \VPBoost variants, and \GDBoost on the CDR
      regression task across 30 boosting iterations. Each plot shows mean
      performance (solid lines) with standard deviation (shaded band) over 3
      independent trials. The \XGBoost baseline is shown as a horizontal dashed
      line at its final ensemble target metric.
    }
    \label{fig:05_05_highdim_cdr}
\end{figure}
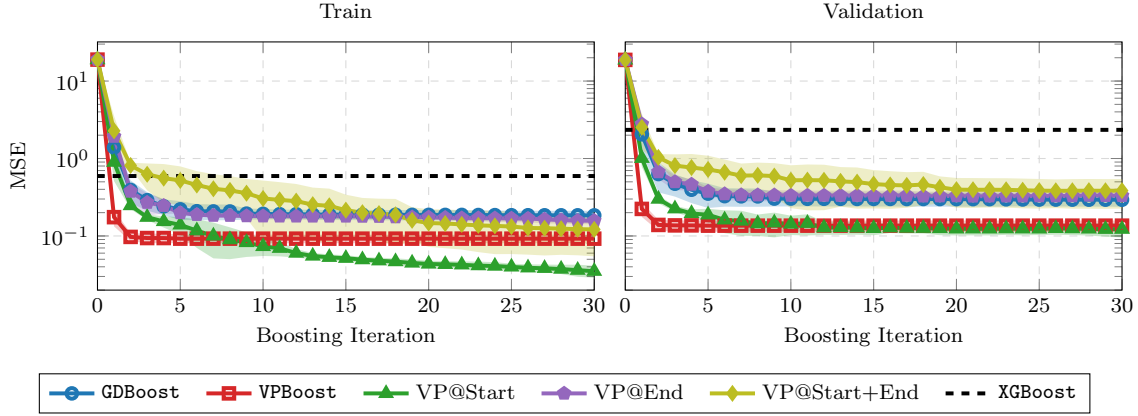

\subsubsection{Higgs Boson Classification}
\label{subsec:higgs}

The Higgs benchmark is a standard large-scale tabular classification task with
28 input features and a conventional held-out test set consisting of the final
500{,}000 examples \citep{whiteson_higgs_2014}. This setting is especially
favorable to \XGBoost-style methods, making it a useful stress test. Each
Higgs weak learner featurizer consists of $2$ hidden layers of width $16$ and
outputs a $16$-dimensional feature vector, followed by a final affine map to a
scalar logit. The capacity per weak learner is $\NTheta = 1008$ and
$\NW = 17$. We construct an ensemble of $M=25$ weak learners, resulting in an
ensemble capacity of $25{,}625$. This is paired with an \XGBoost model
comprised of 200 learners with maximum depth 7, for a capacity of $25{,}400$.

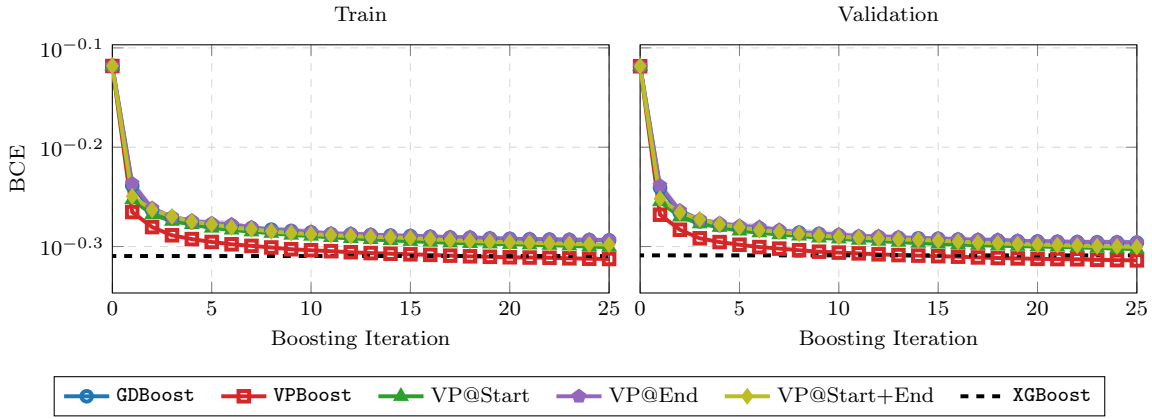
\begin{figure}[!htp]
    \centering
    \def\resultsDir{data/higgs/ensemble_metrics_summary}
    \def\ymin{4.5e-1}
    \def\ymax{8e-1}
    \def\xmin{0}
    \def\xmax{25}
    \def\ylabel{BCE}

\begin{tikzpicture}

\scriptsize 
\begin{groupplot}[group style={group size=2 by 1, 
		ylabels at=edge left, yticklabels at=edge left, 
		xlabels at=edge bottom, xticklabels at=edge bottom, 
		horizontal sep=0.025\linewidth}, 
scale only axis,
width=0.4\linewidth, height=0.2\linewidth, 		
xlabel={Boosting Iteration}, 
		xmin=\xmin, xmax=\xmax,
ylabel={\ylabel}, 
		ymin=\ymin, ymax=\ymax, 
		ymode=log, 
ticklabel style={/pgf/number format/fixed, /pgf/number format/precision=0}, 
grid=major,
		grid style={dashed, gray!30},  
legend to name=sharedlegend, 
		legend columns=6,
		execute at begin axis={
    			\addlegendimage{style_gb}  \addlegendentry{\GDBoost}
    			\addlegendimage{style_vp}   \addlegendentry{\VPBoost}
			\addlegendimage{style_vp_at_start}   \addlegendentry{VP@Start}
			\addlegendimage{style_vp_at_end}   \addlegendentry{VP@End}
			\addlegendimage{style_vp_at_start_and_end}   \addlegendentry{VP@Start+End}
			\addlegendimage{style_xgboost}   \addlegendentry{\XGBoost}
		}, 
		legend style={/tikz/every even column/.append style={column sep=0.25cm}}
		]

\nextgroupplot[title=Train]
	\def\dataType{train}
		\foreach \methodType in {gb, vp, vp_at_start, vp_at_end, vp_at_start_and_end}{
      		  	\edef\temp{
				\noexpand\targetplotB{\resultsDir/\dataType/optimal_\methodType_em_summary.csv}{style_\methodType}
			}\temp
    		}
		
		\pgfplotstableread[col sep=comma]{\resultsDir/\dataType/optimal_xgboost_em_summary.csv}\xgboosttable
		\pgfplotstablegetrowsof{\xgboosttable}
		\pgfmathtruncatemacro{\xgtrainlastrow}{\pgfplotsretval-1}
		\pgfplotstablegetelem{\xgtrainlastrow}{target_metric_mean}\of{\xgboosttable}
		\edef\xgtrainfinal{\pgfplotsretval}
		
		\addplot [style_xgboost, domain=-1:100] {\xgtrainfinal};

	\nextgroupplot[title=Validation]
	\def\dataType{val}
		\foreach \methodType in {gb, vp, vp_at_start, vp_at_end, vp_at_start_and_end}{
      		  	\edef\temp{
				\noexpand\targetplotB{\resultsDir/\dataType/optimal_\methodType_em_summary.csv}{style_\methodType}
				\noexpand
			}\temp
    		}
		
		\pgfplotstableread[col sep=comma]{\resultsDir/\dataType/optimal_xgboost_em_summary.csv}\xgboosttable
		\pgfplotstablegetrowsof{\xgboosttable}
		\pgfmathtruncatemacro{\xgtrainlastrow}{\pgfplotsretval-1}
		\pgfplotstablegetelem{\xgtrainlastrow}{target_metric_mean}\of{\xgboosttable}
		\edef\xgtrainfinal{\pgfplotsretval}
		
		\addplot [style_xgboost, domain=-1:100] {\xgtrainfinal};

\end{groupplot}

\node at ($(group c1r1.south west)!0.5!(group c2r1.south east)$) [anchor=north, yshift=-1cm, xshift=-0.6cm] {\pgfplotslegendfromname{sharedlegend}
};

\end{tikzpicture}

     \caption{
      Comparison of \VPBoost, \VPBoost variants, and \GDBoost on the Higgs
      classification task across 25 boosting iterations. Each plot shows mean
      performance (solid lines) with standard deviation (shaded band) over 3
      independent trials. The \XGBoost baseline is shown as a horizontal dashed
      line at its final ensemble target metric.
    }
    \label{fig:05_05_highdim_higgs}
\end{figure}

As shown in \Cref{fig:05_05_highdim_higgs}, \VPBoost continues to improve
throughout the 25 boosting iterations and attains the strongest final test
performance among the methods compared here. The final result improves upon
both \GDBoost and the best alternative VP variant. The comparison is also
favorable relative to our constrained \XGBoost baseline.

\begin{table}[t]
\centering
\footnotesize

\def\resultsDirA{data/cdr/ensemble_metrics_summary_final/test}
    
	\pgfplotstableread[col sep=comma]{\resultsDirA/optimal_gb_em_summary.csv}\loadedtabletestA

\pgfplotsforeachungrouped \method in {vp_at_start, vp_at_end, vp_at_start_and_end, vp, xgboost} {\pgfplotstableread[col sep=comma]{\resultsDirA/optimal_\method_em_summary.csv}\loadedtabletmpA

        \pgfplotstablevertcat{\loadedtabletestA}{\loadedtabletmpA}
        \pgfplotstableclear{\loadedtabletmpA}
    }

\createcolmeanstd{\loadedtabletestA}{target_metric}{targetmeanstdA}
     \createcolmean{\loadedtabletestA}{skl_r2}{r2meanstdA}
    
\def\resultsDirB{data/mnist/ensemble_metrics_summary/test}
    
	\pgfplotstableread[col sep=comma]{\resultsDirB/optimal_gb_em_summary.csv}\loadedtabletestB

\pgfplotsforeachungrouped \method in {vp_at_start, vp_at_end, vp_at_start_and_end, vp, xgboost} {\pgfplotstableread[col sep=comma]{\resultsDirB/optimal_\method_em_summary.csv}\loadedtabletmpB

        \pgfplotstablevertcat{\loadedtabletestB}{\loadedtabletmpB}
        \pgfplotstableclear{\loadedtabletmpB}
    }

\createcolmeanstd{\loadedtabletestB}{target_metric}{targetmeanstdB}
    \pgfplotstablecreatecol[copy column from table={\loadedtabletestB}{targetmeanstdB}]{targetmeanstdB}{\loadedtabletestA}

 	\createcolmean{\loadedtabletestB}{skl_auc}{aucB}
        \pgfplotstablecreatecol[copy column from table={\loadedtabletestB}{aucB}]{aucB}{\loadedtabletestA}

 \createcolmean{\loadedtabletestB}{skl_accuracy}{accB}
        \pgfplotstablecreatecol[copy column from table={\loadedtabletestB}{accB}]{accB}{\loadedtabletestA}

\def\resultsDirC{data/higgs/ensemble_metrics_summary/test}
    
	\pgfplotstableread[col sep=comma]{\resultsDirC/optimal_gb_em_summary.csv}\loadedtabletestC

\pgfplotsforeachungrouped \method in {vp_at_start, vp_at_end, vp_at_start_and_end, vp, xgboost} {\pgfplotstableread[col sep=comma]{\resultsDirC/optimal_\method_em_summary.csv}\loadedtabletmpC

        \pgfplotstablevertcat{\loadedtabletestC}{\loadedtabletmpC}
        \pgfplotstableclear{\loadedtabletmpC}
    }

\createcolmeanstd{\loadedtabletestC}{target_metric}{targetmeanstdC}
     \pgfplotstablecreatecol[copy column from table={\loadedtabletestC}{targetmeanstdC}]{targetmeanstdC}{\loadedtabletestA}

    \createcolmean{\loadedtabletestC}{skl_auc}{aucC}
        \pgfplotstablecreatecol[copy column from table={\loadedtabletestC}{aucC}]{aucC}{\loadedtabletestA}

	 \createcolmean{\loadedtabletestC}{skl_accuracy}{accC}
        \pgfplotstablecreatecol[copy column from table={\loadedtabletestC}{accC}]{accC}{\loadedtabletestA}

\pgfkeys{
    /pgf/fpu = true,
    /pgf/number format/.cd,
    precision=4,
    fixed,
    fixed zerofill,
1000 sep={.}
}

\pgfplotsset{topentrystyle/.style={/pgfplots/table/@cell content/.add={\cellcolor{lightgray!50}\boldmath}{},}}
	
\def\h{0.5}
    \pgfplotstabletypeset[
    	test_metrics_table_style,
columns={methodname,  targetmeanstdB, aucB,  targetmeanstdA, r2meanstdA, targetmeanstdC, aucC},
columns/targetmeanstdA/.style={column type={@{\hspace{\h cm}}c}, string type, column name=\multicolumn{1}{c}{MSE ($\downarrow$)}},
        columns/r2meanstdA/.style={column type=c, string type, column name=\multicolumn{1}{c}{$R^2$ ($\uparrow$)}},
        columns/targetmeanstdB/.style={column type={@{\hspace{\h cm}}c}, string type, column name=\multicolumn{1}{c}{MCE ($\downarrow$)}},
        columns/aucB/.style={column type=c, string type, column name=\multicolumn{1}{c}{AUC ($\uparrow$)}},
        columns/accB/.style={column type=c, string type, column name=\multicolumn{1}{c}{Acc. ($\uparrow$)}},
        columns/targetmeanstdC/.style={column type={@{\hspace{\h cm}}c}, string type, column name=\multicolumn{1}{c}{MCE ($\downarrow$)}},
        columns/aucC/.style={column type=c, string type, column name=\multicolumn{1}{c}{AUC ($\uparrow$)}},
         columns/accC/.style={column type=c, string type, column name=\multicolumn{1}{c}{Acc. ($\uparrow$)}},
every row 4 column 1/.style={
            postproc cell content/.append style={/pgfplots/table/@cell content/.add={\cellcolor{lightgray!50}\boldmath}{}}},
         every row 4 column 2/.style={
            postproc cell content/.append style={/pgfplots/table/@cell content/.add={\cellcolor{lightgray!50}\boldmath}{}}},
     	  every row 4 column 3/.style={
            postproc cell content/.append style={/pgfplots/table/@cell content/.add={\cellcolor{lightgray!50}\boldmath}{}}},
          every row 4 column 4/.style={
            postproc cell content/.append style={/pgfplots/table/@cell content/.add={\cellcolor{lightgray!50}\boldmath}{}}},
every row 4 column 5/.style={
            postproc cell content/.append style={/pgfplots/table/@cell content/.add={\cellcolor{lightgray!50}\boldmath}{}}},
every row 4 column 6/.style={
            postproc cell content/.append style={/pgfplots/table/@cell content/.add={\cellcolor{lightgray!50}\boldmath}{}}},
         every row 4 column 7/.style={
            postproc cell content/.append style={/pgfplots/table/@cell content/.add={\cellcolor{lightgray!50}\boldmath}{}}},
every head row/.style={
        before row={\toprule
& \multicolumn{2}{c}{\bfseries MNIST } & \multicolumn{2}{c}{\bfseries CDR} & \multicolumn{2}{c}{\bfseries Higgs}  \\
            \cmidrule(lr){2-3}  \cmidrule(lr){4-5}  \cmidrule(lr){6-7}
        },
        after row=\midrule },
    ]{\loadedtabletestA}

\caption{
    Test metrics across three real-world datasets.
    Results show the mean over $10$ random featurizers initializations.
    Standard deviations are shown for the loss metrics and suppressed for AUC and $R^2$ values for presentation purposes. 
    The top metrics per column for boosting and full training are bold and highlighted in gray. 
    Over all boosting methods, \VPBoost achieves the lowest loss per task.
}
\label{tab:toy_experiment_metrics}

\end{table}

These experiments support the claim that \VPBoost can perform strongly on
high-dimensional regression and classification tasks. We additionally compare
against historical results reported for the same datasets. On the Higgs
benchmark, the original \XGBoost paper \citep{chen_xgboost_2016} reports a
test AUC that falls below what \VPBoost achieves here. On the CDR benchmark,
our test MSE improves upon the results reported in
\citet{newman_train_2021}. We stress, however, that these are informal
comparisons: the current \VPBoost implementation is substantially more
computationally intensive than the highly optimized \XGBoost software, and the
experimental conditions differ.
 \subsection{Scaling}
\label{subsec:scaling}

In this section we investigate the computational cost-quality tradeoff of
\VPBoost relative to \GDBoost.
As discussed earlier, the primary expense of \VPBoost is the reduced linear solve
performed during weak learner training. 
For dense multi-output regression, this solve scales cubically in $\NW = \NFeat
\NTarget$, whereas standard gradient-based weak learner training scales
linearly with the parameter count. 
In the boosting setting, however, the weak learners are intentionally small, so
$\NFeat$ is largely controlled by design. 
This leaves the output dimensionality $\NTarget$ as the primary practical
scaling variable. 
To study this effect in isolation, we use \texttt{scikit-learn}'s
\texttt{make\_regression} generator with a fixed weak learner architecture,
vary $\NTarget \in \{1,2,4,8,16,32,64,128\}$, and train both \VPBoost and
\GDBoost for a fixed budget of eight learners over five random seeds.
The results are summarized in \Cref{fig:05_06_scaling_overview}.

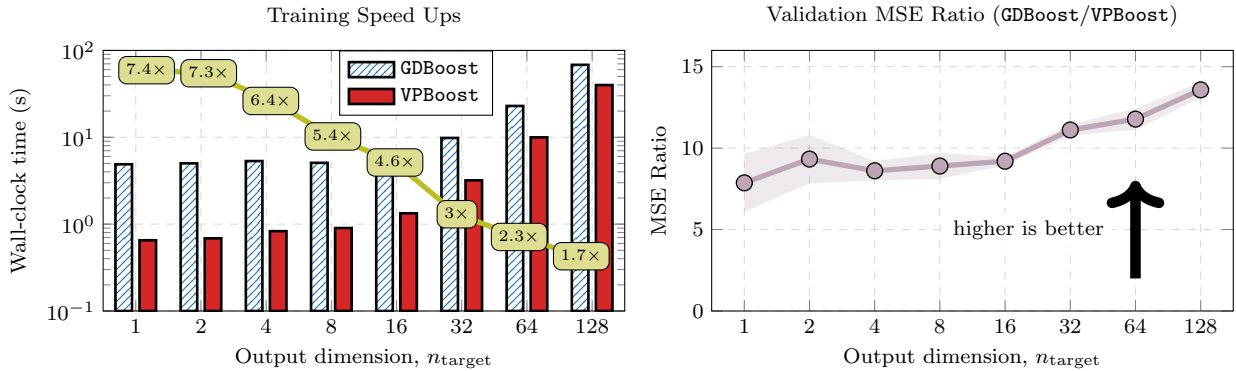
\begin{figure}[!htp]
    \centering
\scriptsize
    \begin{tikzpicture}
        \pgfplotstableread[col sep=comma]{data/scaling/time_to_gb_best_summary.csv}\scalingracevthree
        \pgfplotstableread[col sep=comma]{data/scaling/final_tradeoff.csv}\scalingtradeoffvthree
        \def\colorSpeedup{matplotlib8}

        \begin{groupplot}[
            group style={group size=2 by 1, horizontal sep=0.07\linewidth},
            scale only axis,
            width=0.42\linewidth, height=0.21\linewidth,
            ymin=0, ymax=64,
            xmin=0.707, xmax=181,
            grid=major,
            grid style={dashed, gray!30},
            xlabel={Output dimension, $\NTarget$},
            xtick={1, 2, 4, 8, 16, 32, 64, 128},
            xticklabels={1, 2, 4, 8, 16, 32, 64, 128},
            xmode=log,
            log basis x=2,
        ]

        \nextgroupplot[
            title={Training Speed Ups},
            ylabel={Wall-clock time (s)},
            ymode=log,
            log origin=infty,
            ymin=1e-1, ymax=1e2,
legend style={at={(rel axis cs: 0.6, 1)}, anchor=north},
            legend columns=1
        ]

        \addplot[ybar, bar width=0.22cm, bar shift=-0.16cm, fill=colorGB, area legend, pattern=north east lines, pattern color=colorGB, line width=1pt]
            table[x=output_dim, y=gb_time_to_target_mean] {\scalingracevthree};
        \addlegendentry{\GDBoost}

        \addplot[ybar, bar width=0.22cm, bar shift=+0.16cm, fill=colorVP, area legend, line width=1pt]
            table[x=output_dim, y=vp_time_to_target_mean] {\scalingracevthree};
        \addlegendentry{\VPBoost}

        \nextgroupplot[
            title={Validation MSE Ratio (\GDBoost/\VPBoost)},
            ylabel={MSE Ratio},
            ymin=0, ymax=16
        ]

        \addplot [color=colorGB!50!colorVP!50, mark=*, line width=2pt, mark options={draw=black, line width=0.5pt, mark size=3pt}] table[x=output_dim, y=val_mse_ratio_gb_over_vp_mean, col sep=comma] {\scalingtradeoffvthree};
        \addplot [name path=upper, draw=none, no marks, forget plot]
            table[x=output_dim, y expr=\thisrow{val_mse_ratio_gb_over_vp_mean}+\thisrow{val_mse_ratio_gb_over_vp_stderr}, col sep=comma] {\scalingtradeoffvthree};
        \addplot [name path=lower, draw=none, no marks, forget plot]
            table[x=output_dim, y expr=\thisrow{val_mse_ratio_gb_over_vp_mean}-\thisrow{val_mse_ratio_gb_over_vp_stderr}, col sep=comma] {\scalingtradeoffvthree};
        \addplot [colorGB!50!colorVP!50, opacity=0.25, forget plot] fill between[of=upper and lower];

        \draw[->, line width=4pt] (axis cs:64,2) -- node[midway, left=0.25cm, anchor=east, fill=white] {higher is better} (axis cs:64,8);

        \end{groupplot}

        \begin{groupplot}[
            group style={group size=2 by 1, horizontal sep=0.12\linewidth},
            axis y line*=right,
            axis x line*=none,
            ylabel={\color{\colorSpeedup}Speed Up},
            y axis line style={\colorSpeedup},
            yticklabel style={\colorSpeedup},
            ytick style={\colorSpeedup},
            hide axis,
            xtick=\empty,
            xticklabels=\empty,
            xtick={1, 2, 4, 8, 16, 32, 64, 128},
            xmode=log,
            log basis x=2,
            ymin=0, ymax=8,
            scale only axis,
            width=0.42\linewidth, height=0.21\linewidth,
        ]
            \nextgroupplot
            \addplot[line width=2pt, color=\colorSpeedup] table[x=output_dim, y expr=1/\thisrow{time_ratio_vp_over_gb_mean}] {\scalingracevthree};

            \addplot[draw=none,
                visualization depends on={1/\thisrow{time_ratio_vp_over_gb_mean} \as \speedup},
                nodes near coords={\pgfmathprintnumber[fixed,precision=1]{\speedup}$\times$},
                nodes near coords style={anchor=center, draw, fill=\colorSpeedup!50, rounded corners=3pt, font=\tiny\bfseries}]
                table[x=output_dim, y expr=1/\thisrow{time_ratio_vp_over_gb_mean}] {\scalingracevthree};
        \end{groupplot}

    \end{tikzpicture}

    \caption{
        Scaling behavior as the output dimension $\NTarget$ increases.
        Left: the time needed to match \GDBoost's final validation loss.
        Across all tested output dimensions, \VPBoost reaches that target sooner, although the speedup narrows as $\NTarget$ grows.
        Right: the final validation-MSE ratio \GDBoost/\VPBoost.
        Higher is better for \VPBoost, and the advantage grows with $\NTarget$ over the tested range.
    }
    \label{fig:05_06_scaling_overview}
\end{figure}

The first conclusion is that \VPBoost does indeed cost more to train when the
number of learners is held fixed. 
Across the eight-learner runs, the total runtime ratio of \VPBoost/\GDBoost
grows steadily over the tested range, reflecting the expected dependence on
$\NTarget$. 
Yet this additional cost buys stronger weak learners. 
In the right panel of \Cref{fig:05_06_scaling_overview}, the final validation
MSE ratio \GDBoost/\VPBoost is already large at small output dimension and
improves as $\NTarget$ increases. 
Hence, for a fixed model size, using \VPBoost may be worth the computational
expense.

A more relevant comparison for boosting is often not the cost of training the
same number of learners, but using the cost of reaching the same loss threshold.
The left panel in \Cref{fig:05_06_scaling_overview} therefore measures the
wall-clock time needed for each method to attain the final validation loss
achieved by the eight-learner \GDBoost ensemble. 
By this metric, \VPBoost is faster at every tested output dimension, although
the margin narrows as $\NTarget$ increases. 
This is because each \VPBoost weak learner improves far further on the loss,
so the method reaches the \GDBoost target in far fewer effective boosting
steps. 
This is consistent with the results across nearly all experiments.

At the same time, these results should not be over-interpreted as evidence that
\VPBoost scales effortlessly to arbitrarily large output spaces. 
The favorable trend shown here clearly does not persist as $\NTarget$
increases, even accounting for the fact that the current \VPBoost implementation is far
less optimized than standard gradient-descent routines.
Nevertheless, the experiment identifies a useful intermediate regime: for moderate
multi-output problems, the higher per-learner cost of variable projection can
still be worthwhile because it buys a much stronger reduction in loss per
learner. 
Understanding how to preserve that advantage at much larger output
dimensionalities, such as image-level prediction or classification, remains an
important direction for future work.

\section{Conclusion}
\label{sec:conclusion}

This work introduced \VPBoost, a gradient boosting algorithm for separable
smooth approximators that combines an \XGBoost-style second-order Taylor
approximation with variable projection (Section~\ref{sec:vpboost}).
By exploiting the separable structure and analytically eliminating the final linear
 weights, \VPBoost becomes a functional trust-region method in
which progress in parameter space translates directly to progress in function
space (Algorithm~\ref{alg:trust_region_skeleton_vpboost}). 
We proved that \VPBoost weak learners are automatically descent directions and, under mild subspace regularity conditions, build ensembles that converge to a stationary point of the target loss functional (Section~\ref{sec:convergence_analysis}). 
Across benchmarks spanning synthetic problems,
image recognition, large-scale tabular data, and high-dimensional scientific
machine learning, \VPBoost consistently outperforms gradient descent-based
boosting and is competitive with \XGBoost,  particularly with very weak
learners or when fitting higher-dimensional or smoother target responses (Section~\ref{sec:numerical_experiments}).
The theory and experiments together suggest that fusing VarPro and boosting is more than a training trick: it provides a
principled and computationally advantageous framework to extend boosting beyond classical tree-based settings.

Several future directions follow naturally from this work.
Scaling \VPBoost to larger output dimensions remains an important
practical challenge.  
The reduced linear solve becomes the dominant cost in dense multi-output
settings. 
Extending \VPBoost will likely require additional structure on the linear layer 
and the use of sparse iterative solvers. 
To avoid overfitting with the optimal linear weights, VarPro is a data-hungry method designed for deterministic approximations of expected loss functions. 
Extending to stochastic optimization~\citep{newman_slimtrain---stochastic_2022} and randomized sketching~\citep{antil2025randomizedmatrixsketchingneural} could enable \VPBoost to scale in data. 
An important next step in the theory is to understand generalization for boosted separable
neural networks: which architectural or algorithmic properties control test
error, how ensemble depth interacts with weak learner expressivity, and when
the optimization advantages of \VPBoost translate into reliable out-of-sample
performance. 
 
\acks{

E. Newman and A. Chowdhary acknowledge support from the United States Air Force
Office of Scientific Research under FA9550-26-1-B049 (program manager Dr.
Fariba Fahroo). 
The work by E. Newman was also partially supported by the National Science
Foundation (NSF) under grant DMS-2309751. 
The numerical examples were conducted using computational resources and
services at the Center for Computation and Visualization of Brown University. 

We disclose our use of AI tools in this work.
We used generative search tools, such as Gemini, and ChatGPT’s DeepResearch to
discover recent literature relevant to this work, to find appropriate
terminology and ensure accessibility of this work across domains, and to find
nuanced commands to generate clear figures and tables.
We used Claude Code with Sonnet 4 and Opus 4 and Codex CLI with GPT-5 as a
coding assistant, primarily to build SLURM orchestration for efficient
hyperparameter tuning, and to help prototype, debug, and document the code. 
The code is available upon request and will be released publicly after review
of the manuscript is complete. 
All parts of the paper were written and edited without AI assistance. 
We take full responsibility for the accuracy and integrity of all content in
this paper. 

}

\appendix

\section{Function Spaces and Functional Derivatives}
\label{app:functional_derivatives}

This appendix collects the functional-analytic foundations that underpin the
main text.
\Cref{app:bochner} establishes $\FctnClass$ as a Bochner space, derives its
empirical counterpart under the SAA measure $\mu_N$, and establishes the
operator norm and submultiplicativity used in the convergence analysis.
\Cref{app:frechet} defines G\^{a}teaux and Fr\'{e}chet derivatives on Hilbert
spaces, identifies the functional gradient as a Riesz representer, and derives
the pointwise formula for the gradient of the SAA loss.

\subsection{Bochner Spaces, Empirical Reduction, and Operator Norm}
\label{app:bochner}
We recall the definition of a Bochner space and then specialise to the
hypothesis class $\FctnClass$ and the featurizer norm used in this paper.
Standard references are \citet[Appendix~E]{evans_partial_2010} for an
applied-mathematics introduction and \citet[Chapter~1]{hytonen_analysis_2016}
for the general theory.

\begin{definition}[Bochner space
{\citep[Section~1.2]{hytonen_analysis_2016},\citep{evans_partial_2010}}]
\label{def:bochner}
Let $\mu$ be a $\sigma$-finite measure on $\InputSpace$ and $V$ a Banach space.
The \emph{Bochner space} $L^2(\InputSpace, \mu; V)$ consists of strongly
measurable functions $f: \InputSpace \to V$ for which $\bfx \mapsto
\|f(\bfx)\|_V$ is square-$\mu$-integrable, with norm
\begin{equation*}
  \|f\|_{L^2(\InputSpace,\mu;V)}
  \coloneqq
  \left(\int_{\InputSpace} \|u(\bfx)\|_V^2 \, d\mu(\bfx)\right)^{1/2}.
\end{equation*}
When $V$ is a Hilbert space, $L^2(\InputSpace, \mu; V)$ is itself a Hilbert
space with inner product
$\langle u, v \rangle_{L^2(\InputSpace,\mu;V)}
= \int_{\InputSpace} \langle u(\bfx), v(\bfx) \rangle_V \, d\mu(\bfx)$.
\end{definition}
The hypothesis class $\FctnClass = L^2(\InputSpace, \mu; \Rbb^{\NTarget})$ is
the special case $V = \Rbb^{\NTarget}$ with the Euclidean inner product.

\begin{proposition}[Product-space equivalence
{\citep[Section~1.2]{hytonen_analysis_2016}}]
\label{prop:bochner_product}
For vector space $V = \Rbb^d$,
\begin{equation*}
  L^2(\InputSpace, \mu; \Rbb^d) \;\cong\; \bigl[L^2(\InputSpace, \mu)\bigr]^d
\end{equation*}
isometrically, via $f \mapsto (f_1, \ldots, f_d)$ where $f_j(\bfx) =
[f(\bfx)]_j$.
\end{proposition}

This equivalence is why \citet{nitanda_functional_2018,
atsushi_nitanda_functional_2020} can write $L^d_2(\mu)$ for the same space.

\paragraph{Empirical reduction.}
When $\mu$ is replaced by the empirical measure $\mu_N =
\frac{1}{N}\sum_{i=1}^N \delta_{\bfx_i}$, the Bochner space collapses to a
finite-dimensional space \citep[Section~4.7]{bogachev_measure_2007}.
Specifically, $L^2(\InputSpace, \mu_N; \Rbb^{\NTarget}) \cong
\Rbb^{N\times\NTarget}$ via the evaluation isomorphism
$f \mapsto (f(\bfx_1), \ldots, f(\bfx_N))$, with inner product and norm
\begin{align*}\langle f, g\rangle_{\SAA\FctnClass}
  = \frac{1}{N}\sum_{i=1}^N f(\bfx_i)^\top g(\bfx_i),
  \qquad
  \|f\|_{\SAA\FctnClass}^2
  = \frac{1}{N}\sum_{i=1}^N \|f(\bfx_i)\|_2^2.
\end{align*}

\paragraph{Operator norm and submultiplicativity.}
The featurizer $\FeaturizerMat_{\bftheta}: \InputSpace \to
\Rbb^{\NTarget\times\NW}$ is a matrix-valued function in $L^2(\InputSpace,
\mu; \Rbb^{\NTarget\times\NW})$.
For the convergence analysis, we equip this space with the \emph{induced
operator norm}
\begin{align*}\|\FeaturizerMat_{\bftheta}\|
  \coloneqq
  \sup_{\|\bfw\|_2 = 1}
  \|\FeaturizerMat_{\bftheta}(\cdot)\bfw\|
  = \sup_{\|\bfw\|_2=1}
  \left(
    \int_{\InputSpace} \|\FeaturizerMat_{\bftheta}(\bfx)\bfw\|_2^2\, d\mu(\bfx)
  \right)^{1/2},
\end{align*}
which differs from the Bochner norm (the latter integrates the Frobenius norm
of $\FeaturizerMat_{\bftheta}(\bfx)$ rather than the operator action on
$\bfw$).
In the empirical setting, the operator norm becomes
\begin{equation*}\label{eq:empirical_operator_norm}
  \|\FeaturizerMat_{\bftheta}\|_N
  = \sup_{\|\bfw\|_2 = 1}
  \left(\frac{1}{N}\sum_{i=1}^N \|\FeaturizerMat_{\bftheta}(\bfx_i)\bfw\|_2^2\right)^{1/2}
  = \sigma_{\max}(\widetilde{\FeaturizerMat}_{\bftheta}),
\end{equation*}
the largest singular value of the scaled stacked matrix
$\widetilde{\FeaturizerMat}_{\bftheta} \coloneqq
[\FeaturizerMat_{\bftheta}(\bfx_1)^\top, \ldots,
\FeaturizerMat_{\bftheta}(\bfx_N)^\top]^\top / \sqrt{N} \in \Rbb^{N\NTarget
\times \NW}$, and Lemma~\ref{lem:operator_norm} holds in this setting as
well.

\subsection{Differentiation of Functionals in Hilbert Spaces}
\label{app:frechet}

Following \citet[Section~2.2]{antil_frontiers_2018}, we summarize how
directional, G\^{a}teaux, and Fr\'{e}chet derivatives relate for functionals
defined on a Hilbert space.
Let $\ObjFunctional: \FctnClass \to \Rbb$ denote a functional; $\LossFunctional$
or $\ObjFunctional$ typically play this role in our setting.

\begin{definition}[Directional derivative]
For $f, \varphi \in \FctnClass$, the \emph{directional derivative} of
$\ObjFunctional$ at $f$ in direction $\varphi$ is
\begin{equation*}
  \ObjFunctional^\prime[f;\, \varphi]
  \coloneqq
  \lim_{t \to 0}
  \frac{\ObjFunctional[f + t\varphi] - \ObjFunctional[f]}{t},
\end{equation*}
provided the limit exists.
\end{definition}

\begin{definition}[G\^{a}teaux derivative]
$\ObjFunctional$ is \emph{G\^{a}teaux differentiable} at $f$ if
$\ObjFunctional^\prime[f;\, \varphi]$ exists for every $\varphi \in \FctnClass$
and the map $\varphi \mapsto \ObjFunctional^\prime[f;\, \varphi]$ is linear and
continuous.
The associated bounded linear functional $\ObjFunctional^\prime[f] \in
\FctnClass^\star$ is defined by $\ObjFunctional^\prime[f](\varphi) =
\ObjFunctional^\prime[f;\, \varphi]$.
\end{definition}

\begin{definition}[Fr\'{e}chet derivative]\label{def:frechet_derivative}
$\ObjFunctional$ is \emph{Fr\'{e}chet differentiable} at $f$ if there exists a
bounded linear functional $\ObjFunctional^\prime[f] \in \FctnClass^\star$ such
that
\begin{equation*}
  \ObjFunctional[f + h]
  = \ObjFunctional[f] + \ObjFunctional^\prime[f](h) + o(\|h\|)
  \qquad \text{as } \|h\| \to 0.
\end{equation*}
Fr\'{e}chet differentiability implies G\^{a}teaux differentiability.
\end{definition}

\begin{remark}[Riesz representer / functional gradient]
When $\FctnClass$ is a Hilbert space, every bounded linear functional admits a
unique Riesz representer.
A Fr\'{e}chet derivative $\ObjFunctional^\prime[f] \in \FctnClass^\star$ is
therefore identified with the unique element $\nabla\ObjFunctional[f] \in
\FctnClass$ satisfying
\begin{equation*}
  \ObjFunctional^\prime[f](\varphi)
  = \langle \nabla\ObjFunctional[f],\, \varphi \rangle
  \qquad \forall\, \varphi \in \FctnClass.
\end{equation*}
We refer to $\nabla\ObjFunctional[f]$ as the \emph{functional gradient} at $f$.
\end{remark}

\begin{definition}[Second Fr\'{e}chet derivative]\label{def:frechet_second_derivative}
$\ObjFunctional$ is \emph{twice Fr\'{e}chet differentiable} at $f$ if the
map $f \mapsto \ObjFunctional^\prime[f] \in \FctnClass^\star$ is itself
Fr\'{e}chet differentiable at $f$.
The second derivative is a bounded linear operator
$\ObjFunctional^{\prime\prime}[f]: \FctnClass \to \FctnClass^\star$ satisfying
\begin{equation*}
  \ObjFunctional^\prime[f + h]
  = \ObjFunctional^\prime[f]
  + \ObjFunctional^{\prime\prime}[f](h)
  + o(\|h\|)
  \qquad \text{as } \|h\| \to 0,
\end{equation*}
where the remainder is measured in $\FctnClass^\star$.
For fixed $h \in \FctnClass$, the element $\ObjFunctional^{\prime\prime}[f](h)
\in \FctnClass^\star$ is evaluated at $\varphi \in \FctnClass$ to give the
scalar $\ObjFunctional^{\prime\prime}[f](h)(\varphi) \in \Rbb$; we write this
bilinear evaluation as $\ObjFunctional^{\prime\prime}[f](h, \varphi)$ for
brevity.
Applying the Riesz map to each $\ObjFunctional^{\prime\prime}[f](h) \in
\FctnClass^\star$ yields the \emph{Hessian operator}
$\nabla^2\ObjFunctional[f]: \FctnClass \to \FctnClass$, the unique bounded
linear operator satisfying
\begin{equation*}
  \ObjFunctional^{\prime\prime}[f](h, \varphi)
  = \langle \nabla^2\ObjFunctional[f]\, h,\, \varphi \rangle
  \qquad \forall\, h, \varphi \in \FctnClass.
\end{equation*}
When $\ObjFunctional$ is twice continuously Fr\'{e}chet differentiable,
$\ObjFunctional^{\prime\prime}[f]$ is symmetric and $\nabla^2\ObjFunctional[f]$
is therefore self-adjoint.
\end{definition}

 \section{Proofs for Variable Projection vs. Gradient Descent}
\label{app:vp_vs_gd_proofs}

This appendix contains full statements and proofs for the results presented in Section~\ref{sec:varpro}. 
These are established results \citep{sjoberg_separable_1997} and the presentation here is adapted from more general proofs for non-smooth objective functions from~\cite[Lemma 2.1 and Corollary 2.3]{van_leeuwen_variable_2021}. 

\reducedgradientequality*
\begin{proof}
By the Implicit Function Theorem, $\WOpt(\bftheta)$ is continuously differentiable.
Hence, we can apply the chain rule and compute the gradient of the reduced objective function via
\begin{align*}\nabla \Reduced{\SAA\ObjFctn}(\bftheta) = \nabla_{\bftheta} \SAA\ObjFctn(\WOpt(\bftheta), \bftheta)
= \nabla_{\bftheta}\WOpt(\bftheta)\nabla_{\bfw} \SAA\ObjFctn(\WOpt(\bftheta), \bftheta) + \left.\nabla_{\bftheta} J_N(\bfw, \bftheta)\right|_{\bfw =\WOpt(\bftheta)}.
\end{align*}
By solving~\eqref{eq:w_opt}, $\WOpt(\bftheta)$ satisfies the first-order optimality condition $\nabla_{\bfw} \SAA\ObjFctn(\WOpt(\bftheta), \bftheta) = \bf0$, thereby eliminating the first term.
The Implicit Function Theorem further guarantees $\WOpt(\bftheta)$ is unique, and hence is the only viable choice of $\bfw$ to achieve the desired equality of gradients.
\end{proof}

\begin{lemma}[Lipschitz Continuity of $\WOpt(\bftheta)$]\label{lem:smoothness_w_opt}
Under Assumptions~\ref{assump:vp_smoothness} and~\ref{assump:vp_convexity},
the optimal linear weight function $\WOpt: \ThetaSpace \to \WSpace$ defined
in~\eqref{eq:w_opt} is $K / \lambda_w$-smooth.
\end{lemma}

\begin{proof}
The strong convexity assumption ensures $\nabla_{\bfw} J_N(\bfw, \bftheta) \succeq \lambda_w \bfI_{\NW}$ for all $(\bfw,\bftheta)\in \WSpace \times \ThetaSpace$.
As shown in~\citep[Theorem 2.1.11]{nesterov_introductory_2004}, this condition is equivalent to
\begin{align*}
\lambda_w \|\bfw - \bfw'\|^2 \le \left\langle \nabla_{\bfw} \SAA\ObjFctn(\bfw, \bftheta_1) - \nabla_{\bfw} \SAA\ObjFctn(\bfw', \bftheta_2),  \bfw - \bfw' \right\rangle.
\end{align*}
for all $\bfw, \bfw'\in \WSpace$ and any $\bftheta_1, \bftheta_2\in \ThetaSpace$.
Set $\bfw \equiv \WOpt(\bftheta)$ and $\bfw' \equiv \WOpt(\bftheta')$ for some $\bftheta, \bftheta'\in \ThetaSpace$ and let $\bftheta_1 = \bftheta_2 = \widetilde{\bftheta}$ for some $\widetilde{\bftheta}\in \ThetaSpace$.
Substituting, we get
\begin{align*}
\begin{split}
\lambda_w \|\WOpt(\bftheta) - \WOpt(\bftheta')\|^2
\le \left\langle \nabla_{\bfw} \SAA\ObjFctn(\WOpt(\bftheta), \widetilde{\bftheta}) - \nabla_{\bfw} \SAA\ObjFctn(\WOpt(\bftheta'), \widetilde{\bftheta}),  \WOpt(\bftheta) - \WOpt(\bftheta') \right\rangle.
\end{split}
\end{align*}
Using first-order optimality condition associated with $\WOpt$, we strategically add zero vectors to the above inequality via
\begin{align*}
\begin{split}
\lambda_w \|\WOpt(\bftheta) - \WOpt(\bftheta')\|^2
&\le \left\langle \nabla_{\bfw} \SAA\ObjFctn(\WOpt(\bftheta), \widetilde{\bftheta}) - \nabla_{\bfw} \SAA\ObjFctn(\WOpt(\bftheta), \bftheta),  \WOpt(\bftheta) - \WOpt(\bftheta') \right\rangle\\
&\quad + \left\langle \nabla_{\bfw} \SAA\ObjFctn(\WOpt(\bftheta'), \bftheta') - \nabla_{\bfw} \SAA\ObjFctn(\WOpt(\bftheta'), \widetilde{\bftheta}),  \WOpt(\bftheta) - \WOpt(\bftheta') \right\rangle.
\end{split}
\end{align*}
Setting $\widetilde{\bftheta} \equiv \bftheta'$ and applying the Cauchy-Schwarz inequality on the right, we get
\begin{align*}
\lambda_w \|\WOpt(\bftheta) - \WOpt(\bftheta')\|^2 \le \|  \nabla_{\bfw} \SAA\ObjFctn(\WOpt(\bftheta), \bftheta') - \nabla_{\bfw} \SAA\ObjFctn(\WOpt(\bftheta), \bftheta)\|\|\WOpt(\bftheta) - \WOpt(\bftheta')\|.
\end{align*}
Applying Assumption~\ref{assump:vp_smoothness} to the right and simplifying leaves us with
\begin{align*}
\lambda_w \|\WOpt(\bftheta) - \WOpt(\bftheta')\| \le K\|\bftheta - \bftheta'\|.
\end{align*}
Dividing both sides by $\lambda_w$ yields the expected smoothness condition of $\WOpt$.
\end{proof}

\begin{theorem}[Smoothness of Reduced Objective]\label{thm:reduced_objective_smoothness}
Under Assumptions~\ref{assump:vp_smoothness} and~\ref{assump:vp_convexity} on the full objective function, $\SAA\ObjFctn$,
the reduced objective function, $\Reduced{\SAA\ObjFctn}$, is $\Reduced{K}$-smooth where $\Reduced{K} = K \sqrt{K^2 / \lambda_w^2 + 1}$.
\end{theorem}

\begin{proof}
Let $\bftheta, \bftheta' \in \ThetaSpace$ be arbitrary.
By direct computation, we obtain the inequality
\begin{align*}\begin{split}
\left\|\nabla \Reduced{\SAA\ObjFctn}(\bftheta) - \nabla \Reduced{\SAA\ObjFctn}(\bftheta')\right\|^2
&= \left\|\left.\nabla_{\bftheta} {\SAA\ObjFctn}(\bfw, \bftheta)\right|_{\bfw = \WOpt(\bftheta)}- \left.\nabla_{\bftheta} {\SAA\ObjFctn}(\bfw, \bftheta')\right|_{\bfw = \WOpt(\bftheta')}\right\|^2\\
&\le K^2\left(\|\WOpt(\bftheta) - \WOpt(\bftheta')\|^2 + \|\bftheta - \bftheta'\|^2\right)\\
&\le K^2\left(K^2/\lambda_w^2 + 1 \right)  \|\bftheta - \bftheta'\|^2.
\end{split}
\end{align*}
The first equality follows from Lemma~\ref{lem:reduced_gradient_equality}.
The second inequality comes from the $K$-smoothness assumption on $\SAA\ObjFctn$.
The third inequality comes from Lemma~\ref{lem:smoothness_w_opt}.
Computing the square root of both sides concludes the proof.
\end{proof}

\begin{theorem}[Hessian of Reduced Objective Function $\Reduced{\SAA\ObjFctn}$]\label{thm:reduced_objective_hessian}
Under Assumptions~\ref{assump:vp_smoothness} and~\ref{assump:vp_convexity},
the Hessian of the reduced objective function $\Reduced{\SAA\ObjFctn}$ is
equal to the Schur complement of the Hessian of the full objective function
$\SAA\ObjFctn$, evaluated at the optimal linear weights; that is,
\begin{align*}
\nabla^2\Reduced{\SAA\ObjFctn}(\bftheta) =
\left.\left[
\nabla_{\bftheta}^2 {\SAA\ObjFctn}(\bfw, \bftheta) -
\nabla_\bfw^* \nabla_{\bftheta} {\SAA\ObjFctn}(\bfw, \bftheta)
\left(\nabla_{\bfw}^2 {\SAA\ObjFctn}(\bfw, \bftheta)\right)^{-1}
\nabla_\bftheta^* \nabla_{\bfw} {\SAA\ObjFctn}(\bfw, \bftheta)
\right] \right|_{\bfw = \WOpt(\bftheta)}.
\end{align*}
\end{theorem}

\begin{proof}
The proof comes from direct computation.
Specifically,
\begin{align}\label{eq:app_nearly_schur}
\nabla^2\Reduced{\SAA\ObjFctn}(\bftheta)
&= \nabla_{\bftheta}^*\left(\left.\nabla_{\bftheta}{\SAA\ObjFctn}(\bfw, \bftheta)\right|_{\bfw=\WOpt(\bftheta)}\right)\\
&=\left.\nabla_{\bftheta}^2 \SAA\ObjFctn(\bfw, \bftheta)\right|_{\bfw = \WOpt(\bftheta)}
+ \nabla_{\bfw}^* \left(\left.\nabla_{\bftheta}{\SAA\ObjFctn}(\bfw, \bftheta)\right|_{\bfw = \WOpt(\bftheta)} \right)\nabla_{\bftheta}^*\WOpt(\bftheta).
\end{align}
The first equality follows from Lemma~\ref{lem:reduced_gradient_equality}.
The second comes from the chain rule.

The Jacobian of the optimal weights, $\nabla_{\bftheta}^*\WOpt(\bftheta)$, can be expressed by differentiating through the first-order optimality condition; that is,
\begin{align*}
\nabla_{\bftheta}^*\left(\nabla_{\bfw} \SAA\ObjFctn(\WOpt(\bftheta),\bftheta)\right)
= \nabla_{\bfw}^2 \SAA\ObjFctn(\WOpt(\bftheta),\bftheta) \nabla_{\bftheta}^* \WOpt(\bftheta) + \left.\nabla_{\bftheta}^* \nabla_{\bfw}\SAA\ObjFctn(\bfw, \bftheta)\right|_{\bfw = \WOpt(\bftheta)}
= \bf0.
\end{align*}
Solving the above for the Jacobian of the optimal weights and substituting in~\eqref{eq:app_nearly_schur}, we obtain the Schur complement of the Hessian, as desired.
\end{proof}

\begin{corollary}[Conditioning of the Reduced Hessian]\label{cor:reduced_hessian_conditioning}
Suppose the full Hessian $\nabla^2 \SAA\ObjFctn(\WOpt(\bftheta), \bftheta)$ is
positive definite.
Then every eigenvalue of $\nabla^2\Reduced{\SAA\ObjFctn}(\bftheta)$ lies between
the smallest and largest eigenvalues of
$\nabla^2 \SAA\ObjFctn(\WOpt(\bftheta), \bftheta)$.
Consequently,
\begin{align*}
\kappa\!\left(\nabla^2\Reduced{\SAA\ObjFctn}(\bftheta)\right)
\le
\kappa\!\left(\nabla^2 \SAA\ObjFctn(\WOpt(\bftheta), \bftheta)\right),
\end{align*}
where $\kappa(\bfM) = \lambda_{\rm max}(\bfM) / \lambda_{\rm min}(\bfM)$ for a
symmetric positive definite matrix $\bfM$.
\end{corollary}

\begin{proof}
By Theorem~\ref{thm:reduced_objective_hessian},
$\nabla^2\Reduced{\SAA\ObjFctn}(\bftheta)$ is the Schur complement of the full
Hessian at $(\WOpt(\bftheta), \bftheta)$.
For a symmetric positive definite block matrix, the spectrum of the Schur
complement is more compressed than the spectrum of the full matrix; in
particular, the eigenvalues of the Schur complement lie between the smallest
and largest eigenvalues of the full Hessian
\citep{smith_interlacing_1992}.
The stated condition-number bound follows immediately.
\end{proof}

\begin{assumption}[Strong Convexity of the Full Objective]
\label{assump:vp_theta_convexity}
For the convergence-rate result only, the full objective
$\SAA\ObjFctn: \WSpace \times \ThetaSpace \to \Rbb$ is
$\mu$-strongly convex; that is,
\begin{align*}
\nabla^2 \SAA\ObjFctn(\bfw, \bftheta)
\succeq \mu \bfI_{\NW\NTheta} \succ 0
\end{align*}
for all $(\bfw, \bftheta) \in \WSpace \times \ThetaSpace$.
\end{assumption}

\begin{theorem}[Convergence Rate of VarPro Gradient Descent]\label{thm:varpro_gd_rate}
Under Assumptions~\ref{assump:vp_smoothness}, \ref{assump:vp_convexity},
and~\ref{assump:vp_theta_convexity}, gradient descent on the reduced objective
function $\Reduced{\SAA\ObjFctn}$ with the fixed step size
$\gamma = 1 / \Reduced{K}$ converges at a rate
$1 - \gamma \mu$.
\end{theorem}

\begin{proof}
By Theorem~\ref{thm:reduced_objective_smoothness},
$\Reduced{\SAA\ObjFctn}$ is $\Reduced{K}$-smooth.
By Corollary~\ref{cor:reduced_hessian_conditioning}, every eigenvalue of
$\nabla^2\Reduced{\SAA\ObjFctn}(\bftheta)$ lies between the smallest and
largest eigenvalues of $\nabla^2 \SAA\ObjFctn(\WOpt(\bftheta), \bftheta)$.
Assumption~\ref{assump:vp_theta_convexity} therefore implies
\begin{align*}
\nabla^2\Reduced{\SAA\ObjFctn}(\bftheta)
\succeq \mu \bfI_{\NTheta} \succ 0
\end{align*}
for all $\bftheta \in \ThetaSpace$, so $\Reduced{\SAA\ObjFctn}$ is
$\mu$-strongly convex.
The claimed rate is therefore the standard gradient-descent convergence result
for a $\mu$-strongly convex and $\Reduced{K}$-smooth objective.
\end{proof}
 
\section{Additional Trust-Region Theory}
\label{app:trust_region_proofs}
We present some core theoretical foundations of trust-region convergence to a stationary point based on \citet[Theorems 4.5 and 4.6]{nocedal_numerical_2006} and \citet[Theorem 6.4.6]{conn_trust_2000}. 
The proofs are presented based on our notation in Algorithm~\ref{alg:trust_region_skeleton_vpboost}. 
The presented theory can be applied to more general trust-region algorithms provided acceptance/rejection of trial points and rescaling the regularization parameter/trust-region radius are the core mechanisms.

We start with the Cauchy point and re-state a known theorem that the Cauchy point guarantees sufficient model decrease.  
Here, sufficient model decrease is determined based on the norm of the gradient, the norm of the Hessian, and the trust-region radius. 

\begin{lemma}[Cauchy Point Sufficient Model Reduction]\label{lem:cauchy_point_reduction}
The Cauchy point $h_C^{(m)}$, as defined in Definition~\ref{def:cauchy_point}, satisfies 
	\begin{align}\label{eq:cauchy_lower_bound}
		\caQ[0] - \caQ[h_C^{(m)}] \ge \tfrac{1}{2} \|\caDL\| \min\left\{\Delta^{(m)}, \frac{\|\caDL\|}{\|\caDDL\|}\right\}. 
	\end{align}
\end{lemma}

\begin{proof}
The proof follows directly from \citet[Lemma 4.3]{nocedal_numerical_2006} using the quadratic model based on the true Hessian of the loss functional, $\caDDL$. 

Recall, the Cauchy point (Definition~\ref{def:cauchy_point}) is a positive scalar multiple of the negative gradient, $h_C^{(m)} = -\gamma_C^{(m)} \caDL$, with $\gamma_C^{(m)} > 0$.  
The Cauchy model reduction thereby simplifies to 
	\begin{align}\label{eq:cauchy_model_reduction}
	\caQ[0] - \caQ[h_C^{(m)}] 
		&=\gamma_C^{(m)} \|\caDL\|^2 - \tfrac{1}{2}(\gamma_C^{(m)})^2 \beta^{(m)}
	\end{align}
where  $\beta^{(m)} = \langle \caDL, \caDDL \caDL\rangle$. 
The step size  takes on two possible values
	\begin{align}\label{eq:cauchy_step_size_beta}
		\gamma_C^{(m)} =\min\left\{
						\dfrac{\caDeltaLam}{\|\caDL\|} \;, \; 
						\dfrac{\|\caDL\|^2}{\beta^{(m)}}
					\right\}
	\end{align}
and defaults to the first scalar if $\beta^{(m)} \le 0$. 
Because we assume convexity of $\nabla^2 \SAA\LossFunctional$ (Assumption~\ref{assump:ell_convexity})\footnote{The result can be shown for non-convex loss functionals as well. Refer to \citet[Lemma 4.3]{nocedal_numerical_2006} for details.}, we consider two cases: (i) when $\beta^{(m)} = 0$ and (i) when $\beta^{(m)} > 0$. 

\begin{itemize}[leftmargin=*, align=left]
\item[{\bfseries Case 1 ($\beta^{(m)} = 0$):}] In the first case, the step size is $\gamma_C^{(m)} = \caDeltaLam/\|\caDL\|$ and the Cauchy model reduction becomes
	\begin{align*}
		\caQ[0] - \caQ[h_C^{(m)}] 
		&=\caDeltaLam\|\caDL\| 	
	\end{align*}
The lower bound in \eqref{eq:cauchy_lower_bound} cannot be larger than $\caDeltaLam\|\caDL\|$, and thus the theorem holds.  

\item[{\bfseries Case 2 ($\beta^{(m)} > 0$):}]  Two possibilities arise in the second case.  
	\begin{enumerate}[label={\bfseries (\alph*)}]
		\item First, if $\gamma_C^{(m)} = \caDeltaLam/\|\caDL\|$, the Cauchy model reduction in \eqref{eq:cauchy_model_reduction} becomes
	\begin{align*}
		\caQ[0] - \caQ[h_C^{(m)}] 
		&=\caDeltaLam\|\caDL\| - \tfrac{1}{2}	(\caDeltaLam)^2 \frac{\beta^{(m)}}{\|\caDL\|^2}.
	\end{align*}
Furthermore, by the step size formula in \eqref{eq:cauchy_step_size_beta}, we know 
	\begin{align*}
	\dfrac{\caDeltaLam}{\|\caDL\|} \le \dfrac{\|\caDL\|^2}{\beta^{(m)}}.
	\end{align*}
Using this inequality to substitute for one $\caDeltaLam$ in the quadratic term, a lower bound for the model reduction is
	\begin{align*}
		\caQ[0] - \caQ[h_C^{(m)}]
		&\ge \tfrac{1}{2}\caDeltaLam\|\caDL\|
	\end{align*}
and the inequality of interest \eqref{eq:cauchy_lower_bound} holds. 

	\item Second,  if $\gamma_C^{(m)} =\|\caDL\|^2 / \beta^{(m)}$, the Cauchy model reduction in \eqref{eq:cauchy_model_reduction} becomes 
		\begin{align*}
		\caQ[0] - \caQ[h_C^{(m)}] 
		&=\tfrac{1}{2}\frac{\|\caDL\|^4}{\beta^{(m)}}.
		\end{align*}
	Recall that $\beta^{(m)} = \langle \caDL, \caDDL \caDL \rangle$ is bounded above by 
		\begin{align*}
			\beta^{(m)} \le \|\caDDL\| \|\caDL\|^2.
		\end{align*}
	Substituting this upper bound into the denominator of the model reduction yields the lower bound
		\begin{align*}
		\caQ[0] - \caQ[h_C^{(m)}] 
		&\ge \tfrac{1}{2}\frac{\|\caDL\|^2}{\|\caDDL\|}
		\end{align*}
	and again the inequality of interest \eqref{eq:cauchy_lower_bound} holds. 
		
	\end{enumerate}
\end{itemize} 

We have shown that \eqref{eq:cauchy_lower_bound} holds in all possible cases, and thus the theorem is proven. 
\end{proof}

Intuitively, a large model reduction from the Cauchy point (i.e., from traversing the negative gradient direction) is expected if (a) the optimization landscape has not plateaued (i.e., $\caDL$ far from zero) and has limited curvature locally (i.e., $\|\caDDL\|$ close to zero) and (b) if the trust region is large (i.e., $\Delta^{(m)}$ far from zero). 
On the flip side, a small model reduction from the Cauchy point is expected in the opposite cases, notably in the presence of significant curvature, when the negative gradient direction may be suboptimal.  
The next theorem generalizes the notion of sufficient model decrease to any weak learner relative to the Cauchy point. 

\begin{theorem}[Sufficient Model Reduction]\label{thm:sufficient_model_reduction}
	Suppose  $h^{(m)}$ is a feasible weak learner that satisfies $\|h^{(m)}\| \le \caDeltaLam$ and
		\begin{align*}
		\caQ[0] -\caQ[h^{(m)}]  \ge c_2 \left(\caQ[0] -\caQ[h_C^{(m)}] \right)
		\end{align*}
	 for some $c_2 \in (0,1]$. 
	Then, $h^{(m)}$ achieves sufficient reduction in the sense of Lemma~\ref{lem:cauchy_point_reduction}.
\end{theorem}

\begin{proof}
The proof follows directly from \citet[Theorem 4.4]{nocedal_numerical_2006}.   
In particular, we have 
	\begin{align*}
		\begin{split}
		\caQ[0] -\caQ[h^{(m)}]  
			&\ge c_2 \left(\caQ[0] -\caQ[h_C^{(m)}] \right)\\
			 &\ge \frac{c_2}{2} \|\caDL\| \min\left\{\caDeltaLam, \frac{\|\caDL\|}{\|\caDDL\|}\right\}. 
		\end{split}
	\end{align*}
Because $c_2 \le 1$, Lemma~\ref{lem:cauchy_point_reduction} is satisfied. 
\end{proof}

 \subsection{Proof of Theorem~\ref{thm:vpboost_convergence}} 
  \label{sec:proof_vpboost_convergence}

\vpboostconvergenceA*

\begin{proof} 
This proof follows similar logic as \citet[Theorem 4.5]{nocedal_numerical_2006}. 
For the sake of contradiction, suppose there exists some $\varepsilon > 0$ and a positive integer $M \in \Nbb$ such that 
	\begin{align*}
		\|\caDL\| \ge \varepsilon \qquad \text{for all $m\ge M$}.
	\end{align*} 
We will contradict this statement in two steps. 
First, following Lemma~\ref{lem:varpro_regularization_parameter_upper_bound}, the regularization parameter $\caLambda$ cannot become arbitrarily large in Algorithm~\ref{alg:trust_region_skeleton_vpboost}. 
The second step is to show that the lower boundedness of the loss functional
(Assumption~\ref{assump:ell_bounded_below})
and the regularization increase mechanism of
Algorithm~\ref{alg:trust_region_skeleton_vpboost},
Line~\ref{alg_line:reg_increase} necessitate that $\caLambda \to \infty$,
resulting in a contradiction.
To reach this contradiction, we need to consider two opposite cases: (i) when $\rho^{(m)} \ge \rho_{\rm small}$ for infinitely many iterates and (ii) when $\rho^{(m)}  < \rho_{\rm small}$ for all sufficiently large $m$. 

In the first case, let $\Mcal = \{f^{(m)}\}_{m=0}^\infty$ be the sequence of iterates produced by  Algorithm~\ref{alg:trust_region_skeleton_vpboost} and suppose there is an infinite subsequence of iterates $\Mcal' \subset \Mcal$ such that $\rho^{(m)} \ge \rho_{\rm small}$ for all $f^{(m)}\in \Mcal'$. 
Then, for $f^{(m)}\in \Mcal'$ and $m \ge M$, we have 
	\begin{align*}
		\caL - \Lcal_N[f^{(m)} + \cah] 
			&\ge \rho_{\rm small}\left(\caQ[0] - \caQ[\cah]\right) \ge \rho_{\rm small} \cdot \varepsilon^2 \cdot \frac{\kappa_{\rm align}}{\beta +\caLambda / \alpha_{\rm low}^2}
	\end{align*}
where the last inequality comes from Lemma~\ref{lem:varpro_reduction_lower_bound} and the assumption that $\|\caDL\|\ge \varepsilon$ for all $m\ge M$. 
Because $\Lcal_N$ is bounded from below
(Assumption~\ref{assump:ell_bounded_below}),
the right-hand side must approach zero; otherwise, the loss functional would
decrease indefinitely.
As the only $m$-dependent term in the lower bound, $\lambda_w^{(m)} \to \infty$, resulting in a contradiction of Lemma~\ref{lem:varpro_regularization_parameter_upper_bound}.

This means there can only be finitely many iterates for which $\rho^{(m)} \ge \rho_{\rm small}$. 
Thus, for sufficiently large $m$, $\rho^{(m)} < \rho_{\rm small}$ and,  by Algorithm~\ref{alg:trust_region_skeleton_vpboost}, Line~\ref{alg_line:reg_increase}, the regularization parameter will eventually be multiplied by $\gamma_{\rm up} > 1$ infinitely many times. 
This also results in $\caLambda \to \infty$, and again contradicts Lemma~\ref{lem:varpro_regularization_parameter_upper_bound}.  
As a result, our original assumption that $\|\caDL\| \ge \varepsilon$ for all $m \ge M$ must be false, and thus the theorem statement is proven. 
\end{proof}

 \subsection{Proof of Theorem~\ref{thm:vpboost_convergence_part2}} 
 \label{sec:proof_vpboost_convergence_part2}

\vpboostconvergenceB*

\begin{proof} 
This proof follows the strategy from \citet[Theorem 4.6]{nocedal_numerical_2006}, which is adapted from \citet[Theorem 2.2]{shultz_family_1985}.

Suppose at some weak learner stage $m$, we have not converged to a stationary point; that is, $\|\caDL\| > 0$. 
Then, following Theorem~\ref{thm:vpboost_convergence}, there must exist some $k > m$ such that $\|\nabla \Lcal_N[f^{(k)}]\| < \|\caDL\|$.  
Let $k_m$ be the first such iterate after $m$ where the norm of the gradient decreases. 
We construct a lower bound of the loss functional decrease between these iterates via a telescoping series
	\begin{align*}
		\caL - \Lcal_N[f^{(k_m)}] 
			&= \sum_{j=m}^{k_m-1} \Lcal_N[f^{(j)}] - \Lcal_N[f^{(j + 1)}]\\
			&\ge \sum_{j=m,\, f^{(j)} \not=f^{(j + 1)}}^{k_m-1} \Lcal_N[f^{(j)}] - \Lcal_N[f^{(j + 1)}]\\
			&\ge \sum_{j=m,\, f^{(j)} \not=f^{(j + 1)}}^{k_m-1} \rho_{\rm accept}\left(\caQ[0] - \caQ[h_\star^{(j)}] \right)
	\end{align*}
where the first inequality only considers iterates for which the weak learner was accepted (Algorithm~\ref{alg:trust_region_skeleton_vpboost}, Line~\ref{alg_line:vpboost_accept}) and the final inequality follows from the non-zero acceptance threshold. 
From the lower bound in Lemma~\ref{lem:varpro_reduction_lower_bound}, we proceed via\begin{align}\label{eq:loss_functional_reduction_lower_bound_accept}
        \begin{split}
		\caL - \Lcal_N[f^{(k_m)}] 	&\ge \sum_{j=m,\, f^{(j)} \not=f^{(j + 1)}}^{k_m-1} \rho_{\rm accept} c_1 \|\nabla \Lcal_N[f^{(j)}]\|^2 \cdot \frac{\kappa_{\rm align}}{\beta +\lambda_w^{(j)}/\alpha_{\rm low}^2}\\
			&\ge \left(\rho_{\rm accept} \cdot \frac{\kappa_{\rm align}}{\beta +\lambda_{\rm max}/\alpha_{\rm low}^2}\right) \sum_{j=m,\, f^{(j)} \not=f^{(j + 1)}}^{k_m-1} \|\nabla \Lcal_N[f^{(j)}]\|^2\\
			&\ge\left(\rho_{\rm accept} \cdot \frac{\kappa_{\rm align}}{\beta +\lambda_{\rm max}/\alpha_{\rm low}^2}\right) \|\nabla \Lcal_N[f^{(m)}]\|^2
\end{split}
	\end{align}
where $\lambda_{\rm max}$ is the largest attainable regularization parameter in Algorithm~\ref{alg:trust_region_skeleton_vpboost}, as derived in Lemma~\ref{lem:varpro_regularization_parameter_upper_bound}. 

By design, Algorithm~\ref{alg:trust_region_skeleton_vpboost} only accepts iterates that reduce the loss functional. 
Because the loss functional is bounded below
(Assumption~\ref{assump:loss_functional}:~\ref{assump:ell_bounded_below}),
we necessarily have\footnote{The loss functional value will decrease monotonically, but not necessarily strictly monotonically.} $\caL \downarrow \Lcal_N^*$ as $m\to \infty$ for some constant $\Lcal_N^* > -\infty$.
Following the bound in \eqref{eq:loss_functional_reduction_lower_bound_accept}, we have 
	\begin{align*}
		\caL - \Lcal_N^*  
			&\ge \caL - \Lcal_N[f^{(k_m)}] \\
			&\ge \left(\rho_{\rm accept} \cdot \frac{\kappa_{\rm align}}{\beta +\lambda_{\rm max}/\alpha_{\rm low}^2}\right) \|\nabla \Lcal_N[f^{(m)}]\|^2.
	\end{align*}
Because the left-hand side of the inequality must decay to zero and the lower bound is nonnegative, it follows that the lower bound decays to zero as well. 
Thus, $\|\Lcal_N[f^{(m)}]\| \to 0$ as $m\to \infty$. 
\end{proof}

 \subsection{Proof of Theorem~\ref{thm:vpboost_convergence_rate}} 
 \label{sec:proof_vpboost_convergence_rate}

\vpboostconvergencerate*

\begin{proof} 
The proof follows the logic of \citet[Theorem 4.9]{nocedal_numerical_2006}. 
Because \VPBoost builds a quadratic model using the exact Hessian of the loss function, $\caDDL$, we are able to re-use our previous lemmas directly.   
First, we argue that following \VPBoost (Algorithm~\ref{alg:trust_region_skeleton_vpboost}), the trust-region constraint will eventually become inactive for VarPro weak learners when the iterates are close to $f^*$. 
We then take a detour and consider an ensemble constructed from Newton weak learners and appeal to known quadratic convergence rates for Newton steps.  
We conclude the proof using the asymptotic similarity assumption of VarPro and Newton weak learners to show \VPBoost converges superlinearly to $f^*$.

Following Lemma~\ref{lem:varpro_regularization_parameter_upper_bound}, the regularization parameter $\caLambda$ cannot become arbitrarily large while proceeding through \VPBoost (Algorithm~\ref{alg:trust_region_skeleton_vpboost}). 
Recall from~\eqref{eq:trust_region_radius_formula}, $\caLambda$ is inversely proportional to the trust-region radius $\caDeltaLam$, and thus the trust-region radius cannot become arbitrarily small.  
Let $\Delta_{\rm min} > 0$ be the minimal such radius corresponding to $\lambda_{\rm max}$ in Lemma~\ref{lem:varpro_regularization_parameter_upper_bound}. 
It follows that, for large enough $M$, the sequence of \VPBoost ensembles must remain within the open ball centered at $f^*$ with radius $\Delta_{\rm min}$; that is,   $f^{(m)} \in B(f^*, \Delta_{\rm min})$ for all $m \ge M$ where 
	\begin{align*}
		 B(f^*, \Delta_{\rm min}) = \{f\in \Fcal \mid \|f - f^*\| < \Delta_{\min}\}. 
	 \end{align*}

We now recall some properties of Newton steps, translated to the weak learner setting. 
For some sufficiently large\footnote{Because the Hessian is continuous (Assumption~\ref{assump:ell_twice_continuously_differentiable}) and $\nabla^2 \SAA\LossFunctional[f^*] \succ 0$, we have that $\caDDL \succ 0$, and thus invertible, when $f^{(m)}$ is close enough to $f^*$.} positive integer $M$, let  
	\begin{align*}
	\cahN = -\caDDL^{-1} \caDL
	\end{align*}
denote the Newton weak learner for $m \ge M$. 
It is well-known \citep[Theorem 3.5]{nocedal_numerical_2006} that if $f^{(M)}$ is sufficiently close to $f^*$, then the sequence of Newton-based ensembles will converge quadratically to $f^*$; that is, 
	\begin{align}\label{eq:newton_quadratic}
		\|f^{(m)} + \cahN - f^*\| \le \kappa_{\diamond} \|f^{(m)}  - f^*\|^2
	\end{align}
for some fixed constant $\kappa_\diamond > 0$ related to the curvature of the Hessian at the stationary point \citep[Equation 3.33]{nocedal_numerical_2006}.  
We write \eqref{eq:newton_quadratic} in Big-O notation as $\|f^{(m)} + \cahN - f^*\| = O( \|f^{(m)}  - f^*\|^2)$. 
Using the triangle inequality, we recognize that 
	\begin{align*}
		\|\cahN \| \le \|f^{(m)} + \cahN  - f^*\|  +  \|f^{(m)} - f^*\|. 
	\end{align*}
Thus, $\|\cahN \|$ converges to zero at a rate  proportional to the dominant term $\|f^{(m)} - f^*\|$; in Big-O notation, $\|\cahN \| = O(\|f^{(m)} - f^*\|)$.

Assume we now choose a sufficiently large $M$ such that  for all $m\ge M$, $f^{(m)} \in B(f^*, \Delta_{\rm min})$ and $f^{(m)}$ is sufficiently close to $f^*$ so that Newton weak learners can be formed and updates remain within $B(f^*, \Delta_{\rm min})$. 
We now prove that the VarPro weak learners, $\cah$, are related to the Newton weak learners, $\cahN$, and thus that \VPBoost inherits convergence properties related to Newton ensemble convergence.  
We relate the error of the $(m+1)^{\rm st}$ ensemble based on appending a VarPro weak learner or a Newton weak learner using the triangle inequality via
	\begin{align*}
		\|f^{(m)} + \cah - f^*\| 
			&\le \|f^{(m)} + \cahN - f^*\| + \|\cah - \cahN\| \le O(\|f^{(m)} - f^*\|^2) + o(\|\cahN\|)
	\end{align*}
The final term comes from the asymptotic similarity assumption posed in the theorem statement. 
Because the Newton steps decay to zero, $\|\cahN\| \to 0$ and $\|\cahN\| = O(\|f^{(m)} - f^*\|)$, it necessarily follows that the upper bound becomes
	\begin{align*}
	\|f^{(m)} + \cah - f^*\| = o(\|f^{(m)} - f^*\|). 
	\end{align*}
Thus, \VPBoost converges superlinearly to $f^*$. 
\end{proof}

\end{document}